%% file: arxiv.tex
\documentclass[11pt]{article}

\usepackage{etoolbox}
\newtoggle{colt}
\togglefalse{colt}

\input{header.tex}

\input{macros.tex}

\input{characters.tex}

\usepackage[suppress]{color-edits}
\addauthor{ab}{red}
\addauthor{vs}{olive}
\addauthor{dh}{orange}
\addauthor{pl}{brown}

\Crefname{assumption}{Assumption}{Assumptions}

\title{Behavior Cloning is Not All You Need: The Optimality of On-Policy Distillation for Noisy Expert Feedback}

\author[1]{Ved Sriraman}
\author[1]{Peihan Liu}
\author[1]{Daniel Hsu}
\author[1,2]{Adam Block}

\affil[1]{Department of Computer Science, Columbia University}
\affil[2]{Department of Electrical Engineering, Columbia University}

\date{}

\begin{document}

\maketitle

\begin{abstract}
\input{body_clean/abstract.tex}

\end{abstract}

\input{body_clean/intro.tex}

\input{body_clean/prelims.tex}

\input{body_clean/offline.tex}

\input{body_clean/online.tex}

\input{body_clean/generalizations.tex}

\input{body_clean/experiments.tex}

\input{body_clean/app_add_related_work}

\bibliographystyle{plainnat}
\bibliography{refs}

\tableofcontents


\appendix
\crefalias{section}{appendix}
\crefalias{subsection}{appendix}
\crefalias{subsubsection}{appendix}

\input{body_clean/app_additional_experimental_results}

\input{body_clean/app_technical_prelims.tex}

\input{body_clean/app_offline_proofs.tex}

\input{body_clean/app_online_proofs.tex}

\input{body_clean/app_greedy_online_proofs.tex}

\input{body_clean/app_testing_online_proofs.tex}

\end{document}

%% file: header.tex
\iftoggle{colt}{}{
\usepackage[sort&compress,numbers]{natbib}
}

\usepackage{boxedminipage}
\usepackage{multirow,nicefrac}
\usepackage{makecell,upgreek}
\usepackage{footnote}
\usepackage{longtable}
\usepackage{tablefootnote}
\usepackage[T1]{fontenc}
\usepackage{verbatim} 
\usepackage[utf8]{inputenc}

\usepackage{booktabs}   
\usepackage{adjustbox}
\usepackage{tabularx}

\iftoggle{colt}{}{
\usepackage[margin=0.75in]{geometry}
}
\usepackage{algorithm}
\usepackage{algpseudocode}

\usepackage[dvipsnames]{xcolor}

\usepackage{colortbl}
\definecolor{lightgray}{gray}{0.9}  

\usepackage{amsmath}
\usepackage{amsthm}
\iftoggle{colt}{}{

\usepackage[breaklinks=true]{hyperref}

}
\hypersetup{
	colorlinks=true,
	linkcolor=blue,
	citecolor=blue,
	urlcolor=blue
}

\usepackage[capitalise,nameinlink]{cleveref}

\usepackage{amsfonts,amssymb,mathrsfs}
\usepackage{mathtools}

\usepackage{appendix}
\newcommand{\defeq}{:=}

\theoremstyle{plain}
	\newtheorem{theorem}{Theorem}

	\newtheorem{lemma}{Lemma}

	\newtheorem{corollary}{Corollary}
	\newtheorem{proposition}{Proposition}

	\theoremstyle{definition}

	\newtheorem{definition}{Definition}

	\newtheorem{remark}{Remark}

\Crefname{appendix}{Appendix}{Appendices}

\usepackage{autonum}

\usepackage[T1]{fontenc}
\usepackage[scaled]{beramono}
\usepackage{listings}
\definecolor{varorange}{RGB}{230,126,34}     
\definecolor{opblue}{RGB}{52,152,219}        
\definecolor{ifred}{RGB}{192,57,43}          
\definecolor{commentgray}{RGB}{150,150,150}  
\definecolor{codebg}{rgb}{0.98,0.98,0.98}    

\newsavebox\codebox

\usepackage{authblk}
\usepackage{xspace}

%% file: macros.tex
\newcommand{\abs}[1]{\left|#1\right|}

\newcommand{\rr}{\mathbb{R}}
\newcommand{\ee}{\mathbb{E}}
\newcommand{\pp}{\mathbb{P}}

\newcommand{\kld}[2]{\mathrm{D}_{\mathrm{KL}}\left( #1 \middle\| #2 \right)}
\newcommand{\kldinline}[2]{\mathrm{D}_{\mathrm{KL}}( #1 \| #2)}
\newcommand{\tvd}[2]{\mathrm{TV}\left( #1 , #2 \right)}
\newcommand{\dhel}[2]{\mathrm{D}_{\mathrm{H}^2}\left( #1 , #2 \right)}
\newcommand{\dhelinline}[2]{\mathrm{D}_{\mathrm{H}^2}( #1 , #2 )}

\DeclareMathOperator*{\argmin}{argmin}
\DeclareMathOperator*{\argmax}{argmax}

\DeclareMathOperator{\reg}{Reg}

\newcommand{\pihat}{\hat\pi}
\newcommand{\pistar}{\pi^\star}

\newcommand{\Unif}{\mathrm{Unif}}
\newcommand{\Bernoulli}{\mathrm{Bernoulli}}

\newcommand{\astar}{a^\star}

\newcommand{\ind}[1]{\mathbb{I}\left\{#1\right\}}

\newcommand{\mubar}{\overline{\mu}}

\newcommand{\abar}{\overline{a}}

\newcommand{\nail}{\textsf{NAIL}\xspace}
\newcommand{\neil}{\textsf{NEIL}\xspace}
\newcommand{\gnail}{\textsf{NAILGUN}\xspace}
\newcommand{\algdagger}{\textsf{DAgger}\xspace}
\newcommand{\algsmile}{\textsf{SMILe}\xspace}
\newcommand{\algaggrevate}{\textsf{AggreVate}\xspace}
\newcommand{\alggail}{\textsf{GAIL}\xspace}
\newcommand{\algdart}{\textsf{DART}\xspace}
\newcommand{\algbc}{\textsf{BC}\xspace}
\newcommand{\algllbc}{\textsf{LogLossBC}\xspace}
\newcommand{\nailf}{\textsf{NAIL-F}\xspace}
\newcommand{\nailr}{\textsf{NAIL-R}\xspace}
\newcommand{\opdf}{\textsf{OPD-F}\xspace}
\newcommand{\opdr}{\textsf{OPD-R}\xspace}

\DeclareMathOperator{\supp}{supp}
\newcommand{\gsmk}{\textsc{GSM-8K}}
\newcommand{\gemmab}{\textsc{Gemma3-1B-IT}}
\newcommand{\gemmam}{\textsc{Gemma3-270M-IT}}
\newcommand{\sg}{\mathsf{sg}}
\newcommand{\softmax}{\mathsf{softmax}}

%% file: characters.tex
\renewcommand{\epsilon}{\varepsilon}

\def\ddefloop#1{\ifx\ddefloop#1\else\ddef{#1}\expandafter\ddefloop\fi}
\def\ddef#1{\expandafter\def\csname bb#1\endcsname{\ensuremath{\mathbb{#1}}}}
\ddefloop ABCDEFGHIJKLMNOPQRSTUVWXYZ\ddefloop

\def\ddefloop#1{\ifx\ddefloop#1\else\ddef{#1}\expandafter\ddefloop\fi}
\def\ddef#1{\expandafter\def\csname frak#1\endcsname{\ensuremath{\mathfrak{#1}}}}
\ddefloop ABCDEFGHIJKLMNOPQRSTUVWXYZ\ddefloop

\def\ddefloop#1{\ifx\ddefloop#1\else\ddef{#1}\expandafter\ddefloop\fi}
\def\ddef#1{\expandafter\def\csname fr#1\endcsname{\ensuremath{\mathfrak{#1}}}}
\ddefloop ABCDEFGHIJKLMNOPQRSTUVWXYZ\ddefloop

\def\ddefloop#1{\ifx\ddefloop#1\else\ddef{#1}\expandafter\ddefloop\fi}
\def\ddef#1{\expandafter\def\csname eul#1\endcsname{\ensuremath{\EuScript{#1}}}}
\ddefloop ABCDEFGHIJKLMNOPQRSTUVWXYZ\ddefloop

\def\ddefloop#1{\ifx\ddefloop#1\else\ddef{#1}\expandafter\ddefloop\fi}
\def\ddef#1{\expandafter\def\csname scr#1\endcsname{\ensuremath{\mathscr{#1}}}}
\ddefloop ABCDEFGHIJKLMNOPQRSTUVWXYZ\ddefloop

\def\ddefloop#1{\ifx\ddefloop#1\else\ddef{#1}\expandafter\ddefloop\fi}
\def\ddef#1{\expandafter\def\csname b#1\endcsname{\ensuremath{\mathbf{#1}}}}
\ddefloop ABCDEGHIJKLMNOPQRSTUVWXYZ\ddefloop

\def\ddefloop#1{\ifx\ddefloop#1\else\ddef{#1}\expandafter\ddefloop\fi}
\def\ddef#1{\expandafter\def\csname bhat#1\endcsname{\ensuremath{\hat{\mathbf{#1}}}}}
\ddefloop ABCDEFGHIJKLMNOPQRSTUVWXYZ\ddefloop

\def\ddefloop#1{\ifx\ddefloop#1\else\ddef{#1}\expandafter\ddefloop\fi}
\def\ddef#1{\expandafter\def\csname btil#1\endcsname{\ensuremath{\tilde{\mathbf{#1}}}}}
\ddefloop ABCDEFGHIJKLMNOPQRSTUVWXYZ\ddefloop

\def\ddefloop#1{\ifx\ddefloop#1\else\ddef{#1}\expandafter\ddefloop\fi}
\def\ddef#1{\expandafter\def\csname bst#1\endcsname{\ensuremath{\mathbf{#1}^\star}}}
\ddefloop ABCDEFGHIJKLMNOPQRSTUVWXYZ\ddefloop

\def\ddefloop#1{\ifx\ddefloop#1\else\ddef{#1}\expandafter\ddefloop\fi}
\def\ddef#1{\expandafter\def\csname bst#1\endcsname{\ensuremath{\mathbf{#1}^\star}}}
\ddefloop abcdeghijklmnopqrstuvwxyz\ddefloop

\def\ddefloop#1{\ifx\ddefloop#1\else\ddef{#1}\expandafter\ddefloop\fi}
\def\ddef#1{\expandafter\def\csname bhat#1\endcsname{\ensuremath{\hat{\mathbf{#1}}}}}
\ddefloop abcdefghijklmnopqrstuvwxyz\ddefloop

\def\ddefloop#1{\ifx\ddefloop#1\else\ddef{#1}\expandafter\ddefloop\fi}
\def\ddef#1{\expandafter\def\csname b#1\endcsname{\ensuremath{\mathbf{#1}}}}
\ddefloop abcdeghijklnopqrstuvwxyz\ddefloop

\def\ddefloop#1{\ifx\ddefloop#1\else\ddef{#1}\expandafter\ddefloop\fi}
\def\ddef#1{\expandafter\def\csname barb#1\endcsname{\ensuremath{\bar{\mathbf{#1}}}}}
\ddefloop abcdefghijklmnopqrstuvwxyz\ddefloop

\def\ddef#1{\expandafter\def\csname c#1\endcsname{\ensuremath{\mathcal{#1}}}}
\ddefloop ABCDEFGHIJKLMNOPQRSTUVWXYZ\ddefloop
\def\ddef#1{\expandafter\def\csname h#1\endcsname{\ensuremath{\widehat{#1}}}}
\ddefloop ABCDEFGHIJKLMNOPQRSTUVWXYZ\ddefloop
\def\ddef#1{\expandafter\def\csname hc#1\endcsname{\ensuremath{\widehat{\mathcal{#1}}}}}
\ddefloop ABCDEFGHIJKLMNOPQRSTUVWXYZ\ddefloop
\def\ddef#1{\expandafter\def\csname t#1\endcsname{\ensuremath{\widetilde{#1}}}}
\ddefloop ABCDEFGHIJKLMNOPQRSTUVWXYZ\ddefloop
\def\ddef#1{\expandafter\def\csname tc#1\endcsname{\ensuremath{\widetilde{\mathcal{#1}}}}}
\ddefloop ABCDEFGHIJKLMNOPQRSTUVWXYZ\ddefloop

%% file: body_clean/abstract.tex
Imitation Learning (IL) is a natural framework for learning in sequential decision-making systems and has emerged as the dominant paradigm through which we understand language model training.  A central puzzle is that, while in theory offline IL can be horizon-free and optimal, in practice online methods such as on-policy distillation (OPD) often outperform offline methods such as supervised fine-tuning (SFT).  We propose a noisy expert model to explain this gap, in which the learner only has access to a noisy version of the expert's policy, but wishes to compete against the reward achieved by a clean expert, motivated by the fact that in many applications, e.g. training language models to perform long chains of thought, the expert is often imperfect.  In this setting, we show a sharp separation between offline and online IL.  Offline learning from noisy trajectories is fundamentally hard: to compete with the clean expert, the sample complexity must grow exponentially, in contradistinction to the clean expert setting where no explicit horizon dependence exists.  In contrast, we prove that online interaction with the noisy expert via a novel variant of OPD enables polynomial dependence on the horizon in general.  We further show that, under a natural domination condition on the expert noise distribution, which we show to be necessary for any horizon-free sample complexity, one can obtain such a guarantee, although our proposed algorithm sacrifices statistical efficiency in its dependence on the size of the policy class.  Our analysis leads to an alternative loss function that is commonly considered empirically for LM training. We further provide algorithms and lower bounds, and extend our results to the more realistic setting of unknown corruption when the clean expert is deterministic, thereby providing a theoretical foundation for why on-policy distillation can outperform standard supervised fine-tuning when training language models from imperfect teachers. We complement our theoretical results with experiments on synthetic and natural-language tasks, showing that the OPD variant suggested by our theory outperforms both offline BC and existing OPD objectives under noisy expert feedback.

%% file: body_clean/intro.tex
\section{Introduction}\label{sec:intro}

Imitation learning (IL) is a powerful framework for training agents to perform complex tasks in challenging environments by learning from expert demonstrations \citep{pomerleau1988alvinn,ross2010efficient,ross2011reduction,block2024butterfly}.  In recent years, IL has gained significant attention due to its success in a wide range of applications, including robotics \citep{chi2025diffusion,barreiros2026careful}, autonomous driving \citep{chen2019deep}, and natural language processing \citep{chang2023learning,block2024butterfly,agarwal2024policy,lu2025onpolicydistillation}.  In particular, IL has been proposed as a natural framework through which to understand Language Models (LMs), where the training corpus is thought to consist of `expert demonstrations' of human behavior.  Due to the increasing importance and cost of LMs, this raises the question of how to best leverage expert access to train LMs, and more generally, how to best leverage expert access in IL.

Prior work has distinguished between two paradigms for IL: \emph{offline} IL, where the learner only has access to a fixed dataset of expert demonstrations, and \emph{online} IL, where the learner can interact with the expert and query for demonstrations in states visited by the learner \citep{ross2010efficient,ross2011reduction,foster2024behavior,rohatgi2025computational}.  The canonical algorithm in offline IL is Behavioral Cloning (\algbc), where the learner simply performs supervised learning on the expert demonstrations in an effort to predict the action from the observation.  While it was classically thought that \algbc\ suffers from compounding errors and thus has poor performance in long-horizon tasks, with the suboptimality of the learned policy growing quadratically with the horizon \citep{ross2011reduction,ross2010efficient}, more recent work has shown that \algbc, when instantiated with the correct loss function, is both horizon-free and optimal in a minimax sense, even when comparing to online IL algorithms \citep{foster2024behavior}.  While this theoretical result may seem to suggest that \algbc\ is the best way to leverage expert access, in practice, many works have found that online IL algorithms, such as DAgger \citep{ross2011reduction} and On-Policy Distillation (OPD) \citep{agarwal2024policy,lu2025onpolicydistillation}, can significantly outperform offline IL in a wide range of settings.  
Recent work has suggested that this discrepancy between theory and practice may be due to the fact that the expert is often misspecified, meaning that the expert's policy is not realizable by the learner's policy class \citep{rohatgi2025computational}; such misspecification can occur in robotics and autonomous driving problems when the expert and learner have different observation spaces (e.g. a human demonstrator has more sensory input than a robot learning) \citep{barreiros2026careful} and in LMs when we are trying to distill a large teacher model into a smaller student model \cite{agarwal2024policy,hinton2015distilling}.  While misspecified experts are a natural model in many settings, they do not apply to one of the chief successes of IL for LMs, which is to use OPD to fine-tune a student LM from a teacher LM that has been trained with reinforcement learning (RL) but where both student and teacher share the same architecture \citep{agarwal2024policy,lu2025onpolicydistillation}.  In this setting, the expert is not misspecified, but online IL (OPD) still outperforms offline IL (SFT) empirically.

\begin{figure}[t]
    \centering
    \includegraphics[width=\linewidth]{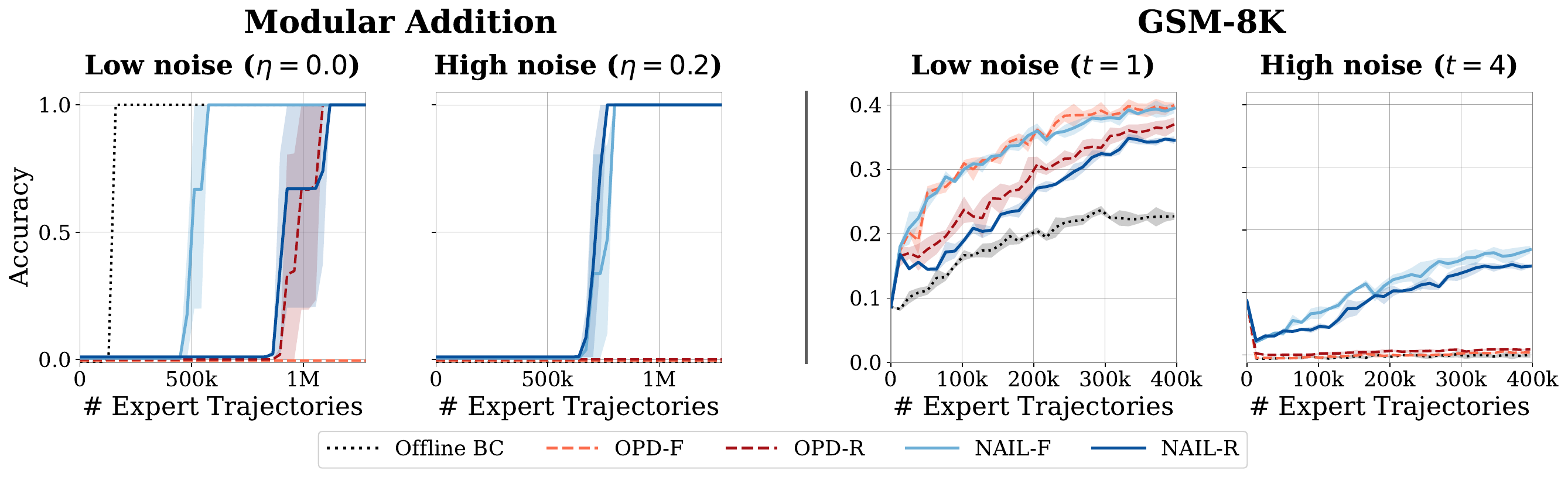}
    \caption{Comparison of offline BC, standard OPD from \citep{agarwal2024policy,lu2025onpolicydistillation}, and our proposed \nail for forward (F) and reverse (R) KL losses with clean and noisy experts. \textbf{Left:} Modular addition task: adding 31 numbers mod 7, where the expert is corrupted with probability $\eta \in \left\{ 0, 0.2 \right\}$. \textbf{Right:} GSM-8K: math reasoning where the expert generates with temperature $t \in \left\{ 1, 4 \right\}$.  In both cases, \nail\  is competitive with offline methods and standard OPD in the low noise regime and strongly outperforms these baselines in the high noise regime. 
    }
    \label{fig:main_modadd_gsm8k}
\end{figure}

In this work, we propose a \emph{noisy expert} model to explain the discrepancy between theory and practice in IL.  In particular, we suppose that there is a true expert policy $\pistar$ that achieves good expected reward when rolled out in an environment, but we only have access to a noisy version of this expert, $\pistar_\eta = (1 - \eta) \cdot \pistar + \eta \cdot \nu$, where $0 \leq \eta < 1$ is a noise level and $\nu$ is some noise distribution.  This model captures the fact that in many real-world settings, such as math, code, and reasoning, the true expert against which we wish to compete could be approximately deterministic, but for computational reasons we are training a stochastic expert.  Moreover, in the LM setting, the teacher model that we are distilling from is often a large model that is not fully converged and thus can be thought of as a noisy expert that occasionally makes mistakes \citep{team2024gemma,yang2025qwen3}.  Such mistakes can also arise from human text that introduces minor errors, typos, and inconsistencies into datasets, which is known to be harmful for training LMs \citep{olmo20242,li2024datacomp,weber2024redpajama}.  We thus ask: \emph{Given access to a noisy expert, how should we leverage this expert to train a learner policy that performs well when rolled out in the environment?}

We provide an answer to this question in both the offline and online settings.  For the sake of simplicity, in the introduction we discuss a special case of the noisy expert model (\Cref{def:kappa_domination}), which is always satisfied when $\pistar$ is deterministic.  We now summarize our main theoretical contributions.  We first consider the setting where both $\eta$ and $\nu$ are known to the learner; while unrealistic in practice, this setting allows us to cleanly characterize the fundamental limits of learning from a noisy expert without adding identifiability concerns (cf. \Cref{sec:generalizations}).  Our first result is that in offline settings, unlike the clean expert setting where horizon-free guarantees are possible, we must pay \emph{exponentially} in horizon $H$ in order to guarantee learning whenever $\eta \gtrsim \nicefrac 1H$.
\begin{theorem}[Informal version of \Cref{thm:offline_bc,prop:offline_lower_bound}]
    Suppose an expert $\pistar$ is contained in a policy class $\Pi$ and a learner has access to $n$ trajectories rolled out from a noisy expert $\pistar_\eta$.  In order to learn a policy $\hat{\pi}$ whose expected reward when rolled out is within $\epsilon$ of that of $\pistar$, \algbc\ requires $n \gtrsim \nicefrac{(1 -\eta)^{-(H+2)} \cdot \log(\abs{\Pi})}{\epsilon}$ and this is optimal.
\end{theorem}
This result stands in marked contrast to the clean expert setting, where \algbc\ can achieve horizon-free guarantees, and suggests that offline IL is fundamentally intractable when the expert is noisy.  In the \emph{online} setting, however, we show that we can circumvent this exponential dependence on the horizon via \emph{On-Policy Distillation}.  Our second main result shows that a natural analogue of OPD, where the learner queries the noisy expert for demonstrations in states visited by the learner, can achieve horizon-free guarantees even when the expert is noisy.
\begin{theorem}[Informal version of \Cref{thm:augmented_hellinger} and \Cref{cor:kl_augmented}]
    Let $\pp^{\pi, \pi'}$ denote the law of an augmented trajectory of states and actions generated by rolling out $\pi$ in an environment, but at each step querying $\pi'$ for an auxiliary action. Then, the suboptimality in expected reward of a policy $\pihat$ can be bounded by $(1 - \eta)^{-2} \cdot (\kldinline{\pp^{\pihat, \pistar_\eta}}{\pp^{\pihat, \pihat_\eta}} \wedge \kldinline{\pp^{\pihat, \pihat_\eta}}{\pp^{\pihat, \pistar_\eta}})$. 
\end{theorem}
In fact, our results are stronger, in that we control suboptimality by the smaller \emph{Hellinger distance} (\Cref{def:hellinger_distance}) between the augmented trajectory distributions.
Note that $\pp^{\pihat, \pistar_\eta}$ is only available through online feedback: we roll out policy $\pihat$ and, at each step, ask for what the (noisy) expert $\pistar_\eta$ would do in that state.  \emph{This result suggests that rolling out the student as will be done at deployment (e.g. greedily) and then using the noisy expert to score the actions taken by the student is the right way to leverage online access to a noisy expert, providing a principled justification for OPD.}
We emphasize that the suggested loss function is \emph{different} from that considered in \citet{lu2025onpolicydistillation,agarwal2024policy}, which  do not distinguish between rollout distribution (e.g. greedy) and student policy (e.g. temperature 1) during training. 

Our final main theoretical result concerns our ability to control this forward KL divergence.
\begin{theorem}[Informal version of \Cref{thm:nail,thm:gnail}]
    Suppose we have access to a policy class $\Pi$ that contains $\pistar$.  If $\eta$ and $\nu$ are known, there is an OPD-like algorithm that can find a policy $\pihat$ with suboptimality in expected reward at most $\epsilon$ with $n \gtrsim \nicefrac{H^2 \cdot \log(\abs{\Pi})}{\epsilon^2(1 - \eta)^2}$ rounds of interaction with the noisy expert $\pistar_\eta$ and this is optimal up to a factor of $H$.  Moreover, if $\eta$ and $\nu$ are unknown and $\pistar$ is deterministic, then as long as $\eta$ is not too large and $\nu$ puts mass at most $\rho$ on any action, $n \gtrsim \nicefrac{\log(\abs{\Pi})}{\epsilon (1 - \eta(1 + \rho))^2}$ rounds of interaction suffice and this is optimal. 
\end{theorem}
In essence, we show that horizon-free guarantees on suboptimality with respect to the \emph{clean} expert are possible when given online access to a noisy expert, when the corruption is unknown but the expert is deterministic, while we pay only polynomially in the horizon in general when given online access to a noisy expert. 

In \Cref{app:horizon_free_testing}, we extend our above results by removing a horizon factor from the above guarantees, even for stochastic policies (when the corruption is known) through a sequential elimination approach, although we pay for this optimal horizon dependence with a linear factor of $|\Pi|$, which is exponentially worse than the logarithmic dependence of the above results. In particular, in the special case on which we focus in the introduction (\Cref{def:kappa_domination}), we achieve \emph{horizon-free guarantees}.

We empirically validate our theory in \Cref{fig:main_modadd_gsm8k} on both a synthetic task of modular addition \citep{li2024chain} and a natural language task of math reasoning (GSM-8K) \citep{cobbe2021training} by minimizing a loss function motivated by \Cref{cor:kl_augmented}. Across both settings, online distillation substantially outperforms \algllbc under noisy expert feedback, while remaining competitive under clean expert feedback, in line with our theoretical findings. Moreover, our \nail\ variants consistently improve over standard OPD in the noisy expert setting, suggesting that the loss function suggested by our theory is indeed more effective for leveraging online access to a noisy expert\footnote{Code is available at \url{https://github.com/plau666/NAIL}.}.  We defer further experimental details to \Cref{sec:experiments}.

In \Cref{sec:prelims} we formally introduce the problem, as well as some prerequisite notions.  In \Cref{sec:offline}, we analyze the \emph{offline} setting before discussing the \emph{online} setting in \Cref{sec:online}.  In \Cref{sec:generalizations}, we extend to the setting where $\eta$ and $\nu$ are unknown and the expert is deterministic. Finally, in \Cref{sec:experiments} we present some preliminary empirical results validating our theoretical findings.  We conclude with a discussion of related work and future directions in \Cref{app:add_related_work} and \Cref{tab:general_results} summarizes our results.

%% file: body_clean/prelims.tex
\section{Formal Problem Setup and Preliminaries}\label{sec:prelims}

We consider Imitation Learning in the standard episodic Markov Decision Process (MDP) setting \citep{ross2010efficient,ross2011reduction,foster2024behavior}.  In particular, we have an MDP $M$ with horizon $H$ consisting of a state space $\cS$, an action space $\cA$, transition kernels $P_h: \cS \times \cA \to \Delta(\cS)$ for each step $h \in [H]$, and a reward function $r: \cS \times \cA \to [0,1]$.  We assume that the initial state distribution is given by some $\rho \in \Delta(\cS)$.  We will consider learning \emph{policies} $\pi: \cS \times [H] \to \Delta(\cA)$ that can be non-stationary across the horizon, mapping states to distributions over actions.  Given a policy $\pi$ we denote by $\pp^\pi$ the distribution over trajectories induced by rolling out $\pi$ in the MDP $M$, i.e., $\tau = (s_1, a_1, \ldots, s_H, a_H)$ with $s_1 \sim \rho$, $a_h \sim \pi_h(\cdot | s_h)$ and $s_{h+1} \sim P_h(\cdot | s_h, a_h)$.  We will also use $\ee^\pi$ to denote the expectation over trajectories induced by $\pi$.  We are primarily concerned with finding policies $\pi$ that have small \emph{regret} with respect to the expert policy $\pistar$:
\iftoggle{colt}{$\reg(\pi) = J(\pistar) - J(\pi)$, where $J(\pi) = \ee^\pi\left[\sum_{h=1}^H r(s_h, a_h)\right]$.}
{
\begin{align}
    \reg(\pi) = J(\pistar) - J(\pi), \quad \text{where} \quad J(\pi) = \ee^\pi\left[\sum_{h=1}^H r(s_h, a_h)\right].
\end{align}
}
As in \citet{foster2024behavior}, we will decouple the reward range from the horizon by assuming that $0 \leq \sum_{h = 1}^H r(s_h, a_h) \leq R$ for some $R$, for any trajectory $\tau$.  While na{\"i}vely, $R$ can scale with the horizon $H$, in many settings of interest we can have $R$ be a constant independent of $H$, which allows for horizon-free regret bounds, e.g. for LMs learning to solve math problems where the trajectory reward is binary \citep{shao2024deepseekmath,guo2025deepseek}.

Unlike in reinforcement learning, imitation learning assumes access only to trajectories $\tau$ without any reward signal.  While in classical IL, one assumes access to expert actions, in this work we consider instead a \emph{noisy expert}, where there exists a corruption policy $\nu : \cS \times [H] \to \Delta(\cA)$ and a corruption level $\eta \in [0,1)$ such that the learner has access to $\pistar_\eta$ defined as
\begin{align}\label{eq:noisy_expert}
    \pistar_{\eta,h}(a \mid s) = (1 - \eta) \cdot \pistar_h(a \mid s) + \eta \cdot \nu_h(a \mid s).
\end{align} 
As in more traditional IL, we distinguish between two different ways in which to interact with the (noisy) expert.  In the \emph{offline} setting \citep{pomerleau1988alvinn,ross2010efficient,foster2024behavior}, we assume access only to expert trajectories and attempt to learn without additional interaction, formalized as follows.
\begin{definition}[Offline Imitation Learning]\label{def:offline_il}
    In the \emph{offline} setting, a learner is given $n$ trajectories $\tau^1, \dots, \tau^n \sim \pp^{\pistar_\eta}$ with $\pistar_\eta$ as in \eqref{eq:noisy_expert} and must output a policy $\pihat$ without further interaction.
\end{definition}
Meanwhile, in \emph{online} imitation learning \citep{ross2011reduction}, the learner is instead allowed to query the expert on trajectories rolled out by the learner itself, formalized as follows.
\begin{definition}\label{def:online_il}
    In the \emph{online IL} setting, a learner interacts with the environment in rounds.  In each round $t$, the learner deploys a policy $\pi_t$ and observes a trajectory $\tau^{(t)}$ drawn from $\pp^{\pi_t}$.  The learner can then query the (noisy) expert $\pistar_\eta$ at each of the states visited in $\tau^{(t)}$ to get the expert's action at those states, forming $\tau^{(t)'} = (s_1, a_1', \dots, s_H, a_H')$, where $a_h' \sim \pistar_{\eta,h}(\cdot | s_h)$.
\end{definition}
When $\eta = 0$, this recovers the online IL setting studied in \citet{ross2011reduction}.  A priori, the online setting is more powerful than the offline setting and  classically, when $\eta = 0$ (i.e., we are in the clean expert setting) this additional power was reflected in an improved sample complexity of online algorithms like \algdagger\ \citep{ross2011reduction} compared to offline algorithms like \algbc. More recently, however, \citet{foster2024behavior} showed that in the realizable setting, \algbc is actually optimal and achieves horizon-free regret bounds.  In order to state these bounds, we recall the definition of the Hellinger distance between two distributions $P$ and $Q$ over the same space $\cX$:
\begin{definition}[Hellinger Distance]\label{def:hellinger_distance}
    The Hellinger distance between two distributions $P$ and $Q$ over the same space $\cX$ is defined as
    \iftoggle{colt}{$\dhel{P}{Q} = \nicefrac{1}{2} \cdot \int_{\cX} ( \sqrt{\nicefrac{d P}{d \mu}(x)} - \sqrt{\nicefrac{d Q}{d \mu}(x)} )^2 d\, \mu(x)$,}
    {
    \begin{align}
        \dhel{P}{Q} = \nicefrac 12 \cdot \int_{\cX} \left( \sqrt{\nicefrac{d P}{d \mu}(x)} - \sqrt{\nicefrac{d Q}{d \mu}(x)} \right)^2 d\, \mu(x),
    \end{align}
    }
    where $\mu$ is any distribution that dominates both $P$ and $Q$.
\end{definition}
While we defer to \citet{polyanskiy2025information} for a detailed discussion of the properties of the Hellinger distance, we provide a brief exposition in \Cref{app:info_theory}.  In particular, the Hellinger distance is intimately related to the more commonly used total variation distance up to a quadratic factor, and it is upper bounded by the KL divergence via Pinsker's inequality.  The following result shows that regret is intimately related to the Hellinger distance between the trajectory distributions of the expert and learned policies.
\begin{theorem}[Theorem 2.1 \& 3.1 from \citep{foster2024behavior}]\label{thm:regret_hellinger}
    Let $\pistar$ and $\pihat$ be any two policies in an MDP with reward range $R$.  Then it holds that
    \begin{align}
        J(\pistar) - J(\pihat) \lesssim R \cdot\sqrt{ \dhel{\pp^{\pistar}}{\pp^{\pihat}}} + R \cdot \dhel{\pp^{\pistar}}{\pp^{\pihat}}.
    \end{align}
    Moreover, if $\pistar$ is deterministic, then it holds that $J(\pistar) - J(\pihat) \lesssim R \cdot \dhel{\pp^{\pistar}}{\pp^{\pihat}}$.
\end{theorem}
In addition, as we recall in \Cref{app:il_prelims}, \citet{foster2024behavior} showed that the above reduction is tight, in the sense that for any two policies $\pistar$ and $\pihat$, there exists a reward function such that the Hellinger distance lower bounds the regret up to constant factors.  \Cref{thm:regret_hellinger} thus motivates us to focus on controlling the Hellinger distance between the trajectory distributions of the expert and learned policies in order to get good regret guarantees; because of this, we will often abuse nomenclature and refer to $\dhel{\pp^{\pistar}}{\pp^{\pihat}}$ as \emph{regret}.  Thus, the question we answer in this work is the following:
\begin{quotation}
  \noindent
    \emph{How much interaction with a noisy expert $\pistar_\eta$ is required in order to find a policy $\pihat$ such that $\dhel{\pp^{\pistar}}{\pp^{\pihat}} \leq \epsilon$ and how does online access to the expert affect this sample complexity?}
\end{quotation}

We will at times focus on the case where the corruption distribution $\nu$ satisfies a certain \emph{domination} condition with respect to the policy class $\Pi$, defined as follows.
\begin{definition}[$\kappa$-Domination]\label{def:kappa_domination}
    We say that a corruption distribution $\nu$ is $\kappa$-dominated by policies $\pi, \pi'$ if for any time step $h \in [H]$, state $s \in \cS$, and action $a \in \supp(\pi_h(\cdot \mid s)) \cup \supp(\pi_h'(\cdot \mid s))$, it holds that $\nu_h(a | s) \leq \kappa \cdot \left( \pi_h(a | s) + \pi_h'(a | s) \right)$.   
    We will say that $\nu$ is $\kappa$-dominated by a policy class $\Pi$ if for fixed $\pistar \in \Pi$, $\nu$ is $\kappa$-dominated by $\pistar$ and $\pi$ for any $\pi \in \Pi$.
\end{definition}
While this condition may appear somewhat unmotivated,\footnote{This assumption arises naturally due to the singularity of the map $\eta \mapsto \sqrt{\eta}$ near $\eta = 0$ (cf. \Cref{rmk:kappa_domination_necessity}).} we will show that it is necessary in order to provide horizon-free guarantees in the noisy expert setting. The condition is weaker than it may first appear, as the likelihood domination only needs to hold on the support of the policies $\pi$ and $\pi'$, which can be much smaller than the entire action space.  Indeed, note that it is always satisfied with $\kappa = 1$ for deterministic policies $\pi$, which we will study in detail in \Cref{sec:generalizations}.

%% file: body_clean/offline.tex
\section{Offline Imitation Learning with a Noisy Expert}\label{sec:offline}

We begin by considering the offline setting, where the learner only has access to a dataset of trajectories generated by the noisy expert $\pistar_\eta$, i.e. $\tau^{1}, \dots \tau^n \sim \pp^{\pistar_\eta}$.  In this and the next sections, we will focus on the \emph{known} corruption setting, where both $\eta$ and $\nu$ are known to the learner and generalize to \emph{unknown} corruptions in \Cref{sec:generalizations}. While the known corruption setting is often not realistic in practice, it allows the main ideas to be presented with significantly greater clarity, as well as providing a useful benchmark for the more difficult unknown corruption setting.  
\iftoggle{colt}{}{

}
In the classical regime, wherein the learner receives clean expert feedback, \citet{foster2024behavior} showed\iftoggle{colt}{:}{ the following result.}
\begin{theorem}[Theorem 2.1 \citet{foster2024behavior}]\label{thm:il_hellinger}
    Let $\pistar \in \Pi$ and suppose $\tau^1, \dots, \tau^n \sim \pp^{\pistar}$.  If
    \begin{align}\label{eq:bc}
        \pihat \in \argmin_{\pi \in \Pi} \sum_{i = 1}^n \sum_{h  =1}^H - \log \pi(a_h^{(i)} \mid s_h^{(i)}),
    \end{align}
    then with probability at least $1 - \delta$, it holds that $\dhel{\pp^{\pihat}}{\pp^{\pistar}} \lesssim \nicefrac{\log\left( \nicefrac{\abs{\Pi}}{\delta} \right)}{n}$.
\end{theorem}
This analysis is based on the key observation that \eqref{eq:bc} is the \emph{Maximum Likelihood Estimator} (MLE) of the trajectory distribution of the expert $\pistar$ based on the observed trajectories, and thus we can apply classical results on the convergence of MLEs in Hellinger distance \citep{geer2000empirical,zhang2006epsilon}. While this is a powerful result when given clean expert feedback, in the noisy expert setting considered here, we only have access to trajectories from $\pistar_\eta$ rather than $\pistar$, and thus, even when $\eta$ and $\nu$ are known, we can guarantee closeness in Hellinger distance only to $\pp^{\pistar_\eta}$ rather than $\pp^{\pistar}$. Indeed, it would thus be natural to replace \eqref{eq:bc} with the following estimator:
\begin{align}\label{eq:bc_noisy}
    \pihat \in \argmin_{\pi \in \Pi} \sum_{i = 1}^n \sum_{h  =1}^H - \log \pi_\eta(a_h^{(i)} \mid s_h^{(i)}),
\end{align} 
where $\pi_\eta = (1 - \eta) \cdot \pi + \eta \cdot \nu$ is the noisy version of $\pi$.  Critically, while the identical analysis as in \Cref{thm:il_hellinger} shows that $\dhel{\pp^{\pistar_\eta}}{\pp^{\pihat_\eta}}$ is guaranteed to be small, our ultimate goal is to compete with the \emph{clean} expert $\pistar$ rather than the noisy expert $\pistar_\eta$, and thus we need to understand how this noisy Hellinger distance controls the clean Hellinger distance $\dhel{\pp^{\pistar}}{\pp^{\pihat}}$\iftoggle{colt}{.}{, which is the content of the following result.} 
\begin{theorem}\label{thm:offline_bc}
    Let $\pi, \pi'$ be policies that $\kappa$-dominate  $\nu$.  Then,
    \iftoggle{colt}{$\dhelinline{\pp^{\pi}}{\pp^{\pi'}} \lesssim \left( 1 + \eta \cdot \kappa \right)\left( 1 - \eta \right)^{- H - 2} \cdot \dhelinline{\pp^{\pi_\eta}}{\pp^{\pi'_\eta}}$.}
    {
    \begin{align}
        \dhel{\pp^{\pi}}{\pp^{\pi'}} \lesssim \left( 1 + \eta \cdot \kappa \right)\left( 1 - \eta \right)^{- H - 2} \cdot \dhel{\pp^{\pi_\eta}}{\pp^{\pi'_\eta}}.
    \end{align}
    }
    In particular, if $\pihat$ is as in \eqref{eq:bc_noisy} then with probability at least $1 - \delta$, it holds that
    \iftoggle{colt}{$\dhel{\pp^{\pihat}}{\pp^{\pistar}}  \lesssim \left( 1 + \eta \cdot \kappa \right)\left( 1 - \eta \right)^{- H - 2} \cdot \nicefrac{\log(\nicefrac{\abs{\Pi}}{\delta})}{n}$.}
    {
    \begin{align}
         \dhel{\pp^{\pihat}}{\pp^{\pistar}}  \lesssim \left( 1 + \eta \cdot \kappa \right)\left( 1 - \eta \right)^{- H - 2} \cdot \frac{\log(\nicefrac{\abs{\Pi}}{\delta})}{n}.
    \end{align}
    }
    In general, absent $\kappa$-domination it holds that
    \iftoggle{colt}{
        $\dhel{\pp^{\pistar}}{\pp^{\pihat}} \lesssim \sqrt{\nicefrac{((1 - \eta)^{-H} - 1)}{\eta ( 1 - \eta)} \cdot \dhel{\pp^{\pihat_\eta}}{\pp^{\pistar_\eta}}}$.
    }
    {
        \begin{align}
            \dhel{\pp^{\pistar}}{\pp^{\pihat}} \lesssim \sqrt{\frac{(1 - \eta)^{-H} - 1}{\eta ( 1 - \eta)} \cdot \dhel{\pp^{\pihat_\eta}}{\pp^{\pistar_\eta}}}.
        \end{align}
    }
\end{theorem}
The proof is deferred to \Cref{app:offline_ub}, where the key idea is to apply backward induction by peeling off the trajectory distributions one step at a time, and then to use the $\kappa$-domination condition to control the error at each step.  Note that in the absence of $\kappa$-domination, as $\eta \downarrow 0$, we do not recover a tight dependence as a polynomial in $H$ factor appears, which is tight.

While \Cref{thm:offline_bc} shows that as long as $\eta$ and $\nu$ are known, a natural analogue of \algbc\ is consistent, the sample complexity is exponential in the horizon $H$ whenever $\eta \gg \nicefrac 1H$, even when the clean expert $\pistar$ is deterministic, in contradistinction to the clean setting where no explicit horizon dependence exists.  Thus in the long horizon regime, e.g. LMs generating long chains of thought in order to solve a hard math problem, this upper bound becomes vacuous.  Unfortunately, our next result shows that this exponential dependence is necessary.
\begin{proposition}\label{prop:offline_lower_bound}
    For any $H \geq 2$, $\kappa \geq 1$, and $\eta < 1$ there exists a horizon $H$ MDP with $3$ actions, policies $\pistar, \pihat$, and $\kappa$-dominated corruption distribution $\nu$ such that
    \iftoggle{colt}{$\dhel{\pp^{\pihat}}{\pp^{\pistar}} \gtrsim \eta \cdot \kappa \cdot (1 - \eta)^{- H - 1} \cdot \dhel{\pp^{\pihat_\eta}}{\pp^{\pistar_\eta}}$.}
    {
    \begin{align}
        \dhel{\pp^{\pihat}}{\pp^{\pistar}} \gtrsim \eta \cdot \kappa \cdot (1 - \eta)^{- H - 1} \cdot \dhel{\pp^{\pihat_\eta}}{\pp^{\pistar_\eta}}.
    \end{align}
    }
    Moreover, if $\kappa = 1$, we can take $\pistar$ to be deterministic. In the absence of $\kappa$-domination, we may even force
    \iftoggle{colt}{
        $ \dhel{\pp^{\pistar}}{\pp^{\pihat}} \gtrsim \sqrt{\nicefrac{((1-\eta)^{-H}  - 1)}{1 - \eta} \cdot \dhel{\pp^{\pihat_\eta}}{\pp^{\pistar_\eta}}}$.
    }
    {
        \begin{align}
            \dhel{\pp^{\pistar}}{\pp^{\pihat}} \gtrsim \sqrt{\frac{(1-\eta)^{-H}  - 1}{1 - \eta} \cdot \dhel{\pp^{\pihat_\eta}}{\pp^{\pistar_\eta}}}.
        \end{align}
    }
\end{proposition}

The constructions witnessing the above result can be found in \Cref{app:offline_lb}. This result demonstrates that even in the easiest possible setting of noisy IL, where the contamination is known and the expert is deterministic, any offline IL algorithm must have sample complexity that depends exponentially on the horizon $H$ in the worst case, presenting a fundamental barrier in the noisy expert setting.  Moreover, the result demonstrates the fundamental necessity of $\kappa$-domination.  Even worse, as we show in \Cref{prop:offline_il_eps_lb}, a corollary of this result demonstrates that \emph{any} offline IL algorithm must suffer from this exponential dependence in sample complexity, making offline IL fundamentally intractable in the noisy expert setting.  In the next section, we show that online IL can reduce this exponential dependence on the horizon to at most polynomial\iftoggle{colt}{}{ even in the presence of expert noise}.

%% file: body_clean/online.tex
\section{On-Policy Distillation Allows for Improved Guarantees}\label{sec:online}

In the previous section, we showed that in the noisy expert setting, offline IL algorithms such as \algbc\ must suffer from an exponential dependence on the horizon $H$, which is unacceptable in the kinds of long-horizon tasks that are ubiquitous in modern AI applications \citep{guo2025deepseek,shao2024deepseekmath,olmo20242}.  In this section, we demonstrate that \emph{online} access to the expert can circumvent this exponential dependence, in direct contradistinction to the classical setting where \citet{foster2024behavior} showed that when given clean expert feedback, \algbc\ is optimal even in the presence of \emph{online expert interaction}.

In order to motivate our main result of the section showing that online access can greatly improve the dependence on the horizon in the noisy expert setting, we first emphasize that the primary obstacle to offline IL in the noisy expert regime is precisely the tightness of \Cref{thm:offline_bc}: in order to recover the clean expert's trajectory distribution in Hellinger distance, we must have exponentially better control on the Hellinger distance between the trajectory distributions of the noisy expert and noisy learned policy.  While this relationship is unfortunately tight, it is natural to wonder if a stronger guarantee on the closeness of the noisy policy trajectories can lead to one paying a smaller constant factor in the relationship with the clean policy trajectories.  In order to show that this is indeed the case, we introduce the notation $\pp^{\pi, \pi'}$ for augmented trajectory distributions.  More precisely, for policies $\pi, \pi'$ we say that 
\iftoggle{colt}{$\tau' = (s_1, a_1',  \dots, s_H, a_H') \sim \pp^{\pi, \pi'}$ if $\tau = (s_1, a_1, \dots, s_H, a_H) \sim \pp^{\pi}$ and $a_h' \sim \pi'(\cdot | s_h)$.}
{
\begin{align}
    \tau' = (s_1, a_1',  \dots, s_H, a_H') \sim \pp^{\pi, \pi'} \text{ if } \tau = (s_1, a_1, \dots, s_H, a_H) \sim \pp^{\pi} \text{ and } a_h' \sim \pi'(\cdot | s_h).
\end{align}
}
In other words, $\pp^{\pi, \pi'}$ is the distribution over trajectories obtained by rolling out $\pi$ in the MDP but then taking actions according to $\pi'$ instead of $\pi$; note that $\pp^{\pi, \pi}$ is in general not the same as $\pp^{\pi}$ if $\pi$ is stochastic due to the resampling of actions, but they do have the same marginal distributions. The following result shows that augmented Hellinger distance on learner-visited states controls the clean trajectory Hellinger distance. 

\begin{theorem}
\label{thm:augmented_hellinger} 
Let $\pistar, \pihat$ \iftoggle{colt}{}{be two policies that} $\kappa$-dominate the corruption $\nu$.  Then for any $0 \leq \eta < 1$ it holds that
\begin{align} 
    \dhel{\pp^{\pistar}}{\pp^{\pihat}} \lesssim \frac{1+\eta \cdot \kappa}{(1-\eta)^2} \cdot \dhel{\pp^{\pihat,\pistar_\eta}}{\pp^{\pihat,\pihat_\eta}}. 
    \label{eq:augmented_hellinger_kappa} 
\end{align} 
Absent the \(\kappa\)-domination condition, it holds that 
\begin{align} 
    \dhel{\pp^{\pistar}}{\pp^{\pihat}} \lesssim \frac{\sqrt H}{1-\eta} \cdot \sqrt{ \dhel{\pp^{\pihat,\pistar_\eta}}{\pp^{\pihat,\pihat_\eta}} }. 
    \label{eq:augmented_hellinger_no_kappa} 
\end{align} 
\end{theorem}
We defer the proof of this result to \Cref{app:augmented_comparisons}. The main technical point is a comparison inequality that uses induction to show that for any policies $\pi, \pi'$, we can control $\dhel{\pp^\pi}{\pp^{\pi'}}$ in terms of the Hellinger distance of the \emph{augmented} trajectory distributions $\pp^{\pi, \pi}$ and $\pp^{\pi, \pi'}$; critically, because the state transitions are shared, we can use the same rollout policy and compare only the distributions of the augmented labels.  We can then use simple information theoretic arguments to control the clean policy distance by the noisy policy distance, which is where the $\kappa$-domination condition becomes relevant. 

In \Cref{prop:no_kappa_augmented_hellinger_lower_bound}, we  further show that \eqref{eq:augmented_hellinger_no_kappa} is tight up to a polynomial dependence on $(1 - \eta)$, and thus the improvement from \Cref{thm:augmented_hellinger} is indeed significant in the presence of $\kappa$-domination.  
Through the application of Pinsker's inequality (\Cref{prop:pinsker}), we achieve the following corollary. 
\begin{corollary}
\label{cor:kl_augmented}
    Let $\pistar, \pihat$ \iftoggle{colt}{}{be two policies that} $\kappa$-dominate the corruption $\nu$.  Then for any $0 \leq \eta < 1$ it holds that
    \begin{align}\label{eq:kl_augmented}
        \dhel{\pp^{\pistar}}{\pp^{\pihat}} \lesssim \frac{1 + \eta \cdot \kappa}{(1 - \eta)^2} \cdot \left(\kld{\pp^{\pihat, \pistar_\eta}}{\pp^{\pihat, \pihat_\eta}} \wedge  \kld{\pp^{\pihat,\pihat_\eta}}{\pp^{\pihat,\pistar_\eta}}\right).
    \end{align}
    Absent the $\kappa$-domination condition, it holds that
        \begin{align}
        \label{eq:kl_augmented_no_domination}
        \dhel{\pp^{\pistar}}{\pp^{\pihat}} \lesssim (1 - \eta)^{-1} \cdot \sqrt{H \cdot \left(\kld{\pp^{\pihat, \pistar_\eta}}{\pp^{\pihat, \pihat_\eta}} \wedge  \kld{\pp^{\pihat,\pihat_\eta}}{\pp^{\pihat,\pistar_\eta}}\right)}.
    \end{align}
\end{corollary}
We emphasize that, unlike the offline setting, \Cref{thm:augmented_hellinger,cor:kl_augmented} are fully \emph{independent of horizon} in the presence of $\kappa$-domination and scales only polynomially with $H$ in general in contradistinction to standard offline IL, where online and offline methods achieve the same statistical rates in the worst case \citep{foster2024behavior}. 

The KL upper bounds in \Cref{cor:kl_augmented} are closely related to the popular \emph{On-Policy Distillation} (OPD) algorithm \cite{agarwal2024policy,lu2025onpolicydistillation}, which is used to finetune a student policy to match the behavior of a teacher policy by minimizing a loss closely related to the right hand side of \eqref{eq:kl_augmented}.  Indeed, a common approach is to roll out a policy $\pi_t$ and then use a target $\pistar_\eta$ to score the actions taken by $\pi_t$, in particular attempting to minimize $\kldinline{\pp^{\pi_t}}{\pp^{\pi_t, \pistar_\eta}}$.  While our analysis suggests instead minimizing divergence on the \emph{augmented trajectory distributions}, this provides some explanation for why OPD can be so much more successful than the offline \algbc\ (i.e., SFT) in practice.

While variants of OPD remain the recommended empirical approach (\iftoggle{colt}{}{as we shall see in }\Cref{sec:experiments}) and \Cref{cor:kl_augmented} provides theoretical backing for the success thereof, we continue in this section by introducing a simple new algorithm, \nail\ (\Cref{alg:nail}).  At its core, much like the classical \algdagger\ algorithm \citep{ross2011reduction}, \nail\ uses online learning to aggregate policies that are trained on the noisy expert's behavior on trajectories rolled out by the learner.  Unlike some instantiations of \algdagger, however, \nail\ is \emph{always on-policy}, in the sense that it rolls out the current policy mixture $\mu_t$ at each round \citep{ross2011reduction}.  Moreover, \nail\ learns the policy at a \emph{trajectory level}, thereby circumventing the horizon dependence that \algdagger\ necessarily incurs by learning a different policy at each time step $h$ \citep{ross2011reduction,foster2024behavior}.  More precisely, \nail\ begins with a uniform distribution $w_1$ over policies in $\Pi$ and at each round $t$ rolls out the mixture policy $\mu_t = \sum_{\pi \in \Pi} w_t(\pi) \cdot \pi$ to get a trajectory $\tau^{(t)}$, queries the noisy expert $\pistar_\eta$ on $\tau^{(t)}$ to get an augmented trajectory $\tau'^{(t)}$, and then updates $w_{t+1}$ using the standard exponential weights update \citep{cesa2006prediction} (cf. \Cref{app:online_learning}) with the loss given by the negative log-likelihood of $\tau'^{(t)}$ under each policy $\pi$.  Finally, \nail\ returns a random policy across the time steps.  We show that \nail\ achieves the following regret guarantee in the noisy expert setting.

\begin{algorithm}[t]
\caption{\nail: Noise-robust Aggregation for Imitation Learning}
\label{alg:nail}
\begin{algorithmic}[1]

\Require Number of rounds $n$, policy class $\Pi$, noisy expert $\pistar_\eta$, corruption level $\eta$, corruption distribution $\nu$.
\State Initialize $w_1 = \Unif(\Pi)$.
\For{$t = 1$ to $n$}
    \State Define $\mu_t = \sum_{\pi \in \Pi} w_t(\pi) \cdot \pi$ and deploy $\mu_t$ to get trajectory $\tau^{(t)} \sim \pp^{\mu_t}$.
    \State Query noisy expert $\pistar_\eta$ on $\tau^{(t)}$ to obtain augmented trajectory $\tau^{(t)'}$.
    \State Update $w_{t+1}(\pi) \propto w_t(\pi) \cdot \left(\prod_{h = 1}^H \pi_\eta(a_h^{(t)'} \mid s_h^{(t)})\right)^{\nicefrac 1H}$.
\EndFor

\State \Return $\pihat = \mu_T$ where $T \sim \Unif([n])$.

\end{algorithmic}
\end{algorithm}

\begin{theorem}\label{thm:nail}
    Let $\pistar \in \Pi$ be an arbitrary policy, $\nu$ be an arbitrary corruption distribution, and $0 \leq \eta < 1$ be a corruption level.  If $\pihat$ is the policy returned by \nail\ (\Cref{alg:nail}) after $n$ rounds with feedback from $\pistar_\eta$, then
    \iftoggle{colt}{
        $\ee\left[ \dhel{\pp^{\pistar}}{\pp^{\pihat}} \right]\lesssim \nicefrac{H}{1 - \eta} \cdot \sqrt{\nicefrac{\log(\abs{\Pi})}{n}}$.
    }{
    \begin{align}
        \ee\left[ \dhel{\pp^{\pistar}}{\pp^{\pihat}} \right] \lesssim \nicefrac{H}{1 - \eta} \cdot \sqrt{\nicefrac{\log(\abs{\Pi})}{n}}.
    \end{align}
    }
\end{theorem}
We defer a proof of this result to \Cref{app:nail_proof}, which proceeds by using the \emph{mixability} (cf. \Cref{app:online_learning}) of the trajectory-level log-loss to control the sum of on-policy trajectory-level KL divergences across rounds $1 \leq t \leq n$, before applying \Cref{cor:kl_augmented} and the convexity of the Hellinger distance to get the desired bound.  We reiterate that \nail\ should be thought of as a theoretical algorithm that is designed to be optimal in the worst case and that directly minimizing the right hand side of \eqref{eq:kl_augmented} is a more practical approach to OPD. We further remark that, while \nail\ uses exponential weights, really any online learning algorithm with small regret for the trajectory-level log-loss could be used in its place, as \Cref{cor:kl_augmented} allows for a generic reduction to online learning in the spirit of \citet{ross2011reduction}.   
\iftoggle{colt}{}
{

}
We now show that \nail\ is optimal up to factors in $H$ and $\eta$. 
\begin{proposition}\label{prop:onlin_lb}
    For any $H \geq 2$, $0 < \eta < 1$, and small $\epsilon$ there exists a horizon $H$ MDP with $3$ actions, a policy class $\Pi$ of size $\abs{\Pi} = 2$ containing $\pistar$, and known corruption $\nu$ such that any algorithm with online access to the noisy expert $\pistar_\eta$ requires $n \gtrsim \nicefrac{\eta \cdot H}{(1 - \eta)^2 \cdot \epsilon^2}$ in order to obtain regret $\epsilon$.
\end{proposition}
We defer a proof to \Cref{app:online_lb_pf}. We emphasize that this lower bound does not preclude a horizon-free guarantee under $\kappa$-domination, and we show in \Cref{app:horizon_free_testing} that such a guarantee is indeed possible, albeit with a much worse dependence on $\abs{\Pi}$. We leave it as an open question whether there exists a computationally efficient procedure that achieves horizon-free guarantees under $\kappa$-domination or without $\kappa$-domination. 

To summarize: there is a marked distinction between the clean and noisy expert settings: in contradistinction to standard IL, when the expert is noisy, online interaction can exponentially improve over offline access to the expert.

%% file: body_clean/generalizations.tex
\section{Imitating Deterministic Experts with Unknown Corruptions}\label{sec:generalizations}

In the previous sections, we focused on the setting of noisy IL with a \emph{known corruption}.  While this setting is simple, it is somewhat unrealistic; we now relax this assumption.  
The first problem we run into is one of \emph{identifiability}: without additional assumptions on the corruption, it may well be the case that it is impossible to identify the expert policy $\pistar$ from the noisy expert $\pistar_\eta$, even absent stochasticity in $\pistar$; a concrete example of this phenomenon is provided in \Cref{app:identifiability}. 
While this does not present an issue in standard IL, where closeness to the observed actions is the key desideratum, in the noisy expert setting, we care about learning a policy that is close to the \emph{clean} expert and thus identifiability is critical.  This issue does not arise when $\nu$ and $\eta$ are known, as clean experts can be recovered from their noisy analogues, but recovery is not possible absent known corruption.  

While many identifiability assumptions are possible, we focus on a natural one, satisfied in many of our motivating applications. We first restrict our focus to \emph{deterministic experts}, motivated by applications to reasoning, math, or coding in LMs, where we think of the true expert as producing a fixed `correct' output \citep{shao2024deepseekmath,guo2025deepseek}.  We further assume that the corruption distribution $\nu$ is `smooth.'
\begin{definition}\label{def:rho_smooth}
    We say that a corruption distribution $\nu$ is $\rho$-smooth if for all time steps $h \in [H]$, states $s \in \cS$, and actions $a \in \cA$, it holds that $\nu_h(a | s) \leq \rho$.
\end{definition} 
Smoothness is a natural assumption when the action space is large, for example when we think of corruption as arising from typos or other minor errors in the output of an LM or human demonstrator.  In particular, if $\nu$ is a uniform distribution over a large action space, then $\nu$ is $\rho$-smooth with $\rho = \abs{\cA}^{-1}$.  Finally, we assume a bound on the noise level such that $\eta \leq \alpha < 1$ to ensure that a \emph{margin} exists between the expert and the noise, which is necessary for identifiability.\footnote{We expect that we could replace determinism and smoothness with a more general margin condition and recover the same results, but we leave this for future work.}

We propose an algorithm, \gnail\ (\Cref{alg:gnail}), similar to \nail, but that is capable of achieving \emph{horizon-free} regret guarantees in the noisy expert setting with unknown corruption under the above assumptions.  While we defer a detailed description of the algorithm to \Cref{app:gnail_description}, we note that it is very similar to \nail, except that we replace the mixture policy $\mu_t$ by its greedy analogue $\mubar_t$, which plays the action with the highest probability under $\mu_t$ at each state and time step; moreover, the precise form of the update is different in order to account for the fact that we do not know $\eta$ and $\nu$.  \iftoggle{colt}{}{We have the following guarantee.}

\begin{theorem}\label{thm:gnail}
    Suppose that $\pistar \in \Pi$ is deterministic, $\nu$ is $\rho$-smooth, and $\eta \leq \alpha$ for some $\alpha$ satisfying $\alpha (1 + \rho) < 1$. There exists an algorithm (\Cref{alg:gnail}) without knowledge of $\eta$ or $\nu$ that returns a policy $\pihat$ after $n$ rounds of online interaction
    such that
    \iftoggle{colt}{$\ee\left[ \dhel{\pp^{\pistar}}{\pp^{\pihat}} \right] \lesssim \nicefrac{\log(\abs{\Pi})}{(1 - \alpha(1 + \rho))^2 \cdot n}$.}
    {
    \begin{align}
        \ee\left[ \dhel{\pp^{\pistar}}{\pp^{\pihat}} \right] \lesssim \frac{\log(\abs{\Pi})}{(1 - \alpha(1 + \rho))^2 \cdot n}.
    \end{align}
    }
\end{theorem}
We defer a proof of \Cref{thm:gnail} to \Cref{app:gnail_proof}, which proceeds through an intricate analysis of the update rule and how the margin condition that $\alpha (1 + \rho) < 1$ allows us to control the error without knowledge of $\eta$ or $\nu$.  Note that unlike in the previous section, the particular form of the update rule is critical to achieving this guarantee, and thus we do not prove a generic reduction to online learning.
\iftoggle{colt}{}{

}
We now conclude this section with a matching lower bound\iftoggle{colt}{.}{, demonstrating the optimality of \gnail\ up to constant factors.}
\begin{proposition}\label{prop:gnail_lb}
    For any $0 < \alpha  < 1$ and any $\rho$ satisfying $\alpha (1 + \rho) < 1$, there exists a horizon $H = 2$ MDP and a \emph{deterministic} policy class $\Pi$ of size $\abs{\Pi} \lesssim \nicefrac 1\rho$ such that in order to achieve expected regret at most $\epsilon$, the learner requires
    \iftoggle{colt}{$n \gtrsim \nicefrac{\alpha \rho }{\epsilon (1 - \alpha(1 + \rho))^2}$}
    {
    \begin{align}
        n \gtrsim \frac{\alpha \rho }{\epsilon (1 - \alpha(1 + \rho))^2}
    \end{align}
    }
    rounds of interaction with an online noisy expert satisfying $\eta \leq \alpha$ and $\nu$ is $\rho$-smooth and \emph{unknown}.
\end{proposition}
The proof of this lower bound can be found in \Cref{app:gnail_lb_proof}.
In particular, in the noisy expert regime where $\alpha$ and $\rho$ are treated as constants, \gnail\ is optimal\iftoggle{colt}{}{ up to constant factors}, even for horizon $H = 2$.

%% file: body_clean/experiments.tex
\section{Experiments}\label{sec:experiments}

We complement our theory with a small suite of experiments on synthetic and natural language tasks, showing that the OPD variant suggested by our theory (\nail) can outperform both offline BC and existing OPD objectives under noisy expert feedback while remaining competitive under clean expert feedback.  In particular, given an expert $\pistar_\eta$, we compare five algorithms: (i) \algllbc (SFT), which minimizes the negative log-likelihood of the noisy expert trajectories; (ii) \opdf, the standard OPD algorithm minimizing the forward KL divergence between student and teacher trajectory distributions \citep{agarwal2024policy}; (iii) \opdr, the OPD algorithm that minimizes the reverse KL divergence between student and teacher trajectory distributions \citep{lu2025onpolicydistillation,agarwal2024policy}; (iv) \nailf, the variant of \nail\ minimizing forward KL divergence between the augmented trajectory distributions; and (v) \nailr, the variant of \nail\ minimizing reverse KL divergence.  In particular, the primary difference between OPD and \nail is that \nail always generates \emph{greedily} from the student policy, whereas OPD samples trajectories from the student policy; we emphasize that we do not instantiate the exponential weights update in \Cref{alg:nail} but instead use gradient descent to minimize the upper bound in \Cref{cor:kl_augmented}. We consider two tasks: (i) modular addition\iftoggle{colt}{}{, a synthetic task in which the model is trained to report an explicit chain of intermediate computations,} and (ii) GSM-8K \citep{cobbe2021training}\iftoggle{colt}{}{, a natural language mathematical reasoning benchmark, both outlined below}.  Additional details are in \Cref{app:add_experimental_details}.

\subsection{Modular Addition}

We begin with a synthetic task from \citet{li2024chain}, that of modular addition.  Namely, for fixed $p$ and $m$, given a sequence of $m$ integers from $[p]$, the goal is to compute the sum of these integers modulo $p$.  In the cited work, the authors demonstrated that low-depth transformers empirically benefit from a chain of thought (CoT) reasoning process in this task, making it a suitable testbed for studying the effect of horizon.  All models are nanoGPT \citep{Karpathy2022}.  We train our expert $\pistar$ on a large corpus of data with $p = 7$ and $m = 31$ until its accuracy (with CoT) is approximately perfect.  We then construct a noisy expert $\pistar_\eta$ by letting $\nu$ be uniform over tokens and corrupting the expert with probability $\eta$ at each token if we have not yet seen an error, and corrupting all remaining tokens with probability $1$ if we have already seen an error.  We then train the student models using the five algorithms described above and evaluate the accuracy of the final learned policy $\pihat$ on a held-out test set of examples.  We defer further details to \Cref{app:synthetic_task_details}.
\Cref{fig:main_modadd_gsm8k}(left) summarizes the results of this experiment.  For the clean expert ($\eta = 0$) regime, we see that most methods work, with offline BC (i.e., SFT) learning the fastest, which is consistent with the fact that offline BC is optimal in the clean expert setting \citep{foster2024behavior}.  For the noisy expert ($\eta = 0.2$) regime, we see a sharp separation between our proposed methods (\nailf, \nailr) and the baselines (\algllbc, \opdf, \opdr), with the former achieving perfect accuracy and the latter failing to learn anything, even after over 1M expert interactions. We defer further comments and ablations to \Cref{app:add_experimental_discussion} and \Cref{app:add_experiments_and_ablations}.

\vscomment{TODO: mention comments and ablations here. Add ablation plots for mixed KL in the main body?}

\subsection{Mathematical Reasoning}

We also consider a standard mathematical reasoning task, \gsmk\ \citep{cobbe2021training}.  Here, we move beyond the realizable setting of our theory and use \gemmab\ as the teacher and \gemmam\ as the student \citep{team2024gemma}.  To instantiate clean and noisy expert feedback with the same underlying model, we consider temperature 1 sampling as the `clean' expert, which achieves 59.89\% accuracy on GSM-8K test, and temperature 4 sampling as the `noisy' expert, which achieves 0\% accuracy.  All training prompts are drawn from TinyGSM \citep{liu2023tinygsm} and we evaluate zero-shot, exact-match accuracy on the \gsmk\ test set with greedy decoding, reporting results over three random seeds.  We defer further details to \Cref{app:math_task_details}.
We report results in \Cref{fig:main_modadd_gsm8k}(right), where we see that all online methods substantially outperform \algllbc\ even in the clean expert setting, consistent with results from \citet{agarwal2024policy,lu2025onpolicydistillation}.  In the noisy expert setting, only \nailf\ and \nailr\ remain effective, achieving nontrivial performance while \algllbc, \opdf, and \opdr\ all fail to learn anything, maintaining 0 accuracy, like the noisy expert, throughout training.  

\vscomment{TODO: mention comments and ablations here. Add ablation plots for mixed KL in the main body?}

%% file: body_clean/app_add_related_work.tex
\section{Discussion} \label{app:add_related_work}
In this work we investigated the problem of imitation learning with noisy experts, where the learner only has access to a corrupted version of the expert's policy.  We showed theoretically that while offline imitation learning can be consistent in this setting, it suffers from an exponential dependence on the horizon, making it fundamentally intractable in the long horizon regime.  On the other hand, we showed that online imitation learning can circumvent this exponential dependence and achieve horizon-free guarantees in some cases and polynomial in horizon guarantees in general, in direct contradistinction to the more classical IL setting with clean experts.  Moreover, we provided formal theoretical justification for the benefits of the commonly used \emph{On-Policy Distillation} (OPD) algorithm and validated our theory with a small suite of experiments on synthetic and natural language tasks.  A summary of our theoretical results is in \Cref{tab:general_results}.  In this section, we discuss related work and future directions for research.
\input{body_clean/table_summary.tex}

\subsection{Limitations}

While our work presents a clean theoretical picture, there are several limitations in immediately translating our results to practice.  First, consistent with the literature in RL and IL, our results are all in the finite $\Pi$ setting \citep{ross2011reduction,ross2010efficient,foster2024behavior,rohatgi2025computational}; standard arguments can be used to extend our results to infinite $\Pi$ under covering conditions.  Second, we adopt an idealized corruption model, which may not capture all the nuances of real-world expert noise; we defer to future work the task of developing more sophisticated corruption models that may better capture the types of noise observed in practice. Third, many of our results are in the known corruption setting, and we only provide results in the unknown corruption setting when the expert is deterministic; while this setting is relevant for many of our motivating applications, e.g. LMs learning to perform long chains of thought, it would be interesting to understand the unknown corruption setting more generally. Fourth, there remains a linear in horizon gap between our upper bound for arbitrary experts and our lower bound; closing this gap would be an interesting direction for future work. Similarly, whether or not fully horizon-free rates with logarithmic dependence on $|\Pi|$ are possible for online IL in the presence of $\kappa$-domination is an interesting question. Finally, due to resource constraints, our experiments are limited in scope and scale; it would be interesting to conduct a more comprehensive empirical investigation of the phenomena we identify here, especially in the context of training larger language models for longer.

\subsection{Related Work}

\paragraph{Imitation Learning.} Behavior cloning was first introduced empirically in \citet{pomerleau1988alvinn} and has since engendered the field of Imitation Learning, with a large body of work considering both theoretical and empirical aspects of the problem \citep{ross2010efficient,ho2016generative,block2023provable,chi2025diffusion,barreiros2026careful}.  Of particular note is \citet{ross2011reduction}, which introduced the online IL algorithm \algdagger{}; while \nail\ is conceptually similar (asking for expert feedback to correct potential student mistakes), the analysis and guarantees are quite different, with \algdagger\ attempting to reduce horizon dependence in a clean expert setting. 

Other IL algorithms study different feedback and data-collection models. \algsmile\ \citep{ross2010efficient} works in a similar fashion as \algdagger\ to collect learner-induced states under clean expert feedback, but instead learns a sequence of policies and mixes them to control distribution shift. \algaggrevate\ \citep{ross2014reinforcement} strengthens the feedback model by using expert
cost-to-go or advantage information, and \alggail\ \citep{ho2016generative} instead casts IL as adversarial occupancy-measure matching from fixed expert demonstrations. \algdart\ \citep{laskey2017dart}
collects off-policy demonstrations from a noise-injected expert so that the
demonstrations include recovery behavior. 

More recently, \citet{foster2024behavior} provided a comprehensive analysis of behavior cloning in the realizable, clean expert setting and \citet{rohatgi2025computational} extended this to the agnostic setting.  In the online setting, \citet{li2024interactive} considered a variant of DAgger where the learner pays per-state instead of per-trajectory in expert feedback, while \citet{li2023agnostic} also considers the agnostic setting but with strong assumptions on the MDP structure and a horizon dependence in their guarantees.  In contrast to these works, we consider a more challenging setting with noisy experts and show that online IL can achieve horizon-free guarantees in some settings and polynomial dependence on horizon in general even in this more difficult setting, while offline IL suffers from an exponential dependence on the horizon.  On the empirical side, OPD has seen recent success in training large language models (LLMs) to perform long chains of thought \citep{bai2023qwen,agarwal2024policy,lu2025onpolicydistillation}, and our theoretical results provide formal justification for the benefits of OPD in this setting.

The connections between IL and LMs have been explored in a number of recent works \citep{block2024butterfly,chang2023learning,rohatgi2025computational,foster2024behavior}, with the deterministic expert setting being particularly relevant to LMs learning long chains of thought \citep{joshi2025theory,altabaa2025cot}.  Our results generalize beyond this setting, allowing for stochastic experts and corrupted feedback.

\paragraph{Noisy Expert Feedback.} Most prior work in IL assumes access to a clean expert, with the notable exceptions of \citet{sekhari2023selective,sekhari2023contextual}.  While a promising start, and the former work provides an exponential in horizon lower bound for offline learning in a particular setting, these works make strong assumptions on the noise and MDP structure and only apply to specific policy classes; moreover, they only consider a very specific type of corruption arising from preference-based feedback, in contradistinction to our bounds which apply to arbitrary policy classes and MDPs. 
Another related work that concerns a noisy-expert model is that of \citet{swamy2022causal}, which focuses (under very strong assumptions) on causal IL under temporally correlated noise, where action noise changes future states and creates spurious state-action correlations that BC can learn. Our corruption model has the same rolled-in character, but our focus is different: rather than debiasing offline demonstrations via causal assumptions, we show that online local queries can avoid learning from fully contaminated trajectories.

\paragraph{Distillation.} Knowledge Distillation is typically the process of training a student model to match a stronger teacher, emphasizing compression and efficient deployment. Early work emphasized transferring softened teacher predictions from cumbersome ensembles to smaller students \cite{hinton2015distilling}, and the same idea was later extended to reinforcement-learning policies \cite{rusu2015policy} and to autoregressive sequence models through sequence-level distillation \cite{kim2016sequence}. Critically, this continued as an off-policy paradigm: the student is trained on a fixed transfer set, teacher-generated sequences, or teacher logits, as in \cite{sanh2019distilbert, jiao2020tinybert, wang2020minilm}; see also the general survey of \cite{gou2021knowledge}. While highly effective for compression, these methods do not train on prefixes induced by the student's own generations, and are therefore susceptible to the train--test mismatch inherent in autoregressive generation, a well known problem in imitation learning \cite{pomerleau1988alvinn, ross2010efficient}.

On-policy distillation \cite{agarwal2024policy} addresses this mismatch by training the student on self-generated sequences while receiving teacher feedback on those sequences using varying discrepancy measures between teacher and student. One objective in distillation is the forward KL from the teacher to the student, which corresponds in our terminology to \opdf. Prior work of \citet{agarwal2024policy} has noted that this objective can be problematic when the student is not expressive enough to fit the teacher distribution: because forward KL is mode-covering, the resulting student may place probability on generations that are unlikely under the teacher.  Subsequent work \cite{guminillm, lu2025onpolicydistillation} has therefore explored alternative on-policy objectives, advocating for sequence/trajectory-level reverse KL to avoid the mode-covering behavior of forward KL in generative settings, while \cite{ko2024distillm} smooths the loss and reduces rollout cost with an adaptive off-policy reuse scheme. 

Our results offer a more unifying perspective. In the noisy-expert setting, the difficulty is not only which KL direction is used, but also which prefix distribution the teacher is queried on and which feedback is informative. Sampling from the student can itself induce noisy or off-distribution prefixes, reducing the usefulness of local teacher feedback. This helps explain empirical findings that sampled-token OPD can be brittle, and that stability can improve by restricting both the rollout distribution and the teacher signal used for local updates \citep{fu2026revisiting}. It is also consistent with recent work that makes OPD updates more selective, for example by down-weighting teacher KL on failed trajectories where the teacher itself is uncertain \citep{zheng2026scope}. More generally, work on pivotal or ``forking'' tokens suggests that a small number of CoT tokens can determine which reasoning branch the model follows, and hence whether downstream feedback remains informative \citep{abdin2024phi,wang2025beyond}. Our theory gives a principled prescription: matching teacher feedback along the learner's deployed prefix distribution; in our experiments with deterministic experts, using greedy student rollouts together with augmented trajectory-level KL objectives helps suppress sampling-induced noise and learn more effectively from imperfect teachers. We leave a more detailed study of when forward versus reverse KL is preferable to future work. More recently,  self-play and self-distillation methods have shown that on-policy data generation can itself be a source of improvement even without a separate external teacher \cite{chen2024self, zhao2026self, hubotter2026reinforcement, yang2026self, zhang2026embarrassinglysimpleselfdistillationimproves}. For a more detailed overview, we refer to the recent survey of \citet{song2026survey}.

\subsection{Future Directions}
While we provided a near-tight theoretical analysis of the problems studied, there remain many interesting open questions for future work.  First, in the unknown corruption setting, we focused on the case of deterministic experts and smooth corruptions with margin, but it would be interesting to understand the extent to which these assumptions can be relaxed while still achieving horizon-free guarantees, or even consistency. Second, while we considered the natural linear corruption setting, motivated by classical statistical theory \citep{ronchetti2009robust} and applications to LMs, alternative corruption models are also interesting to study, for example trajectory-level, multiplicative, or token-adaptive corruptions that target semantically pivotal ``fork'' tokens in long CoT reasoning \cite{abdin2024phi, chen2025coverage, wang2025beyond, zheng2026scope}.
Third, our online IL guarantees are \emph{in expectation} due to the inherited guarantees from online learning with log loss, whereas the offline guarantees are \emph{with high probability} and resolving this discrepancy could provide additional insight.  Finally, scaling up the experimental results to more complex environments and real-world tasks remains an important direction for future work.

\iftoggle{colt}{
\subsection{Broader Impacts} 
\label{app:broader_impacts}

This work is primarily theoretical and aims to improve our understanding of imitation learning from imperfect expert feedback. We hope that showing fundamental separation between offline and online methods (such as on-policy distillation) will lead to more reliable and sample-efficient post-training methods, especially in long-horizon settings. 

While our methods could be seen as enabling effective learning from imperfect teachers, they could also make it easier to train models for harmful applications if deployed without appropriate safeguards. Our experiments are limited to synthetic tasks and standard mathematical reasoning benchmarks, and we do not release deployed systems or high-risk models. We therefore view the main impact of this work as methodological, with downstream risks depending on how future systems built using these ideas are trained, evaluated, and released.
}{}

%% file: body_clean/table_summary.tex
\newcommand{\KLaug}{{\mathsf{KL}_{\mathrm{aug}}}}
\newcommand{\Haug}{{\mathsf{H}_{\mathrm{aug}}}}

\begin{table}[!p]
\centering
\caption{Summary of the main theoretical results in the general noisy-expert setting. Here
\(\Haug(\widehat\pi)\) denotes
\(\dhel{\pp^{\widehat\pi,\pistar_\eta}}{\pp^{\widehat\pi,\widehat\pi_\eta}}\), and
\(\KLaug(\widehat\pi)\) denotes the smaller of the two corresponding augmented-trajectory KL divergences appearing in
\Cref{cor:kl_augmented}. Bounds control \(\dhel{\pp^{\pistar}}{\pp^{\pihat}}\) and are stated up to universal constants and logarithmic factors when indicated.}
\label{tab:general_results}

\begin{adjustbox}{max totalsize={\textwidth}{0.78\textheight},center}
\begin{minipage}{\textwidth}
\small
\setlength{\tabcolsep}{3pt}
\renewcommand{\arraystretch}{1.08}
\begin{tabularx}{\textwidth}{p{0.1\textwidth} p{0.26\textwidth} p{0.31\textwidth} p{0.21\textwidth}}
\toprule
\textbf{Regime} & \textbf{Setting} & \textbf{Bound} & \textbf{Reference} \\
\midrule
\multicolumn{4}{l}{\textbf{Offline}} \\
\midrule

&
No \(\kappa\)-domination.
&
\[
\sqrt{
\frac{(1-\eta)^{-H}-1}{\eta(1-\eta)}
\cdot
\frac{\log(|\Pi|/\delta)}{n}
}.
\]
&
\Cref{thm:offline_bc,prop:offline_lower_bound}; this is tight.
\\[2.5em]

&
$\kappa$-domination.
&
\[
(1 + \eta \cdot \kappa) (1 - \eta)^{-H-2} \cdot \frac{\log(|\Pi|/\delta)}{n}.
\]
&
\Cref{prop:offline_lb_kappa}; this is tight.
\\

\midrule
\multicolumn{4}{l}{\textbf{Online}} \\
\midrule

&
Augmented-Hellinger comparison without $\kappa$-domination.
&
\[
\frac{1}{1-\eta}
\sqrt{H \cdot \Haug(\widehat{\pi})}.
\]
&
\Cref{thm:augmented_hellinger}; \Cref{prop:no_kappa_augmented_hellinger_lower_bound} shows the \(\sqrt H\) dependence is necessary.
\\[2.5em]

&
Augmented-Hellinger comparison with $\kappa$-domination. 
&
\[
\frac{1 + \eta \cdot \kappa}{(1 - \eta)^2} \cdot \Haug(\widehat{\pi}).
\]
&
\Cref{thm:augmented_hellinger}; applying Pinsker gives the KL form in \Cref{cor:kl_augmented}.
\\[2.5em]

&
NAIL / trajectory-level online aggregation
&
\[
\frac{H}{1-\eta}
\sqrt{\frac{\log|\Pi|}{n}}.
\]
&
\Cref{thm:nail}. The lower bound in \Cref{prop:onlin_lb} is a factor of \(H\) and \(\eta\) away from this upper bound.
\\[3em]

&
Deterministic \(\pi^\star\), \(\rho\)-smooth \(\nu\), and \(\eta\leq \alpha\) with \(\alpha(1+\rho)<1\)
&
\[
\frac{\log|\Pi|}{(1-\alpha(1+\rho))^2 n}.
\]
&
\Cref{thm:gnail,prop:gnail_lb}; this is tight up to a log factor.
\\

&
Candidate testing with \(\kappa\)-domination
&
\(\displaystyle
\frac{1+\eta\kappa}{(1-\eta)^2}
\cdot
\frac{|\Pi|\log(|\Pi|/\delta)}{n}
\)
&
\Cref{thm:kappa_candidate_testing}; horizon-free but linear in \(|\Pi|\).
\\

&
Candidate testing without \(\kappa\)-domination
&
\(\displaystyle
\frac{1}{1-\eta}
\sqrt{\frac{H \cdot |\Pi|\log(|\Pi|/\delta)}{n}}
\)
&
\Cref{thm:candidate_testing_no_domination}; optimal horizon dependence but linear in $\abs{\Pi}$.
\\

\bottomrule
\end{tabularx}
\end{minipage}
\end{adjustbox}
\end{table}

%% file: body_clean/app_additional_experimental_results.tex
\section{Further Empirical Results}
\label{app:add_experimental_results}

This section provides additional details and discussion for the experiments in \Cref{sec:experiments}. 
\Cref{app:add_experimental_details} describes the implementation of the methods in the autoregressive LM setting, as well as the synthetic and reasoning tasks setup. \Cref{app:add_experimental_discussion} gives a more detailed interpretation of the main empirical results, including the differences between clean and noisy feedback and between greedy and sampled rollouts. 
Finally, \Cref{app:add_experiments_and_ablations} reports additional experiments and ablations.

\subsection{Further Experimental Details}\label{app:add_experimental_details}

In this section, we provide implementation details for the empirical results in \Cref{sec:experiments}. 
\Cref{app:empirical_methods} defines the offline and online learning objectives that we compare, including behavior cloning, standard OPD, and our $\nail$-inspired variants; in \Cref{app:method_gradient_estimators} we spell out their gradient estimators. \Cref{app:synthetic_task_details} describes the setup of the synthetic CoT tasks, including the data-generation procedure, corruption model, architectures, and training details. 
Finally, \Cref{app:math_task_details} gives the corresponding details for the GSM8K experiments, including the teacher and student models, sampling procedure, training setup, and evaluation protocol.

We used NVIDIA RTX 6000 Pro GPUs for the modular addition experiments, with each run taking 1 hour. In total, the modular addition experiments used 50 GPU hours. For the GSM8K experiments, we used NVIDIA A100 40GB GPUs. Each online-method run took approximately 24 hours, while each offline-method run took approximately 3 hours; this disparity in runtime is an artifact of our highly suboptimal hardware setup and a na{\"i}ve implementation and we do not believe such a discrepancy would exist in a more optimal setup. In total, the GSM8K experiments used roughly 1300 GPU hours.

\subsubsection{Methods and Implementations} \label{app:empirical_methods}

We compare offline behavior cloning against several on-policy distillation objectives. The
key distinction among the online methods is that there are two separate choices: the
policy used to generate prefixes, and the divergence used to compare the student and
teacher distributions on those prefixes. This distinction is motivated by \Cref{cor:kl_augmented},
which controls the clean-expert trajectory error through KL divergences between
\emph{augmented} trajectory laws. In particular, the rollout distribution determines which
prefixes are visited, while the per-prefix student distribution determines the next-token
distribution being matched to the noisy expert.

Let $\pi_\theta$ denote the trainable student policy and let $\bar{\pi}_\theta$ denote the
greedy rollout policy induced by $\pi_\theta$. Thus $y_{1:T} \sim \bar{\pi}_\theta(\cdot \mid x)$
denotes the trajectory obtained by greedy decoding from the student, while
$\pi_\theta(\cdot \mid x,y_{<t})$ denotes the student's full next-token distribution on the
resulting prefix. We use $\pi^\star_\eta$ to denote the noisy teacher. In the synthetic
experiments, $\pi^\star_\eta$ is obtained by corrupting the clean expert token with the
specified corruption law. In the language-model experiments, it is implemented by
sampling from the teacher at the specified temperature.

\textbf{\nailf.} The forward-KL version of our method (which minimizes the first term in  \eqref{eq:kl_augmented}) uses greedy student prefixes and trains the
student to match noisy-teacher next-token feedback on those prefixes. Ideally, this
corresponds to minimizing
\begin{align}
\mathcal{L}_{\mathrm{\nailf}}(\theta)
=
\ee_{x \sim \cX}
\ee_{y_{1:T} \sim \bar\pi_\theta(\cdot \mid x)}
\left[
\sum_{t=1}^T
\kldinline{
\pi^\star_\eta(\cdot \mid x, y_{<t})
}{
\pi_\theta(\cdot \mid x, y_{<t})
}
\right],
\end{align}
Equivalently, up to the entropy of the noisy teacher, this is a soft-target
cross-entropy objective:
\begin{align}
-
\ee_{x \sim \cX}
\ee_{y_{1:T} \sim \bar\pi_\theta(\cdot \mid x)}
\left[
\sum_{t=1}^T \sum_{a \in \Sigma}
\pi^\star_\eta(a \mid x, y_{<t})
\log \pi_\theta(a \mid x, y_{<t})
\right].
\end{align}
Since our implementation queries the teacher by sampling, we
estimate this loss by drawing an independent teacher token
$\tilde{y}_t \sim \pi^\star_\eta(\cdot \mid x,y_{<t})$ at each greedy student prefix and
minimizing
\begin{align}
\widehat{\cL}_{\mathrm{\nailf}}(\theta) =
\ee_{x \sim \cX}
\ee_{y_{1:T} \sim \bar\pi_\theta(\cdot \mid x)}
\ee_{\tilde y_t \sim \pi^\star_\eta(\cdot \mid x,y_{<t})}
\left[
\sum_{t=1}^T
-\log \pi_\theta(\tilde y_t \mid x,y_{<t})
\right].
\end{align}
Thus, \nailf can be thought of as ``local behavior cloning'' on the learner's own greedy
prefixes: rather than imitating complete noisy teacher trajectories, it repeatedly asks
what the noisy teacher would do at states actually reached by greedily rolling out the current student, and then trains the student to predict the noisy teacher's sample via cross-entropy. This
is the empirical analogue of the forward KL augmented-trajectory objective suggested by
our theory. When $\eta$ and $\nu$ are known, the literal theoretical objective replaces
$\pi_\theta$ above by the noisy student policy
$\pi_{\theta,\eta}=(1-\eta)\pi_\theta+\eta\nu$; in our experiments we use the simpler
student-matching surrogate above.

\textbf{\nailr.}
We also consider the reverse KL analogue suggested by minimizing the second term in \eqref{eq:kl_augmented}. As in \nailf, prefixes are generated by the greedy student rollout, but the KL
direction is reversed:
\begin{align}
\mathcal{L}_{\mathrm{\nailr}}(\theta)
=\ee_{x \sim \cX}
\ee_{y_{1:T} \sim \bar\pi_\theta(\cdot \mid x)}
\left[
\sum_{t=1}^T
\kldinline{
\pi_\theta(\cdot \mid x,y_{<t})
}{
\pi^\star_\eta(\cdot \mid x,y_{<t})
}
\right].
\end{align}
We estimate the student expectation in the reverse KL by drawing an auxiliary token
$\hat{y}_t \sim \pi_\theta(\cdot \mid x,y_{<t})$ at each greedy prefix:
\begin{align}
\widehat{\cL}_{\mathrm{\nailr}}(\theta)
=
\ee_{x \sim \cX}
\ee_{y_{1:T} \sim \bar\pi_\theta(\cdot \mid x)}
\ee_{\hat y_t \sim \pi_\theta(\cdot \mid x,y_{<t})}
\left[
\sum_{t=1}^T
\log \pi_\theta(\hat y_t \mid x,y_{<t})
-
\log \pi^\star_\eta(\hat y_t \mid x,y_{<t})
\right].
\end{align}
We note that the auxiliary token $\hat{y}_t$ is not necessarily the greedy rollout token $y_t$. The
greedy rollout tokens determine the prefix distribution, while the auxiliary samples estimate the
student expectation in the reverse KL. This separation is needed to match the
augmented-trajectory viewpoint: the rollout policy and the next-token distribution being
compared need not be the same object.

\textbf{$\opdf$.}
As an ablation, we also consider the forward-KL objective with sampled student
rollouts rather than greedy rollouts:
\begin{align}
\mathcal{L}_{\mathrm{\opdf}}(\theta)
=
\ee_{x \sim \cX}
\ee_{y_{1:T} \sim \pi_\theta(\cdot \mid x)}
\left[
\sum_{t=1}^T
\kldinline{
\pi^\star_\eta(\cdot \mid x, y_{<t})
}{
\pi_\theta(\cdot \mid x, y_{<t})
}
\right],
\end{align}
This isolates the effect of the rollout distribution: comparing $\opdf$ to \nailf tests
whether greedy learner prefixes are important, while comparing $\opdf$ to $\opdr$ tests
the effect of KL direction under the same sampled-prefix distribution.

\textbf{$\opdr$.} Finally, we include the standard on-policy distillation baseline using reverse KL \cite{lu2025onpolicydistillation}.
This method samples trajectories from the student, queries the teacher on those same
student-generated prefixes, and updates the student using the teacher log-probability of
the sampled student tokens:
\begin{align}
\cL_{\mathrm{\opdr}}(\theta)
&\defeq
\ee_{x \sim \cX}
\ee_{y \sim \pi_\theta(\cdot \mid x)}
\left[
\sum_{t=1}^{T}
\kldinline{
\pi_\theta(\cdot \mid x,y_{<t})
}{
\pi^\star_\eta(\cdot \mid x,y_{<t})
}
\right].
\end{align}
Using the sampled rollout token $y_t \sim \pi_\theta$, this becomes:
\begin{align}
\widehat{\cL}_{\mathrm{\opdr}}(\theta)
=
\ee_{x \sim \cX}
\ee_{y_{1:T} \sim \pi_\theta(\cdot \mid x)}
\left[
\sum_{t=1}^{T}
\Big(
\log \pi_\theta(y_t \mid x,y_{<t})
-
\log \pi^\star_\eta(y_t \mid x,y_{<t})
\Big)
\right].
\end{align}

\textbf{\textsf{LogLossBC}.} The offline baseline is log-loss behavior cloning on a fixed dataset of noisy expert
rollouts \cite{foster2024behavior}. Let $\mathcal{D}_\eta$ denote trajectories generated by rolling out the noisy
expert $\pi^\star_\eta$. We train
\begin{align}
\cL_{\algllbc}(\theta)
=
\ee_{(x,y) \sim \cD_\eta}
\left[
-\sum_{t=1}^{|y|}
\log \pi_\theta(y_t \mid x, y_{<t})
\right].
\end{align}
Unlike the OPD variants, behavior cloning never queries the teacher on prefixes induced by the
current student. Consequently, if an early token in an offline teacher trajectory is
corrupted, the learner is also trained on all downstream prefixes induced by that
corruption.

\subsubsection{Gradient estimators.}
\label{app:method_gradient_estimators}

We now spell out the stochastic gradient estimators used to optimize the empirical
losses defined above. Let $s_t=(x,y_{<t})$ denote a visited prefix, and write
$p_{\theta,t}(\cdot)=\pi_\theta(\cdot\mid s_t)$ and
$q_t(\cdot)=\pi^\star_\eta(\cdot\mid s_t)$. Throughout,
$p_{\theta,t}$ denotes the temperature-one student distribution whose KL is
optimized. Rollout-temperature ablations only change the distribution used to
collect prefixes. In particular, for rollout temperature $\tau$, let
$\rho_{\theta,\tau}$ denote the stopped prefix-collection policy, with
\begin{align}
\rho_{\theta,0}=\bar\pi_\theta,
\qquad
\rho_{\theta,\tau}(\cdot\mid s_t)
=
\softmax(z_\theta(s_t)/\tau)
\quad \text{for } \tau>0 .
\end{align}
All gradients below treat prefixes sampled from $\rho_{\theta,\tau}$ as fixed;
the gradient update is taken only through the next-token distribution
$p_{\theta,t}$ at those visited prefixes. The default settings are
$\tau=0$ for \nailf and \nailr and $\tau=1$ for \opdf and \opdr, while the
rollout-temperature ablations vary this prefix-collection temperature.

For $\widehat{\cL}_{\mathrm{\nailf}}(\theta)$, after drawing prefixes
$y_{1:T}\sim \rho_{\theta,\tau}(\cdot\mid x)$ and independent teacher tokens
$\tilde y_t\sim q_t$, we use
\begin{align}
\widehat g^{\mathrm F}_{\tau}
=
-\sum_{t=1}^T
\nabla_\theta \log p_{\theta,t}(\tilde y_t).
\end{align}
This is an unbiased stochastic gradient estimator for the stopped-prefix
forward-KL objective because the entropy term of $q_t$ is independent of
$\theta$. The methods differ only in how the prefixes are collected:
$\tau=0$ gives the default \nailf estimator with greedy prefixes, while
$\tau=1$ gives the default sampled-prefix \opdf estimator. Other rollout
temperatures use the same next-token estimator on prefixes sampled from
$\rho_{\theta,\tau}$. The offline \algllbc baseline is ordinary
cross-entropy on realized noisy-expert trajectories; it has the same token-level
gradient form, but its prefixes and labels come from an offline dataset
rather than from student rollouts.

For $\widehat{\cL}_{\mathrm{\nailr}}(\theta)$, prefix collection and
reverse-KL token sampling are separate. After drawing prefixes
$y_{1:T}\sim \rho_{\theta,\tau}(\cdot\mid x)$, the implementation draws an
auxiliary token from a stopped copy of the current student distribution at the
visited prefix: 
$\hat y_t\sim \sg(p_{\theta,t})$,
where $\sg(\cdot)$ denotes stop-gradient. Thus, the rollout temperature $\tau$ affects which prefixes are visited, but
the auxiliary token for the reverse-KL estimator is sampled from
$p_{\theta,t}$, not from the rollout distribution. The stopped-prefix
score-function estimator is
\begin{align}
\widehat g^{\mathrm R}_{\tau}
=
\sum_{t=1}^T
\left(
\sg(\log p_{\theta,t}(\hat y_t))
-
\log q_t(\hat y_t)
\right)
\nabla_\theta \log p_{\theta,t}(\hat y_t),
\qquad
\hat y_t\sim \sg(p_{\theta,t}).
\end{align}
Equivalently, this is implemented by backpropagating through the surrogate
\begin{align}
\widehat \ell^{\mathrm R}_t(\theta)
=
-
\exp\!\left(
\log p_{\theta,t}(\hat y_t)
-
\sg(\log p_{\theta,t}(\hat y_t))
\right)
\left(
\log q_t(\hat y_t)
-
\sg(\log p_{\theta,t}(\hat y_t))
\right).
\end{align}
The exponential factor is numerically equal to one at the sampled parameter
value, but its numerator still carries gradient, so differentiating gives exactly
the estimator above.

For $\widehat{\cL}_{\mathrm{\opdr}}(\theta)$, as in standard OPD, the sampled
rollout token itself is used as the reverse-KL sample. Therefore, if the rollout
token is sampled from the current temperature-one student distribution, this coincides
with the on-policy reverse-KL estimator above. If one instead reuses rollout
tokens sampled at a different temperature, then the token distribution no longer
matches $p_{\theta,t}$, and the resulting update should be viewed as a
temperature-mismatched surrogate. Note that this mirrors the recipe described by \citet{lu2025onpolicydistillation}.

For the interpolated objective in \Cref{app:add_experiments_and_ablations}, where
\begin{align}
\widehat{\cL}_{\beta}(\theta)
=
(1-\beta)\widehat{\cL}_{\mathrm{\nailf}}(\theta)
+
\beta \widehat{\cL}_{\mathrm{\nailr}}(\theta),
\end{align}
we use the corresponding convex combination of the two stopped-prefix estimators,
\begin{align}
\widehat g_{\beta}
=
(1-\beta)\widehat g^{\mathrm F}_{0}
+
\beta \widehat g^{\mathrm R}_{0}.
\end{align}

\subsubsection{Synthetic Task: Modular Addition} \label{app:synthetic_task_details}

 We use the modular-addition task of \citet{li2024chain}, denoted $C_p$. Given $p \in \mathbb{N}$, the vocabulary is $\{0,\ldots,p-1,=\}$. An input has the form $x=(x_1,\ldots,x_m,=)$, where each $x_i$ is sampled independently from $\{0,\ldots,p-1\}$. The final answer is
$f^*(x)=\sum_{i=1}^{m} x_i \bmod p$
and the chain of thought consists of the running partial sums
$(\sum_{i=1}^t x_i \bmod p)_{t=1}^{m}$. Modular addition is not intended to be an inherently serial hardness instance; indeed, \citet{li2024chain} note that it is parallelizable and, in principle, can be solved by constant-depth transformers with sufficient precision. Nevertheless, their experiments show that low-depth transformers without CoT can perform near chance on $C_7$, while CoT substantially improves performance, especially at longer sequence lengths. We include it as a simple controlled setting in which the target CoT has a transparent step-by-step structure, allowing us to study how noisy intermediate feedback affects offline and online imitation-learning methods.

\textbf{Training and Evaluation Details.}
We use the nanoGPT codebase \citep{Karpathy2022} for all modular-addition experiments. For the tokenizer, we use task-specific integer token IDs from the aforementioned symbolic vocabulary. Unless otherwise stated, all runs use a depth-one transformer with $8$ attention heads, embedding dimension $512$, dropout $0$, no bias terms, and block size set to the CoT sequence length for the corresponding $m$. We train with batch size $64$ using \textsf{Adam} \citep{kingma-2015} with learning rate $10^{-5}$, warmup for $2000$ iterations, weight decay $0$, $\beta_1=0.9$, $\beta_2=0.95$, gradient clipping at $1.0$, and evaluation every $500$ iterations.

We fix $p=7$ and $m=31$. All methods use the same prompt bank, containing $15$ million training prompts and $5000$ validation prompts. We train on a fixed $3$ million-prompt subset in a fixed order. The clean expert is fixed across all runs and is synthetically modified using the corruption law described below. It was trained for $10{,}000$ optimizer steps with batch size $64$. The model reached perfect validation accuracy after roughly $2500$ optimizer steps and stayed at this level, but we use the final checkpoint for all experiments.

We compare the five methods for two noise levels, $\eta \in \{0,0.2\}$, and across three seeds. For each run seed, we use the same seed for student optimization and for the stochastic components of the noisy teacher law. For \textsf{LogLossBC}, the seed also determines the independently rendered noisy dataset. For the online methods, there is no pre-rendered trajectory; the same seed controls the online teacher sampling.

The noisy teacher is an absorbing instantiation of the state-dependent noise model in \eqref{eq:noisy_expert}. The state is augmented with a binary flag indicating whether the prefix has already made a semantic error. In the unpoisoned state, at each target token the corruption distribution is
\begin{align}
\nu(\cdot \mid x,y_{<t}) = \mathrm{Unif}(\{0,\ldots,p-1\}),
\end{align}
so that
\begin{align}
\pi^\star_\eta(\cdot \mid x,y_{<t})
=
(1-\eta)\pi^\star(\cdot \mid x,y_{<t})
+
\eta \nu(\cdot \mid x,y_{<t}).
\end{align}
Here the uniform distribution includes the clean token. Since every modular-addition target token is a digit representing a running sum, every target token is treated as semantic. Once the sampled token differs from the clean running-sum token, the trajectory enters the poisoned state. In the poisoned state, both $\pi^\star$ and $\nu$ are defined to be uniform over $\{0,\ldots,p-1\}$, so all subsequent semantic feedback is independent of the clean computation while remaining syntactically valid.

This absorbing corruption law models a worst-case form of semantic error propagation in long CoT reasoning: after an early pivotal mistake, the suffix may remain syntactically plausible while providing minimal information about the clean continuation. This is motivated by recent work on pivotal or forking tokens in reasoning models, where a small number of tokens can strongly affect the success of the final completion \citep{abdin2024phi,wang2025beyond}. In our synthetic task, the running-sum tokens play this role, so corrupting one makes downstream semantic feedback uninformative about the correct computation. This pessimistic noise model provides a starting point for a
controlled setting in which offline training on corrupted rollouts can suffer from the
trajectory-level contamination predicted by our theory, and we leave the consideration of more nuanced noise models to future work.

For \textsf{LogLossBC}, training trajectories are rendered autoregressively from this noisy teacher law. Thus, in the offline setting, an example is fully informative only if no visible semantic corruption occurs. Under the idealized independent-trigger calculation, the probability of this event is
$\left(1-\eta+\nicefrac{\eta}{p}\right)^m \approx 2.9 \times 10^{-3}$ in our setting. 

For the online methods, the poisoned flag is inferred from the student-generated prefix: if a previous running-sum token in the student prefix differs from the clean target, then later teacher feedback is uniform over residues. Consequently, $\eta=0$ has different semantics for offline and online runs under this law: offline rendering is clean at $\eta=0$, while online feedback can still become uninformative after a student prefix error.

Online methods are trained for one pass over the fixed $3$ million-prompt subset, giving $3{,}000{,}000/64=46{,}875$ iterations. \textsf{LogLossBC} is also trained for one pass over the corresponding rendered $3$ million-example dataset. All methods are evaluated on the same clean validation prompt bank.

\subsubsection{Mathematical Reasoning: GSM-8K} \label{app:math_task_details}

We evaluate on the GSM-8K test set \citep{cobbe2021training} using zero-shot greedy decoding with max generation length 1024 and seed 42. Each problem is placed in the \textsc{Gemma3} chat template with the following prefix.
\begin{quote}
\small
\begin{verbatim}
Please reason step by step, and put your final answer within \boxed{}.
\end{verbatim}
\end{quote}

For example, an evaluation prompt takes the form
\begin{quote}
\small
\begin{verbatim}
<bos><start_of_turn>user
Please reason step by step, and put your final answer within \boxed{}.

Janet's ducks lay 16 eggs per day. She eats three for breakfast every morning
and bakes muffins for her friends every day with four. She sells the remainder
at the farmers' market daily for $2 per fresh duck egg. How much in dollars
does she make every day at the farmers' market?<end_of_turn>
<start_of_turn>model
\end{verbatim}
\end{quote}
We extract the final answer in the \verb|\boxed{}| from responses and compute exact-match accuracy after standard answer normalization. The student model, \textsc{Gemma3-270M-IT}, achieves 9.09\% accuracy.

\paragraph{Training data.}
We construct the training prompt set from TinyGSM \cite{liu2023tinygsm}. Specifically, we first filter out examples whose code answers contain more than 1024 characters, and then remove examples whose code answers are not executable. This filtering yields approximately 11M examples. We then sample a 400K-example subset that is used for all training runs.

\paragraph{Training details.}
For the \textsf{LogLossBC} baseline, we generate fixed teacher rollouts on the 400K TinyGSM prompts using \gemmab. We generate clean rollouts with temperature $1$ in the low-noise setting and noisy rollouts with temperature $4$ in the high-noise setting, both using random seed $42$ and a maximum generation length of $1024$ new tokens. We measure the accuracy of these clean and noisy teacher rollouts on the 400K TinyGSM subset by executing the reference code answers and treating the execution outputs as ground truth. The clean teacher obtains 53.64\% accuracy on the training subset, while the noisy teacher obtains 0\% accuracy. We then supervised fine-tune the \gemmam\ student separately on the clean and noisy teacher rollouts.

For the online reverse-KL variants, \opdr and \nailr, we follow the recipe of \citet{lu2025onpolicydistillation}. We generate the student's next token using temperature-$1$ sampling for \opdr and greedy decoding for \nailr. Along these student-generated trajectories, we compute the student and teacher log probabilities conditioned on the same prefixes, using either the clean temperature-$1$ teacher or the noisy temperature-$4$ teacher. The student is then updated using the corresponding importance-sampling loss.

For the forward-KL variants, \opdf and \nailf, we instead query the teacher locally on learner-visited prefixes. At each such prefix, we sample the teacher's next token, using the temperature-$1$ teacher in the low-noise setting and the temperature-$4$ teacher in the high-noise setting, and treat the sampled token as a hard next-token label. The student is then updated using the standard next-token log-loss. Subsequent learner prefixes are generated with temperature-$1$ sampling for \opdf and greedy decoding for \nailf.

All methods are trained on the same 400K TinyGSM prompts with a maximum generation length of 1024 new tokens. For the online methods, we follow the recommended recipe from \citet{lu2025onpolicydistillation}. For the offline method, we tune the learning rate slightly in preliminary experiments, while keeping all other hyperparameters fixed.
Specifically, unless otherwise stated, all methods are trained for one epoch using \textsf{AdamW} with default parameters, learning rate 1e-4, linear warmup followed by cosine decay, batch size 64, gradient clipping at 1, and bf16 precision. We use LoRA with rank 128 and $\alpha=256$, applied to all modules. For each method, we run three random seeds, 42, 43, and 44.

\subsection{Further Discussion of Empirical Results in \Cref{sec:experiments}}
\label{app:add_experimental_discussion}
All configurations are run with three random seeds. In the plots, each curve shows the mean across seeds, and the shaded region corresponds to one standard deviation. We note that some shaded regions are not visible as the runs are nearly identical across seeds. In particular, for Modular Addition, the maximum standard deviation is less than $10^{-2}$ for \textsf{LogLossBC} and \opdf in the low-noise setting, and for \opdf, \opdr, and \textsf{LogLossBC} in the high-noise setting (as the latter all fail to learn).

For clarity of presentation, the main text shows accuracy only over the first $1$M expert trajectories for Modular Addition, where the relevant separations between methods are most visible. In \Cref{fig:modadd_val_loss}, we provide the corresponding validation-loss curves over the full $3$M-trajectory training horizon. These longer-horizon curves confirm the same qualitative picture: in the low noise setting, the NAIL variants rapidly drive validation loss close to zero, while \opdf plateaus early; in the high noise setting, the gap becomes more pronounced, with
\nailf and \nailr remaining substantially more stable than the offline and OPD baselines.

\begin{figure}[t]
    \centering
    \includegraphics[width=1\linewidth]{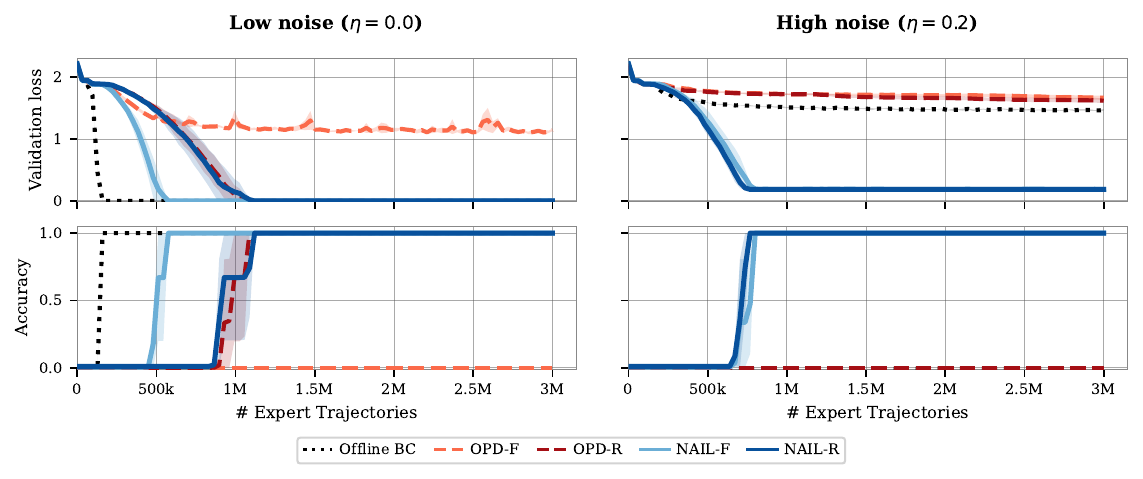}
    \caption{Full modular-addition results over $3$M expert trajectories. \textbf{Top}: validation loss; \textbf{Bottom}: accuracy. Curves show the mean over three random seeds, with shaded regions indicating one standard deviation. In the low-noise setting $(\eta=0)$, the NAIL variants drive validation loss to zero and reach perfect accuracy, while \opdf plateaus. In the high-noise setting $(\eta=0.2)$, the separation is more pronounced: \nailf and \nailr remain stable and reach perfect accuracy, whereas offline \textsf{LogLossBC} and the OPD baselines fail to solve the task.}
    \label{fig:modadd_val_loss}
\end{figure}

To expand on the discussion in \Cref{sec:experiments}, \Cref{fig:modadd_val_loss} shows a sharp contrast between the clean and noisy regimes for the modular addition task. When $\eta=0.2$, \algllbc\ fails because an early corrupted CoT token makes the remaining suffix random, so most offline trajectories contain little usable signal. In contrast, our \nailf\ and \nailr\ reach perfect accuracy, consistent with the idea that querying on student-induced prefixes avoids imitating fully corrupted trajectories. The failure of sampled-rollout OPD further suggests that controlling the rollout distribution, here via greedy prefixes, is important for keeping teacher feedback informative. 

When $\eta = 0.0$, since offline traces are clean, \algllbc{} learns efficiently, as predicted by \citet{foster2024behavior}. Among online methods, $\nailf$ learns much faster than $\opdf$, suggesting that greedy rollouts help suppress sampling-induced noise that hinders access to uncorrupted teacher feedback. The gap between $\nailf$ and the reverse-KL methods here points to the fact that teacher-sampled cross-entropy gives a direct positive signal for the correct next token, whereas reverse KL only scores student-sampled tokens and is therefore less directly corrective. Finally, when there is no expert noise, greedy rollouts offer little benefit for \nailr\ over \opdr\ because both methods update using student-sampled tokens. Critically, \opdf\ flatlines near chance, as it asks for expert feedback on student prefixes corrupted by the student's own sampling, while \opdr\ scores the student's own sampled tokens, thus providing a more direct measure of the student's performance even on noisy trajectories.

For GSM-8K, noisy expert ($t=4$) rollouts often become globally uninformative, containing degenerate text such as random Unicode tokens, and therefore provide little usable supervision. As a result, offline \algllbc{} performs even worse than the vanilla student baseline of 9.09\%. The failure of sampled-rollout OPD suggests that adaptivity alone is not sufficient under this level of expert noise. Specifically, when rollouts are sampled from a weak student, the resulting prefixes can be noisy or off-distribution, and querying a temperature-$4$ expert on these prefixes often yields feedback that remains uninformative. In contrast, \nailf{} and \nailr{} use greedy student prefixes, which better control the query prefix distribution and make the expert feedback more reliable. Although the noisy expert limits final performance, both \nailf{} and \nailr{} continue to improve with more expert trajectories and surpass the vanilla student baseline, showing that they can still extract useful learning signal from high-temperature expert.

In the low-noise regime ($t=1$), by contrast, all online methods substantially outperform offline \algllbc{}, indicating that online interaction is especially beneficial in this likely misspecified setting. Among the online methods, the forward-KL variants perform best: both \opdf{} and \nailf{} reach roughly 40\% accuracy, while the reverse-KL variants improve more gradually and plateau slightly lower. This gap reflects an objective-level effect also seen in Modular Addition. Unlike in clean Modular Addition, however, \opdf{} remains competitive on GSM-8K, suggesting that prefixes sampled from the temperature-$1$ \textsc{Gemma3-270M-IT} student are sufficiently reliable in this setting.

\subsection{Further Experiments}
\vscomment{@Peihan: could we run further ablations for GSM8K?}
\label{app:add_experiments_and_ablations}

We include two additional ablations on the modular addition task to better understand which parts of the online objective matter. 
In \Cref{app:beta_interpolation}, we fix greedy student rollouts and interpolate between the forward- and reverse-KL augmented-trajectory losses, testing whether performance depends on a single KL direction. 

\subsubsection{Ablating student rollout temperature for Modular Addition}
\label{app:student_temp}
We next ablate the temperature used when sampling student rollouts. In the main experiments, the student prefixes are generated \emph{greedily}, i.e. temperature $t=0$. Here, we instead sample from the student policy with various temperatures $t \in \{0.1, 0.3, 0.5\}$. We evaluate both \nailf and \nailr on Modular Addition under the same low- and high-noise settings as in the main body. All variants are trained using the same recipe as in \Cref{app:add_experimental_details}; for readability, \Cref{fig:temp_ablation_modadd} shows only the first $1.25$M expert-query trajectories.

\begin{figure}[t]
    \centering
    \includegraphics[width=1\linewidth]{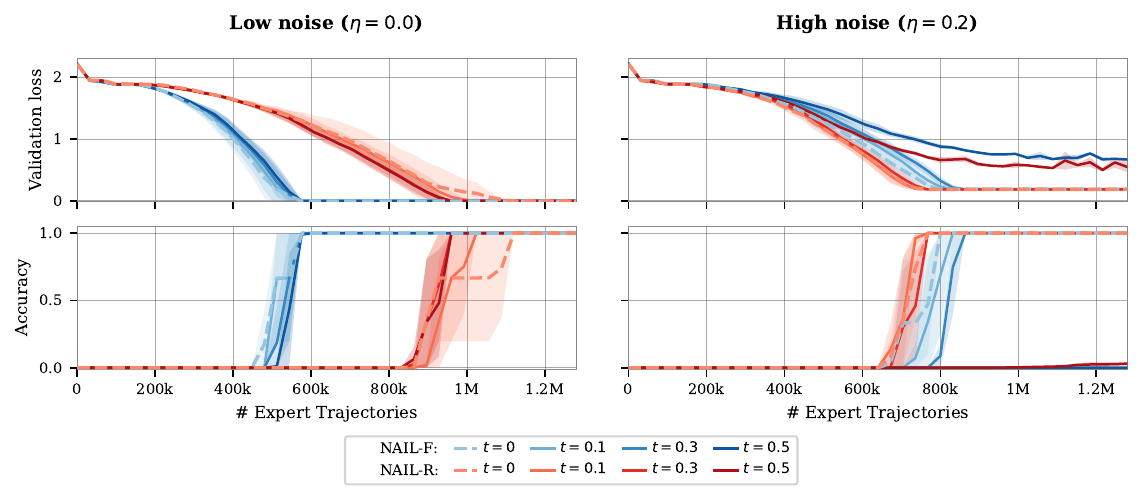}
    \caption{Ablation of student rollout temperature for \nailf and \nailr on Modular Addition. The parameter $t$ controls the student sampling temperature used during training rollout. Curves show the mean over three random seeds, with shaded regions indicating one standard deviation. \textbf{Left:} low-noise setting, \nailf learns faster and is relatively insensitive to temperature, while \nailr is substantially slower. \textbf{Right:} high-noise setting, moderate student temperatures improve robustness for both objectives, with NAIL-F solving the task across temperatures and NAIL-R degrading at larger $t$.}
    \label{fig:temp_ablation_modadd}
\end{figure}

In the low-noise setting $(\eta=0)$, \nailf solves the task at roughly the same sample complexity across rollout temperatures, indicating that forward-KL training is fairly robust to this choice. In contrast, \nailr remains slower across temperatures, consistent with the main results in \Cref{fig:main_modadd_gsm8k}. In the high-noise setting $(\eta=0.2)$, small to moderate rollout temperatures can still solve the task, but large temperature substantially hurts both \nailf and \nailr. This suggests that some stochasticity in student rollouts is tolerable, but excessive exploration can produce prefixes that are too noisy or off-distribution for effective expert querying.

\subsubsection{Ablating Horizon and Noise Level for Modular Addition}
\label{app:modadd_horizon_noise_ablation}

We further ablate the dependence of the modular-addition results on the chain-of-thought length
and corruption rate. We consider
\begin{align}
m \in \{3,9,15,23,31\},
\qquad
\eta \in \{0,0.05,0.1,0.15,0.2\},
\end{align}
with $p=7$, using the same absorbing random-suffix corruption law as in
\Cref{app:synthetic_task_details}. Recall that once a sampled intermediate token differs from the clean
running-sum token, the teacher enters a poisoned state in which subsequent semantic feedback is
uninformative about the clean computation. Thus the effective difficulty is governed by both $m$
and $\eta$: for offline rollouts, the probability that a trajectory remains unpoisoned decreases
approximately as $(1-\nicefrac{6\eta}{7})^m$.

\begin{figure}[t]
    \centering
    \includegraphics[width=1\linewidth]{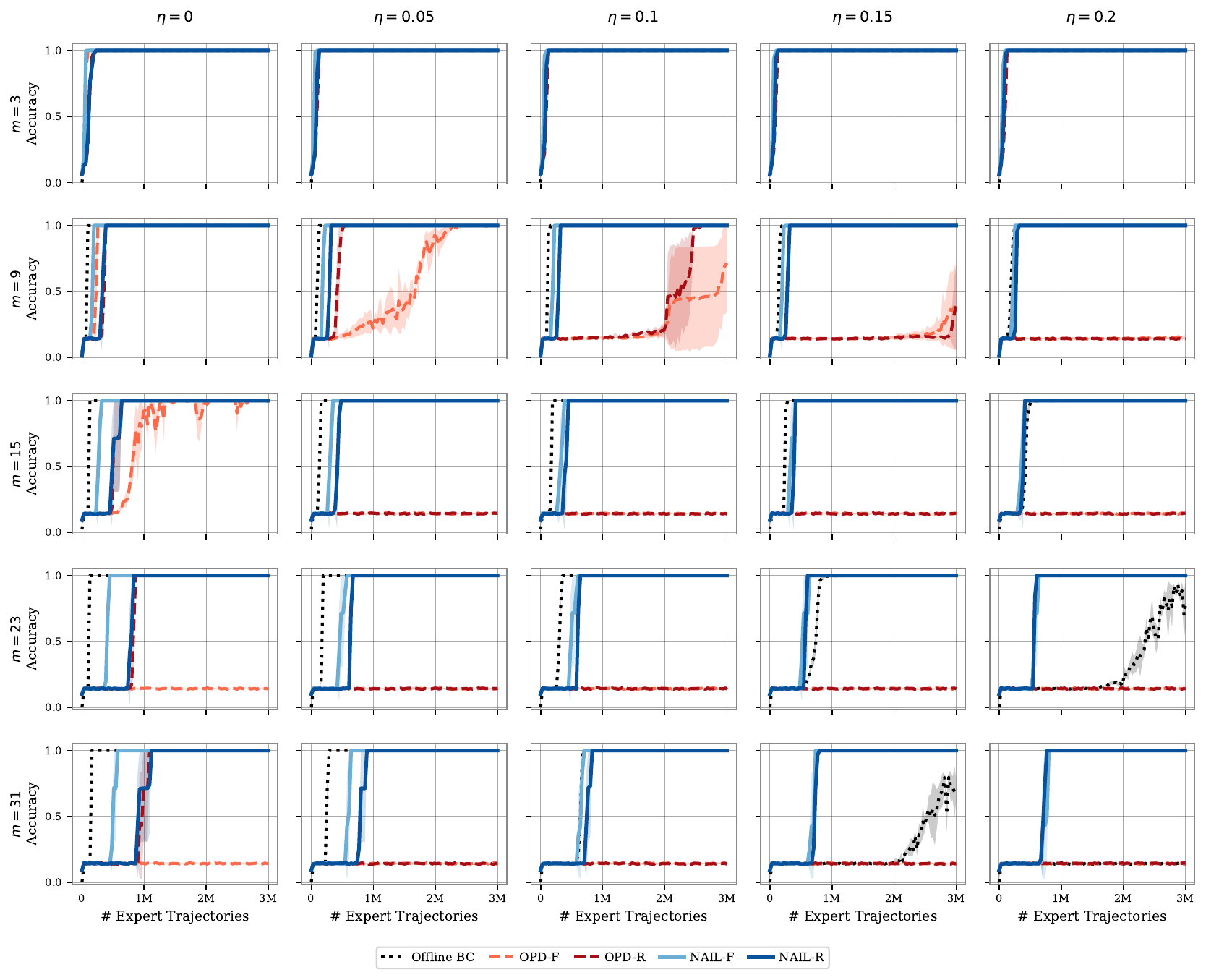}
    \caption{Ablation of CoT length ($m$) and corruption rate ($\eta$) on Modular Addition. Curves show the mean over three random seeds, with shaded regions indicating one standard deviation. 
    }
    \label{fig:modadd_horizon_noise_ablation}
\end{figure}
The results are shown in \Cref{fig:modadd_horizon_noise_ablation}. Across essentially all horizons and noise levels, the greedy-prefix NAIL variants are stable. By contrast, the sampled-prefix OPD variants are substantially less robust. Our ablations interpolate between the two extremes previously presented in \Cref{fig:modadd_val_loss}, and our findings adhere to the intuition detailed in \Cref{app:add_experimental_discussion}.

Offline behavior cloning exhibits a more nuanced pattern. As expected, it performs well in the clean
setting and degrades when noisy trajectories are long enough that most offline demonstrations become
poisoned. However, in some intermediate regimes, adding a small amount of noise can improve or
accelerate offline BC relative to the clean run. This is not entirely surprising: prior work has observed
that injecting noise into behavior-cloning data can improve robustness by exposing the learner to
off-nominal states \citep{laskey2017dart}. In our setting this effect is incidental rather than algorithmically imposed, since the same corruption can also destroy the semantic content of the trajectory. The ablation therefore highlights two competing effects: mild noise can act as a form of data augmentation, while larger $\eta, m$ leads to semantic poisoning and loss of clean supervision.

\subsubsection{Interpolating between \nailf and \nailr}
\label{app:beta_interpolation}
To further probe the role of the KL direction, we fix greedy student rollouts and interpolate between the forward- and reverse-KL losses. For $\beta \in [0,1]$, we minimize
\begin{align}
    J_\beta(\pi_\theta) = (1 - \beta)\cdot \kld{\pp^{\bar\pi_\theta, \pistar_\eta}}{\pp^{\bar\pi_\theta, \pi_{\theta}}} + \beta  \cdot\kld{\pp^{\bar\pi_\theta, \pi_{\theta}}}{\pp^{\bar\pi_\theta, \pistar_\eta}}.
\end{align}
Thus $\beta=0$ recovers \nailf{}, while $\beta=1$ recovers \nailr{}.

\paragraph{Modular Addition.}
We train each variant for $3$M expert-query trajectories using the same recipe as in \Cref{app:add_experimental_details}; for readability, \Cref{fig:beta_interpolation_modadd} shows only the first $1.25$M trajectories. All interpolated variants eventually reach perfect accuracy in both the low-noise and high-noise regimes, but their learning dynamics differ.

\begin{figure}[t]
    \centering
    \includegraphics[width=1\linewidth]{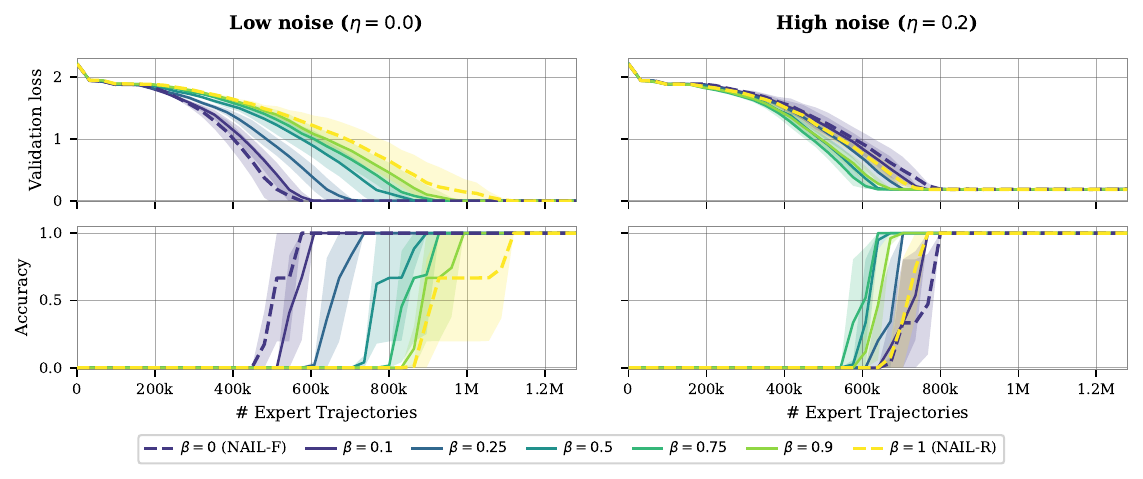}
    \caption{Interpolation between \nailf and \nailr on Modular Addition. The parameter $\beta$ interpolates between the forward-KL $(\beta=0)$ and the reverse-KL $(\beta=1)$ losses. Curves show the mean over three random seeds, with shaded regions indicating one standard deviation. All interpolated variants eventually solve the task in both noise regimes, but the learning speed depends strongly on $\beta$. 
    \textbf{Left:} in the low-noise setting, forward-KL-heavy objectives learn fastest. \textbf{Right:} in the high-noise setting, intermediate and reverse-KL-heavy objectives are competitive, suggesting that (i) the robustness of NAIL is primarily driven by querying the expert on learner-induced prefixes rather than by a single KL direction, and (ii) the best KL direction is task- and noise-dependent.}
    \label{fig:beta_interpolation_modadd}
\end{figure}

When $\eta=0$, performance changes smoothly from that of \nailf{} to \nailr{} as $\beta$ increases, matching the behavior in \Cref{fig:main_modadd_gsm8k}: forward-KL-heavy objectives learn fastest. When $\eta=0.2$, however, intermediate and reverse-KL-heavy mixtures learn slightly faster than either endpoint, with larger values of $\beta$ reaching perfect accuracy roughly $100$K trajectories earlier. These results suggest that the optimal KL mixture is task- and noise-dependent. A more systematic study of when to prefer forward KL, reverse KL, or mixtures of the two is an interesting direction for future work.

\paragraph{GSM-8K.} We train each variant on TinyGSM using the same recipe as in \Cref{app:add_experimental_details}. \Cref{fig:beta_interpolation_gsm8k} shows that the effect of $\beta$ is smoother on GSM-8K than on Modular Addition. In the low-noise regime ($t=1$), forward-KL-heavy objectives achieve the best performance: \nailf{} and small values of $\beta$ reach the highest final accuracy, while reverse-KL-heavy mixtures improve more slowly and plateau lower. In the high-noise regime ($t=4$), the curves are closer together, and no single endpoint dominates. Intermediate mixtures are competitive with, and in some cases slightly better than the two endpoints \nailf{} and \nailr{}, consistent with the high-noise trend observed in Modular Addition.

\begin{figure}[t]
    \centering
    \includegraphics[width=1\linewidth]{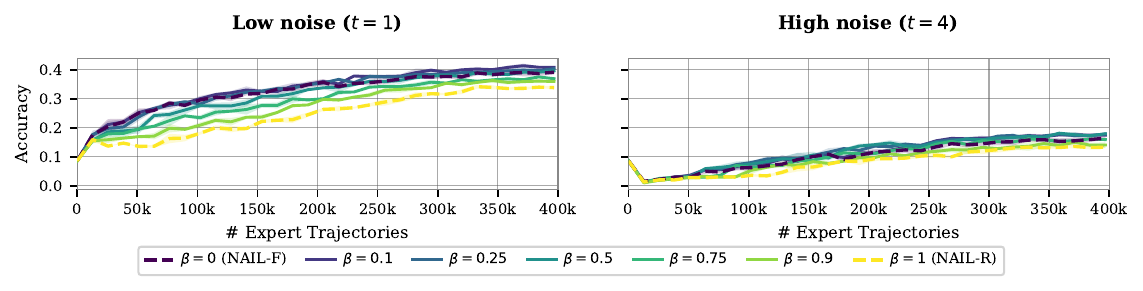}
    \caption{Interpolation between \nailf and \nailr on GSM-8K. The parameter $\beta$ interpolates between the forward-KL $(\beta=0)$ and the reverse-KL $(\beta=1)$ losses. Curves show the mean over three random seeds, with shaded regions indicating one standard deviation. 
    \textbf{Left:} in the low-noise setting $(t=1)$, forward-KL-heavy objectives learn fastest and reach the highest accuracy. 
    \textbf{Right:} in the high-noise setting $(t=4)$, all mixtures improve gradually and remain relatively close, with intermediate values of $\beta$ slightly outperforming the endpoints.
    }
    \label{fig:beta_interpolation_gsm8k}
\end{figure}

%% file: body_clean/app_technical_prelims.tex
\section{Technical Tools}\label{app:prelims}

In this section, we recall some technical tools that are used throughout the proofs. We begin in \Cref{app:info_theory}, where we recall some basic definitions and properties of KL divergence and Hellinger distance. We proceed in \Cref{app:mle} by stating some classical results on the performance of maximum likelihood estimators. In \Cref{app:il_prelims}, we recall some key results from the theory of imitation learning that relate regret to Hellinger distance between trajectory distributions.  Finally, in \Cref{app:online_learning} we recall some standard results from online learning that are used in the analysis of our online algorithms.

\subsection{Information Theory}\label{app:info_theory}

In this section we recall some basic results from information theory that are used throughout the paper.  For a more complete introduction to the topic, see \citet{polyanskiy2025information}.

We first recall the definitions of KL divergence.
\begin{definition}\label{def:kl_divergence_app}
    Let $P, Q$ be two distributions over the same space $\cX$.  The KL divergence between $P$ and $Q$ is defined as
    \begin{align}
        \kld{P}{Q} = \ee_{X \sim P} \left[ \log \frac{dP}{dQ}(X) \right]
    \end{align}
    with $\kld{P}{Q} = \infty$ if $P$ is not absolutely continuous with respect to $Q$.
\end{definition}
We use the following classical properties of KL divergence repeatedly throughout the paper; see \citet{polyanskiy2025information} for details.
\begin{proposition}\label{prop:kl_properties}
    Let $P, Q$ be distributions over the same space $\cX$.  Then it holds that $\kld{P}{Q} \geq 0$ with equality if and only if $P = Q$.  Moreover, $(P, Q) \mapsto \kld{P}{Q}$ is jointly convex in its arguments, i.e. for any $\lambda \in [0,1]$ and any distributions $P_1, P_2, Q_1, Q_2$,
    \begin{align}
        \kld{\lambda P_1 + (1 - \lambda) P_2}{\lambda Q_1 + (1 - \lambda) Q_2} \leq \lambda \cdot \kld{P_1}{Q_1} + (1 - \lambda) \cdot \kld{P_2}{Q_2}.
    \end{align}
    Furthermore, if $P, Q \in \Delta(\cX_1 \times \cdots \cX_n)$, then the KL divergence satisfies a chain rule: if $P_{X_1, \ldots, X_n}$ and $Q_{X_1, \ldots, X_n}$ are the distributions of $(X_1, \ldots, X_n)$ under $P$ and $Q$ respectively, then
    \begin{align}
        \kld{P_{X_1, \ldots, X_n}}{Q_{X_1, \ldots, X_n}} = \sum_{i = 1}^n \ee_{X_1, \ldots, X_{i - 1} \sim P}\left[ \kld{P_{X_i | X_1, \ldots, X_{i - 1}}}{Q_{X_i | X_1, \ldots, X_{i - 1}}} \right].
    \end{align}
\end{proposition}
While KL divergence is a fundamental notion of distance between distributions, it is infinite when the two distributions are not absolutely continuous with respect to each other, which can be problematic in many settings in IL.  We thus also consider the \emph{Hellinger distance}, defined as follows.
\begin{definition}\label{def:hellinger_distance_app}
    Let $P, Q$ be two distributions over the same space $\cX$.  The Hellinger distance between $P$ and $Q$ is defined as
    \begin{align}
        \dhel{P}{Q} &= 1 - \ee_{X \sim P} \left[ \sqrt{\frac{dQ}{dP}(X)} \right] \\
        &= 1 - \ee_{X \sim Q} \left[ \sqrt{\frac{dP}{dQ}(X)} \right] \\
        &= \frac{1}{2} \cdot \int_{\cX} \left( \sqrt{\frac{dP}{d\mu}(x)} - \sqrt{\frac{dQ}{d\mu}(x)} \right)^2 d\mu(x),
    \end{align}
    where $\mu$ is any measure such that $P$ and $Q$ are absolutely continuous with respect to $\mu$, e.g. $\mu = \nicefrac{(P+Q)}{2}$.
\end{definition}
While Hellinger distance is also nonnegative and equal to zero if and only if the two distributions are equal, it is a weaker notion of distance than KL divergence, as it is always bounded by $1$.  In addition, it satisfies a Pinsker-type inequality that relates it to KL divergence.
\begin{proposition}[Pinsker's inequality]\label{prop:pinsker}
    Let $P, Q$ be two distributions over the same space $\cX$.  Then,
    \begin{align}
        \dhel{P}{Q} \leq \kld{P}{Q}.
    \end{align}
\end{proposition}
The Hellinger distance is also intimately related to the total variation distance $\tvd{P}{Q} = \sup_{A \subseteq \cX} |P(A) - Q(A)|$ up to a quadratic factor.
\begin{proposition}\label{prop:hellinger_tv_relation}
    Let $P, Q$ be two distributions over the same space $\cX$.  Then,
    \begin{align}
        \dhel{P}{Q} \leq \tvd{P}{Q} \leq \sqrt{2 \cdot \dhel{P}{Q}}.
    \end{align}
\end{proposition}
Much like KL divergence, Hellinger is jointly convex in its arguments.  In contradistinction to KL divergence, however, Hellinger distance does not satisfy a chain rule, but it does satisfy two weaker properties that will be sufficient for our purposes.
\begin{proposition}\label{prop:hellinger_decomposition}
    Let $P, Q$ be two distributions over $\cX_1 \times \cdots \times \cX_n$.  Let $B_{n+1}(x_{1:n}) = 1$ and for all $i \leq n$, let for some common dominating measure $\mu$,
    \begin{align}
        B_i(x_{1:i-1}) = \int_{\cX_i} \sqrt{\frac{dP_{X_i | X_{1:i-1} = x_{1:i-1}}}{d\mu}(x_i) \cdot \frac{dQ_{X_i | X_{1:i-1} = x_{1:i-1}}}{d\mu}(x_i)} \cdot B_{i+1}(x_{1:i}) d \mu(x_i).
    \end{align}
    Then,
    \begin{align}
        \dhel{P}{Q} = 1 - B_1.
    \end{align}
    In particular, if $P, Q$ are product distributions, i.e. $P = P_1 \times \cdots P_n$ and $Q = Q_1 \times \cdots Q_n$, then
    \begin{align}
        \dhel{P}{Q} = 1 - \prod_{i = 1}^n (1 - \dhel{P_i}{Q_i}) \leq \sum_{i = 1}^n \dhel{P_i}{Q_i}.
    \end{align}
\end{proposition}
As we will see, it is sometimes more convenient to work with the $B_i$ quantities defined above, which we refer to as the Hellinger \textit{affinities}. The second property is more similar in form to the chain rule for KL divergence, but it is an inequality rather than an equality.
\begin{proposition}[Lemma D.2 \citet{foster2024online}]\label{prop:hellinger_subadditive}
    Let $P, Q$ be two distributions over $\cX_1 \times \cdots \times \cX_n$.  Then,
    \begin{align}
        \dhel{P}{Q} \leq \sum_{h = 1}^H \ee_{P}\left[ \dhel{P_{X_h | X_{1:h-1}}}{Q_{X_h | X_{1:h-1}}} \right].
    \end{align}
\end{proposition}

\subsection{Maximum Likelihood Estimation}\label{app:mle}

Maximum Likelihood Estimation (MLE) is a fundamental statistical estimation technique that returns the density in a given class that maximizes the likelihood of the observed data.  More precisely if $\Pi'$ is a class of conditional distributions $\cS \mapsto \Delta(\cA)$ and $(s^{(i)}, a^{(i)})_{i = 1}^n$ are samples, then the MLE $\pihat$ is defined as
\begin{align}\label{eq:mle}
    \pihat \in \argmax_{\pi \in \Pi'} \sum_{i = 1}^n \log \pi(a^{(i)} | s^{(i)}) = \argmin_{\pi \in \Pi'} \sum_{i  =1}^n -\log \pi(a^{(i)} | s^{(i)}).
\end{align}
While the following result is due to \citet{geer2000empirical,zhang2006epsilon}, we state the version from \citet{foster2024behavior} that is most relevant to our setting.
\begin{theorem}[Proposition B.1 from \cite{foster2024behavior}]\label{thm:mle}
    Let $\Pi'$ be a finite class of conditional distributions $\cS \mapsto \Delta(\cA)$ and let $\pistar \in \Pi'$.  Let $(s^{(i)}, a^{(i)})_{i = 1}^n$ be i.i.d. samples from $\pp^{\pistar}$, where $\pp^{\pistar}$ is the distribution over $\cS \times \cA$ induced by sampling $s \sim \rho$ and $a \sim \pistar(\cdot | s)$.  If $\pihat$ is the MLE in \eqref{eq:mle}, then with probability at least $1 - \delta$,
    \begin{align}
        \dhel{\pp^{\pistar}}{\pp^{\pihat}} \leq 12 \cdot \frac{\log\left( \nicefrac{\abs{\Pi'}}\delta \right)}{n}.
    \end{align}
\end{theorem}
An immediate consequence of \Cref{thm:mle} is that Behavior Cloning with the logarithmic loss achieves Hellinger distance that scales in a horizon-free manner with the number of samples.
\begin{corollary}[Proposition 2.1 from \cite{foster2024behavior}]\label{cor:il_hellinger}
    Let $\Pi$ be a finite class of policies $\cS \times [H] \to \Delta(\cA)$ and let $\pistar \in \Pi$.  Let $\tau^{(i)} = (s_1^{(i)}, a_1^{(i)}, \ldots, s_H^{(i)}, a_H^{(i)})$ be i.i.d. trajectories from $\pp^{\pistar}$ and let
    \begin{align}
        \pihat \in \argmin_{\pi \in \Pi} \sum_{i = 1}^n \sum_{h = 1}^H -\log \pi(a_h^{(i)} | s_h^{(i)}).
    \end{align}
    Then with probability at least $1 - \delta$,
    \begin{align}
        \dhel{\pp^{\pihat}}{\pp^{\pistar}} \lesssim \frac{\log\left( \nicefrac{\abs{\Pi}}{\delta} \right)}{n}.
    \end{align}
    In particular,
    \begin{align}
        \ee\left[ \dhel{\pp^{\pihat}}{\pp^{\pistar}} \right] \lesssim \frac{\log\left( n \cdot \abs{\Pi} \right)}{n}.
    \end{align}
\end{corollary}
Indeed, \Cref{cor:il_hellinger} follows from observing that the policy does not affect the transition densities and thus
\begin{align}
    \pihat &= \argmin_{\pi \in \Pi} \sum_{i = 1}^n \sum_{h = 1}^H - \log\left(\pi\left(a^{(i)}_h \,|\, s^{(i)}_h\right)\right) = \argmin_{\pi \in \Pi} \sum_{i = 1}^n \sum_{h = 1}^H - \log\left( \pi\left(a^{(i)}_h \,|\, s^{(i)}_h\right) \cdot P_h\left(s^{(i)}_{h+1} \,|\, a^{(i)}_h, s^{(i)}_h\right) \right) \\
    &= \argmin_{\pi \in \Pi}  \sum_{i = 1}^n -\log \pp^{\pi}(\tau^{(i)}).
\end{align}
The second statement is immediate from the first and the fact that Hellinger distance is bounded by $1$.

\subsection{Imitation Learning}\label{app:il_prelims}

Recent work has revealed the fundamental importance of the Hellinger distance in the theory of interactive decision making \citep{foster2021efficient,foster2021statistical,rohatgi2025computational}.  In this section we recall the key fact that, at least in a minimax sense, Imitation Learning is essentially equivalent to learning the trajectory distribution of the expert in Hellinger distance.  Following \citet{foster2024behavior}, we consider the deterministic and stochastic cases separately.  In the deterministic case, we have the following result.
\begin{theorem}[Theorem 2.1 from \cite{foster2024behavior}]\label{thm:il_hellinger_deterministic}
    Let $\pistar$ be a \emph{deterministic} policy and let $\pihat$ be an arbitrary (possibly stochastic) policy.  Then,
    \begin{align}
        J(\pistar) - J(\pihat) \leq 4 R \cdot \dhel{\pp^{\pihat}}{\pp^{\pistar}},
    \end{align}
    where $R$ is such that
    \begin{align}
        0 \leq \sum_{h = 1}^H r(s_h, a_h) \leq R
    \end{align}
    for all trajectories $\tau = (s_1, a_1, \ldots, s_H, a_H)$.
\end{theorem}
In the stochastic case, we recall the following result, which is a consequence of the more general results in \citet{foster2024behavior}.
\begin{theorem}[Theorem 3.1 and Proposition 3.1 from \citet{foster2024behavior}]\label{thm:il_hellinger_stochastic}
    Let $\pistar$ be a (possibly stochastic) policy and let $\pihat$ be any policy.  Then,
    \begin{align}
        J(\pistar) - J(\pihat) \lesssim \sqrt{ R^2 \cdot \dhel{\pp^{\pistar}}{\pp^{\pihat}}} + R \log\left( \frac{R}{\dhel{\pp^{\pihat}}{\pp^{\pistar}}} \right) \cdot \dhel{\pp^{\pihat}}{\pp^{\pistar}}.
    \end{align}
\end{theorem}
Note that \citet{foster2024behavior} proves a tighter result, replacing the $R^2$ under the square root with
\begin{align}
    \sigma_{\pistar}^2 = \sum_{h = 1}^H \ee^{\pistar}\left[ \left( \ee\left[Q_h^{\pistar}(s_h, a_h) | s_h\right] - Q_h^{\pistar}(s_h, a_h)\right)^2 \right],
\end{align}
where $Q^{\pistar}$ is the $Q$-function of $\pistar$ (cf. e.g. \citet{sutton1998reinforcement}).  The second cited result above, \citet[Proposition 3.1]{foster2024behavior}, controls $\sigma_{\pistar}^2$ by $R^2$ leading to the version stated above.  While $\sigma_{\pistar}^2$ is a significantly more refined quantity that can be much smaller than $R^2$ in many settings of interest, such as when $\pistar$ is near-deterministic, we use the version stated above for simplicity.  We leave to future work the interesting problem of understanding the precise role of $\sigma_{\pistar}^2$ in the noisy expert setting and whether it can be used to obtain improved guarantees in certain regimes.

Note that \citet[Theorem G.3]{foster2024behavior} shows that the above reductions are tight up to constants in the sense that for any MDP and pair of policies, there exist reward functions that achieve the reverse inequalities up to constant factors.  Thus, in a minimax sense, IL is essentially equivalent to learning the trajectory distribution of the expert in Hellinger distance.  More formally, the result states the following.
\begin{theorem}[Theorem G.3 from \citet{foster2024behavior}]\label{thm:il_hellinger_lower_bound}
    Let $M$ be an MDP.  Then for any pair of policies $\pistar, \pihat$ and any $\sigma > 0$ there exists a reward function $r$ such that $\sigma_{\pistar}^2 \leq \sigma^2$ and
    \begin{align}
        J(\pistar) - J(\pihat) \gtrsim \sqrt{\sigma^2 \cdot \dhel{\pp^{\pistar}}{\pp^{\pihat}}}.
    \end{align}
    Moreover, there exists a reward function $r$ such that
    \begin{align}
        J(\pistar) - J(\pihat) \gtrsim R \cdot \dhel{\pp^{\pistar}}{\pp^{\pihat}}
    \end{align}
    and the same conclusion applies even if we assume $\pistar$ is deterministic.
\end{theorem}

By \Cref{thm:il_hellinger_lower_bound}, for any of our lower bounds on regret, it suffices to construct instances where learning the trajectory distribution of the expert in Hellinger distance is hard, which is what we do.

\subsection{Online Learning}\label{app:online_learning}

Several of our results rest on the use of online learning algorithms, and in particular the exponential weights algorithm, to learn policies in an online fashion.  In this section we recall the key definitions and results from online learning that are used throughout the paper.  For a more complete introduction to the topic, see \citet{cesa2006prediction}.

We are only concerned with online learning over finite classes of experts in this work.  For a finite class $\Pi$, the online learning problem proceeds in rounds $t \in [n]$ as follows.  The learner at the beginning is informed of a loss function $\ell: \Pi \times \cY \to \rr$ and must choose a distribution $w_t \in \Delta(\Pi)$ over the experts.  Then, an outcome $y_t \in \cY$ is revealed and the learner suffers loss $\ee_{\pi \sim w_t}[\ell(\pi, y_t)]$.  The goal of the learner is to minimize regret, defined as
\begin{align}
    \reg_n = \sum_{t = 1}^n \ee_{\pi \sim w_t}[\ell(\pi, y_t)] - \min_{\pi \in \Pi} \sum_{t = 1}^n \ell(\pi, y_t).
\end{align}
The exponential weights algorithm is a simple and classical online learning algorithm that achieves optimal regret guarantees in the adversarial setting, given in \Cref{alg:exponential_weights}.

\begin{algorithm}[t]
\caption{Exponential Weights}
\label{alg:exponential_weights}
\begin{algorithmic}[1]
\Require Number of rounds $n$, loss function $\ell: \Pi \times \cY \to \rr$, learning rate $\lambda > 0$.
\State Set  $w_1 = \Unif(\Pi)$.
\For{$t = 1$ to $n$}
    \State Observe $y_t$ and suffer loss $\ee_{\pi \sim w_t}[\ell(\pi, y_t)]$.
    \State Update $w_{t+1}(\pi) \propto w_t(\pi) \cdot e^{-\lambda \cdot \ell(\pi, y_t)}$.
\EndFor
\end{algorithmic}
\end{algorithm}

We make use of the following definition.
\begin{definition}\label{def:exp_concave}
    A loss $\ell: \Pi \times \cY \to \rr$ is $\beta$-exp-concave if for all $y \in \cY$, the function $\pi \mapsto e^{-\beta \cdot \ell(\pi, y)}$ is concave.
\end{definition}
We recall the following result about the regret of the exponential weights algorithm when the loss is exp-concave.
\begin{proposition}[Proposition 3.1 from \citet{cesa2006prediction}]\label{prop:exp_weights_regret}
    If $\ell$ is $\beta$-exp-concave, then the exponential weights algorithm with learning rate $\lambda = \beta$ achieves regret
    \begin{align}
        \reg_n \leq \frac{\log(\abs{\Pi})}{\beta}.
    \end{align}
\end{proposition}
We emphasize that this regret is independent of the precise choice of $y_1, \dots, y_n$, and thus holds even if the outcomes are chosen by an adversary that observes the learner's distribution $w_t$ at each round.   Finally, we will recall a more general definition that is similar to exp-concavity but allows for a better regret bound.
\begin{definition}\label{def:mixability}
    A loss $\ell: \Pi \times \cY \to \rr$ is $\beta$-mixable if for all $y \in \cY$ and all distributions $w \in \Delta(\Pi)$, there exists $\pi_w \in \Pi$ such that
    \begin{align}
        e^{-\beta \cdot \ell(\pi_w, y)} \geq \ee_{\pi \sim w} \left[ e^{-\beta \cdot \ell(\pi, y)}\right].
    \end{align}
\end{definition}
Mixability is a classical notion in online learning and a more complete discussion of it can be found in \citet{cesa2006prediction}.  We recall the following guarantee for the exponential weights algorithm when the loss is mixable.
\begin{proposition}[Proposition 3.2 from \citet{cesa2006prediction}]\label{prop:exp_weights_mixability}
    If $\ell$ is $\beta$-mixable, then the exponential weights algorithm with learning rate $\lambda = \beta$ achieves regret
    \begin{align}
        \reg_n \leq \frac{\log(\abs{\Pi})}{\beta}.
    \end{align}
\end{proposition}

%% file: body_clean/app_offline_proofs.tex
\section{Proofs from Section \ref{sec:offline}}

In this appendix, we prove the two results in \Cref{sec:offline} involving offline imitation learning with a noisy expert.  We begin by proving the upper bound that scales exponentially in horizon in  \Cref{app:offline_ub} before proving that this exponential dependence is necessary in the worst case in \Cref{app:offline_lb}.  We conclude in \Cref{app:offline_il_eps_lb} by showing that any offline IL algorithm must suffer from this exponential dependence, making offline IL fundamentally intractable in the noisy expert setting.

\subsection{Proof of Theorem \ref{thm:offline_bc}}\label{app:offline_ub}

We first state a slightly tighter version of the main theorem, which recovers \Cref{thm:offline_bc}.  We first conclude the proof under $\kappa$-domination and then show how to get a weaker guarantee in the absence of $\kappa$-domination.
\begin{theorem}\label{thm:offline_bc_tight}
    Let $\pi, \pi'$ be policies that $\kappa$-dominate the corruption $\nu$.  For any $0 \leq \eta < 1$, it holds that
    \begin{align}
        \dhel{\pp^{\pi}}{\pp^{\pi'}} \leq \frac{2 (1 + \eta (2 \kappa - 1))}{(1 - \eta)^{H + 2}} \cdot \dhel{\pp^{\pi_\eta}}{\pp^{\pi_\eta'}}.
    \end{align}
\end{theorem}
Noting that $\eta < 1$ immediately shows that the first statement of \Cref{thm:offline_bc} follows from \Cref{thm:offline_bc_tight}. To prove the latter, we fix policies $\pi, \pi'$ and introduce some notation in order to help with the proof.  First, for any state $s \in \cS$, define
\begin{align}
    \pp^{\pi}_h(s) = \pp^{\pi}\left( a_h, \tau_{h+1:H} | s_h =s \right)
\end{align}
to be the conditional distribution of the action $a_h$ and the future trajectory of states and actions conditioned on the event that $s_h = s$.  We then define
\begin{align}
    D_h(s) = \dhel{\pp_h^\pi(s)}{\pp_h^{\pi'}(s)} \quad \text{and} \quad D_h^\eta(s) = \dhel{\pp_h^{\pi_\eta}(s)}{\pp_h^{\pi'_\eta}(s)}.
\end{align}
We have the following recursion.
\begin{lemma}\label{lem:dh_recursion}
    Let $\pi, \pi'$ be any two policies.  Then $D_{H+1}(s) = 0$ for all $s \in \cS$ and for any $s \in \cS$ and $h \leq H$, it holds that
    \begin{align}
        D_h(s) = \dhel{\pi_h(\cdot | s)}{\pi_h'(\cdot | s)} + \sum_{a \in \cA} \sqrt{\pi_h(a|s) \cdot \pi_h'(a | s)} \cdot \int D_{h+1}(s') d P_h(s' | s, a).
    \end{align}
\end{lemma}
\begin{proof}
    We use \Cref{prop:hellinger_decomposition}.  Indeed, we compute
    \begin{align}
        1 - D_h(s) 
        &= \ee^{\pi}\left[\sqrt{\frac{\pi_h'(a_h | s)}{\pi_h(a_h|s)}}  \left( 1- D_{h+1}(s_{h+1}) \right)  \right] \\
        &= \sum_{a \in \cA} \sqrt{\pi_h(a | s) \cdot \pi_h'(a | s)} \cdot \ee_{s_{h+1} \sim P_h(\cdot | s, a)}\left[ 1 - D_{h+1}(s_{h+1}) \right].
    \end{align}
    \dhcomment{Also I didn't follow the conditioning.}
    The result follows by rearranging and using the definition of the Hellinger squared distance.
\end{proof}

\Cref{lem:dh_recursion} reduces the problem of bounding the contraction in Hellinger distance of trajectory distributions to bounding the contraction on a per-step basis.  We thus prove that in the case of $\kappa$-dominated corruption distributions $\nu$, we can control this contraction at each time step $h$.
\begin{lemma}\label{lem:dhel_contraction_kappa}
    Let $\pi,\pi'$ be arbitrary policies and suppose that $\nu$ is $\kappa$-dominated, i.e., for all $1 \leq h \leq H$, all $s \in \cS$ and $a \in \supp(\pi_h(\cdot | s)) \cup \supp(\pi_h'(\cdot | s))$, it holds that
    \begin{align}
        \nu_h(a | s) \leq \kappa \left( \pi_h(a | s) + \pi_h'(a | s) \right).
    \end{align}
    Then for any $0 \leq \eta < 1$, any state $s \in \cS$ and time $1 \leq h \leq H$, it holds that
    \begin{align}
        \dhel{\pi_h(\cdot | s)}{\pi_h'(\cdot | s)} \leq \frac{2 \left( 1 + \eta \left( 2 \kappa - 1 \right) \right)}{(1 - \eta)^2} \cdot \dhel{\pi_{\eta,h}(\cdot | s)}{\pi_{\eta, h}'(\cdot | s)}.
    \end{align}
\end{lemma}
\begin{proof}
    For the sake of notational simplicity, we will suppress the $h$ in this proof; in addition, because we have fixed a state $s$, we will write $\pi(a)$ for $\pi_h(a | s)$.  We observe that for any $a \in \supp(\pi) \cup \supp(\pi')$, it holds that
    \begin{align}
        \pi_\eta(a) + \pi_{\eta}'(a) &= \left( 1 - \eta\right)\left( \pi(a) + \pi'(a) \right) + 2 \eta \cdot \nu(a) \\
        &\leq \left( 1 - \eta\right)\left( \pi(a) + \pi'(a) \right) + 2 \eta \kappa \cdot \left( \pi(a) + \pi'(a) \right) \\
        &= \left( 1 + \eta (2 \kappa - 1) \right)\left( \pi(a) + \pi'(a) \right).
    \end{align}
    Thus,
    \begin{align}
        \left( \sqrt{\pi_\eta(a)} + \sqrt{\pi_\eta'(a)} \right)^2 \leq 2 \left( \pi_\eta(a) + \pi_\eta'(a) \right) \leq 2\left( 1 + \eta (2 \kappa - 1) \right) \left( \sqrt{\pi(a)} + \sqrt{\pi'(a)} \right)^2,
    \end{align}
    where we used the fact that $a + b \leq \left( \sqrt{a} + \sqrt{b} \right)^2 \leq 2 (a + b)$ for nonnegative $a,b$.  Plugging into the definition of Hellinger distance, we see that
    \begin{align}
        \dhel{\pi_\eta}{\pi_\eta'} &= \frac 12 \sum_{a \in \cA} \left( \sqrt{\pi_\eta(a)} - \sqrt{\pi_\eta'}(a) \right)^2 \\
        &= \frac 12 \sum_{a \in \cA} \frac{\left( \pi_\eta(a) - \pi_\eta'(a) \right)^2}{\left( \sqrt{\pi_\eta(a)} + \sqrt{\pi_\eta'(a)} \right)^2} \\
        &\geq \frac{\left( 1 - \eta \right)^2}{2(1 + \eta (2 \kappa - 1))} \cdot \frac{1}{2} \sum_{a \in \cA} \frac{\left( \pi(a) - \pi'(a) \right)^2}{\left( \sqrt{\pi(a)} + \sqrt{\pi'(a)} \right)^2} \\
        &= \frac{\left( 1 - \eta \right)^2}{2(1 + \eta (2 \kappa - 1))} \cdot \dhel{\pi}{\pi'}.
    \end{align}
    The result follows.
\end{proof}
\begin{remark}\label{rmk:kappa_domination_necessity}
    Note that it is precisely in \Cref{lem:dhel_contraction_kappa} that $\kappa$-domination is used and in the proof one can understand the necessity of such an assumption.  Indeed, the problem is precisely that the map $\eta \mapsto \sqrt{\eta}$ is not Lipschitz near $\eta = 0$.  For a simple example of what can go wrong, let
    \begin{align}
        \pi = \Bernoulli(0) = \delta_0, \quad \pi' = \Bernoulli(\epsilon), \quad \text{and} \quad \nu = \Bernoulli\left(\nicefrac 12\right).
    \end{align}
    An elementary computation then shows that
    \begin{align}
        \dhel{\pi}{\pi'} \asymp \epsilon \quad \text{but} \quad \dhel{\pi_\eta}{\pi'_\eta} \asymp \epsilon^2,
    \end{align}
    for $\eta > 0$.  Thus, in order to prevent such a quadratic blowup, we need to ensure that $\nu$ does not put too much mass on actions that receive small, but positive, probability under $\pi$ or $\pi'$, which is precisely what $\kappa$-domination ensures.
\end{remark}
\begin{remark}\label{rmk:eta_quadratic_dependence}
    Note that even with $\kappa$-domination, the $(1 - \eta)^2$ dependence in the comparison is real.  Indeed, suppose that $\pistar = \delta_{a_1}$ and $\pihat = \delta_{a_2}$ for some $a_1 \neq a_2$, and $\nu = \nicefrac 12\left( \pistar + \pihat \right)$.  Then we have $\kappa$-domination with $\kappa = 1$.  On the other hand, we have $\dhel{\pistar}{\pihat} = 1$, whereas
    \begin{align}
        \dhel{\pistar_\eta}{\pihat_\eta} = 1 - \sqrt{\eta(2 - \eta)} \asymp \frac{(1 - \eta)^2}{2}.
    \end{align}
    Thus the $(1 - \eta)^2$ dependence is tight up to constant factors.
\end{remark}

We are now ready to conclude the proof of the result.
\begin{proof}[Proof of \Cref{thm:offline_bc_tight}]
    We prove the following claim by reverse induction from $H, \dots, 1$: it holds for any $s \in \cS$ that
    \begin{align}\label{eq:dhel_kappa_1}
        D_h^\eta(s) \geq (1 - \eta)^{H + 1 - h} \cdot  \frac{(1 - \eta)^2}{2 \left( 1 + \eta (2 \kappa - 1) \right)} \cdot D_{h}(s).
    \end{align}
    The case $h = H$ follows immediately from \Cref{lem:dhel_contraction_kappa}.  We thus suppose that \eqref{eq:dhel_kappa_1} holds for $h+1$.  By \Cref{lem:dh_recursion}, it holds that
    \begin{align}
        D_h^\eta(s) &= \dhel{\pi_{\eta,h}(\cdot | s)}{\pi_{\eta,h}'(\cdot | s)} + \sum_{a \in \cA} \sqrt{\pi_{\eta,h}(a | s) \cdot \pi_{\eta,h}'(a | s)} \cdot \int D_{h+1}^\eta(s') d P_h(s' | s, a) \\
        &\geq \frac{(1 - \eta)^2}{2 \left( 1 + \eta (2 \kappa - 1) \right)} \cdot \dhel{\pi_h(\cdot | s)}{\pi_h'(\cdot | s)} + \sum_{a \in \cA} \sqrt{\pi_{\eta,h}(a | s) \cdot \pi_{\eta,h}'(a | s)} \cdot \int D_{h+1}^\eta(s') d P_h(s' | s, a) \\
        &\geq \frac{(1 - \eta)^2}{2 \left( 1 + \eta (2 \kappa - 1) \right)} \cdot \dhel{\pi_h(\cdot | s)}{\pi_h'(\cdot | s)} + (1 - \eta) \cdot \sum_{a \in \cA} \sqrt{\pi_h(a | s) \cdot \pi_h'(a | s)} \cdot \int  D_{h+1}^\eta(s') d P_h(s' | s, a) \\
        &\geq \frac{(1 - \eta)^2}{2 \left( 1 + \eta (2 \kappa - 1) \right)} \cdot \dhel{\pi_h(\cdot | s)}{\pi_h'(\cdot | s)} \\
        &\quad + (1 - \eta) \cdot (1 - \eta)^{H - h} \cdot  \frac{(1 - \eta)^2}{2 \left( 1 + \eta (2 \kappa - 1) \right)} \cdot  \sum_{a \in \cA} \sqrt{\pi_h(a | s) \cdot \pi_h'(a | s)} \cdot \int  D_{h+1}(s') d P_h(s' | s, a) \\
        &\geq \frac{(1 - \eta)^{H+3-h}}{2 \left( 1 + \eta (2 \kappa - 1) \right)} \cdot \left( \dhel{\pi_h(\cdot | s)}{\pi_h'(\cdot | s)} +  \sum_{a \in \cA} \sqrt{\pi_h(a | s) \cdot \pi_h'(a | s)} \cdot \int  D_{h+1}(s') d P_h(s' | s, a)\right) \\
        &=  \frac{(1 - \eta)^{H + 3 - h}}{2 \left( 1 + \eta (2 \kappa - 1) \right)} \cdot D_{h}(s) ,
    \end{align}
    where the first inequality used \Cref{lem:dhel_contraction_kappa}, the second inequality used the fact that
    \begin{align}
        \sqrt{\pi_{\eta,h}(a | s) \cdot \pi_{\eta,h}'(a | s)} &= \sqrt{\left( (1 - \eta) \pi_h(a | s) + \eta \nu_h(a | s) \right) \left( (1 - \eta) \pi_h'(a | s) + \eta \nu_h(a | s) \right)} \\
        &\geq (1-  \eta) \cdot \sqrt{\pi_h(a | s) \cdot \pi_h'(a | s)},
    \end{align}
    the third inequality used the inductive hypothesis, and the final inequality again used \Cref{lem:dh_recursion}.  The result follows immediately.
\end{proof}
Note that in the case that $\pi, \pi'$ are both \emph{deterministic} policies, then any $\nu$ is $\kappa$-dominated with $\kappa = 1$ and thus
\begin{align}
    \dhel{\pp^{\pi}}{\pp^{\pi'}} \lesssim (1 - \eta)^{- H - 2} \cdot \dhel{\pp^{\pi_\eta}}{\pp^{\pi_\eta'}} \leq  \frac{e^{\nicefrac{\eta H}{1-\eta}}}{(1 - \eta)^2} \cdot \dhel{\pp^{\pi_\eta}}{\pp^{\pi_\eta'}} .
\end{align}

We now show how to get a weaker guarantee in the absence of $\kappa$-domination.  We restate the result with constants made explicit now.  
\begin{proposition}\label{prop:hellinger_contraction_arbitrary}
    Let $M$ be a horizon $H$ MDP and let $\pi, \pi'$ be two policies.  For any choice of $\nu$ and any $0 < \eta < 1$, it holds that
    \begin{align}
        \dhel{\pp^{\pi}}{\pp^{\pi'}} \leq \sqrt{2 \cdot \frac{(1 - \eta)^{- H} - 1}{\eta ( 1- \eta)} \cdot \dhel{\pp^{\pi_\eta}}{\pp^{\pi'_\eta}}}.
    \end{align}
\end{proposition}
We first prove the following general result on the contraction of Hellinger distance under arbitrary corruptions.
\begin{lemma}\label{lem:hellinger_contraction_arbitrary}
    Let $P, Q$ be two distributions over a common space $\mathcal{X}$ and let $\nu$ be an arbitrary distribution over $\mathcal{X}$.  For any $0 < \eta < 1$, letting $P_\eta = (1 - \eta) P + \eta \nu$ and $Q_\eta = (1 - \eta) Q + \eta \nu$, it holds that
    \begin{align}
        \dhel{P}{Q} \leq \frac{\sqrt{2}}{1 - \eta} \cdot \sqrt{\dhel{P_\eta}{Q_\eta}}.
    \end{align}
\end{lemma}
\begin{proof}
    By \Cref{prop:hellinger_tv_relation}, it holds that
    \begin{align}
        \sqrt{2 \dhel{P_\eta}{Q_\eta}} &\geq \tvd{P_\eta}{Q_\eta} \\
        &= \sup_{0 \leq f \leq 1} \left| \ee_{P_\eta}[f] - \ee_{Q_\eta}[f] \right| \\
        &= (1 - \eta) \cdot \sup_{0 \leq f \leq 1} \left| \ee_{P}[f] - \ee_{Q}[f] \right| \\
        &= (1 - \eta) \cdot \tvd{P}{Q} \\
        &\geq (1 - \eta) \cdot \dhel{P}{Q},
    \end{align}
    where the second equality comes from the linearity of expectation and the final inequality again comes from \Cref{prop:hellinger_tv_relation}.
\end{proof}
We can now prove the proposition.
\begin{proof}[Proof of \Cref{prop:hellinger_contraction_arbitrary}]
    We will apply backward induction and \Cref{lem:dh_recursion}.  Indeed, we use the identical notation as that used in the proof thereof.  We suppose that there are constants $C_h$ for $h = H, \dots, 1$ satisfying $C_H = 2^{-1} \cdot(1 - \eta)^2$ and
    \begin{align}
        \frac{1}{C_h} = \frac{2}{(1 - \eta)^2} + \frac{1}{(1 - \eta) \cdot C_{h+1}}
    \end{align}
    such that
    \begin{align}
        D_h^\eta(s) \geq C_h \cdot D_{h}(s)^2 \quad \text{for all } s \in \cS.
    \end{align}
    By \Cref{lem:hellinger_contraction_arbitrary}, the statement holds for $h = H$.  Now, by \Cref{lem:dh_recursion}, it holds that
    \begin{align}
        D_h^\eta(s) &= \dhel{\pi_{\eta,h}(\cdot | s)}{\pi_{\eta,h}'(\cdot | s)} + \sum_{a \in \cA} \sqrt{\pi_{\eta,h}(a|s) \cdot \pi_{\eta,h}'(a | s)} \cdot \int D_{h+1}^\eta(s') d P_h(s' | s, a) \\
        &\geq 2^{-1} \cdot(1 - \eta)^2 \cdot \dhel{\pi_h(\cdot | s)}{\pi_h'(\cdot | s)}^2 \\
        &\quad + (1 - \eta) \cdot \sum_{a \in \cA} \sqrt{\pi_{h}(a|s) \cdot \pi_{h}'(a | s)} \cdot \int D_{h+1}^\eta(s') d P_h(s' | s, a) \\
        &\geq 2^{-1} \cdot(1 - \eta)^2\cdot \dhel{\pi_h(\cdot | s)}{\pi_h'(\cdot | s)}^2 \\
        &\quad + (1 - \eta) \cdot \sum_{a \in \cA} \sqrt{\pi_{h}(a|s) \cdot \pi_{h}'(a | s)} \cdot \int C_{h+1} D_{h+1}(s')^2 d P_h(s' | s, a) \\
        &\geq  2^{-1} \cdot(1 - \eta)^2 \cdot \dhel{\pi_h(\cdot | s)}{\pi_h'(\cdot | s)}^2 \\
        &\quad + C_{h+1} (1 - \eta) \cdot \left(\sum_{a \in \cA} \sqrt{\pi_{h}(a|s) \cdot \pi_{h}'(a | s)} \cdot \int  D_{h+1}(s') d P_h(s' | s, a)\right)^2,
    \end{align}
    where the second inequality follows from the inductive hypothesis and the final inequality follows from Jensen's inequality and the fact that $\sum_a \sqrt{\pi_h(a|s) \cdot \pi_h'(a | s)} \leq 1$.  By AM-GM, it holds that for $a,b,x,y \geq 0$,
    \begin{align}
        a \cdot x^2 + b \cdot y^2 \geq \frac{1}{\frac 1a + \frac 1b} \cdot (x + y)^2.
    \end{align}
    Thus, applying this to the final expression in the above display, we see that
    \begin{align}
        D_h^\eta(s) &\geq C_h \cdot \left( \dhel{\pi_h(\cdot | s)}{\pi_h'(\cdot | s)} +  \sum_{a \in \cA} \sqrt{\pi_{h}(a|s) \cdot \pi_{h}'(a | s)} \cdot \int  D_{h+1}(s') d P_h(s' | s, a)\right)^2 \\
        &= C_h \cdot D_h(s)^2,
    \end{align}
    where the equality follows from \Cref{lem:dh_recursion}.  Thus, the inductive step is complete.  

    It remains to bound $C_h$ itself.  Letting $c_h = C_h^{-1}$, we see that
    \begin{align}
        c_H = 2(1-\eta)^{-2},
        \qquad
        c_h = 2(1 - \eta)^{-2} + (1 - \eta)^{-1} \cdot c_{h+1}
    \end{align}
    for $h < H$.  Thus by induction, it holds that
    \begin{align}
        c_1 = 2 \cdot\frac{(1 - \eta)^{- H} - 1}{\eta ( 1- \eta)}.
    \end{align}
    The result follows.
\end{proof}

Finally, we note that \Cref{thm:offline_bc} follows immediately from \Cref{thm:offline_bc_tight} and \Cref{prop:hellinger_contraction_arbitrary}, concluding the proof.

\subsection{Proof of Proposition \ref{prop:offline_lower_bound}}\label{app:offline_lb}

We will prove three lower bounds: one for arbitrary $\kappa$, one for $\kappa = 1$ with a deterministic expert, and one that holds absent $\kappa$-domination.  We begin by stating the common construction for the first two, before proving these first two results separately.

We will take an MDP that has $H+1$ states $s_1, \dots, s_H$ and an absorbing state $\perp$.  We will suppose the action space $\cA = \left\{ a_1, a_2, a_3 \right\}$.  Let the transition functions for $h < H$ be
\begin{align}
    P_h(s' | s, a) = \begin{cases}
        \delta_{s_{h+1}} & a = a_1 \text{ and }  s = s_h \\
        \delta_\perp & a \in \left\{ a_2, a_3 \right\} \text{ or } s = \perp
    \end{cases}.
\end{align}
In other words, $\perp$ is an absorbing state that is reached by taking a `wrong' action at any state, otherwise we deterministically transition to the next state. Since the only place where the two clean policies will differ is hidden behind an $H$-step survival event, and corruption before that point sends the learner to a state where the policies are identical, offline noisy trajectories almost never reveal the difference between the two policies, even though they differ with probability 1 in the clean setting.  

We now state the formal proposition for the case of $\kappa = 1$ and deterministic experts.
\begin{proposition}\label{prop:offline_lb_kappa_1}
    For any $H \geq 1$ and $0 \leq \eta < 1$, there exists an MDP, a policy class $\Pi$ of size 2, and a corruption distribution $\nu$ such that both policies in $\Pi$ are deterministic, but
    \begin{align}
        \dhel{\pp^{\pistar}}{\pp^{\pihat}} = 1 \quad \text{but} \quad \dhel{\pp^{\pistar_\eta}}{\pp^{\pihat_\eta}} \leq (1 - \eta)^{H + 1}.
    \end{align}
\end{proposition}
\begin{proof}
    We will consider the MDP described above with policy class $\Pi = \left\{ \pistar, \pihat \right\}$ of size 2 such that $\pihat( \cdot | s_h) = \pistar(\cdot | s_h) = \delta_{a_1}$ (an atom on $a_1$) for $h < H$, $\pistar(s_H) = \delta_{a_1}$, $\pihat(s_H) = \delta_{a_2}$, and both policies take the same action at $\perp$.  Finally, suppose that
\begin{align}
    \nu(\cdot | s_H) = \frac{\delta_{a_1} + \delta_{a_2}}{2} \quad \text{and} \quad \nu(\cdot | s_h) = \delta_{a_3} \text{ for } h < H.
\end{align}
We claim that for any $0 \leq \eta < 1$, the result of the proposition holds for this construction.  Note that the $\eta = 0$ case is trivial, so suppose that $\eta > 0$.  We now observe that $\dhel{\pp^{\pistar}}{\pp^{\pihat}} = 1$ because the policies deterministically differ in the final time step, which is reached with probability 1.  Thus we focus on upper bounding $\dhelinline{\pp^{\pistar_\eta}}{\pp^{\pihat_\eta}}$.

Note first that for $\pi \in \left\{ \pistar, \pihat \right\}$, with probability at least $1 - (1 - \eta)^{H-1}$, it holds that the trajectory under $\pi$ will transition to $\perp$ before the final time step; moreover, conditional on this event, both corrupted policies induce the same distribution.  On the complementary event of no contamination up to step $H$, it holds that
\begin{align}
    \pistar_\eta = \left( 1 - \frac{\eta}{2} \right)\cdot  \delta_{a_1} + \frac{\eta}{2} \cdot \delta_{a_2} \quad \text{and} \quad \pihat_\eta = \frac{\eta}{2} \cdot \delta_{a_1} + \left(1 - \frac \eta 2  \right) \cdot \delta_{a_2}.
\end{align}
Thus,
\begin{align}
    \dhel{\pistar_\eta(\cdot | s_H)}{\pihat_\eta(\cdot | s_H)} = 1 - \sqrt{\eta (2 - \eta)} = \frac{(1 - \eta)^2}{1 + \sqrt{\eta(2 - \eta)}} \leq (1 - \eta)^2.
\end{align}
Combining these observations, we see that
\begin{align}
    \dhel{\pp^{\pistar_\eta}}{\pp^{\pihat_\eta}} \leq (1 - \eta)^{H - 1} \cdot (1 - \eta)^2 = (1 - \eta)^{H+1}.
\end{align}
The result follows.
\end{proof}
We now state the formal proposition for the case of arbitrary $\kappa$.
\begin{proposition}\label{prop:offline_lb_kappa}
    For any $H \geq 1$, $\kappa \geq 1$, and $0 \leq \eta < 1$, there exists an MDP, a policy class $\Pi$ of size 2, and a corruption distribution $\nu$ such that both policies in $\Pi$ $\kappa$-dominate $\nu$, but
    \begin{align}
        \dhel{\pp^{\pi}}{\pp^{\pi'}} \geq 4 \eta \cdot \kappa \cdot (1 - \eta)^{-H - 1}\cdot  \dhel{\pp^{\pi_\eta}}{\pp^{\pi_\eta'}}.
    \end{align}
\end{proposition}
\begin{proof}
    We will consider the MDP described above but now let for $h < H$,
    \begin{align}
        \pi_h(\cdot | s_h) = \pi_h'(\cdot | s_h) = (1 - \epsilon) \cdot \delta_{a_1} + \epsilon \cdot \delta_{a_3}
    \end{align}
    and
    \begin{align}
        \pi_H(\cdot | s_H) = (1 - \epsilon) \cdot \delta_{a_1} + \epsilon \cdot \delta_{a_2} \quad \text{and} \quad \pi_H'(\cdot | s_H) = (1 - \epsilon) \cdot \delta_{a_1} + \epsilon \cdot \delta_{a_3}.
    \end{align}
    Let both policies take action $a_2$ on $\perp$.  Finally, let
    \begin{align}
        \nu_h(\cdot | s_h) = \delta_{a_3} \text{ for } h < H \quad \text{and} \quad \nu_H(\cdot | s_H) = \frac{\delta_{a_2} + \delta_{a_3}}{2},
    \end{align}
    and let $\nu(\cdot | \perp) = \delta_{a_2}$.  Suppose that $\epsilon = \nicefrac{1}{2 \kappa}$.  We first note that $\nu$ is $\kappa$-dominated.  Indeed, for $h < H$ this is immediate for action $a_1$, and for $a_3$,
    \begin{align}
        \nu_h(a_3 | s_h) = 1 = \kappa \cdot 2 \epsilon = \kappa \cdot \left( \pi_h(a_3 | s_h) + \pi_h'(a_3 | s_h) \right).
    \end{align}
    At $s_H$, for $a_2$ we have
    \begin{align}
        \nu_H(a_2 | s_H) = \frac 12 = \kappa \cdot  \epsilon = \kappa \cdot \left( \pi_H(a_2 | s_H) + \pi_H'(a_2 | s_H) \right),
    \end{align}
    and similarly for $a_3$.  Thus $\nu$ is $\kappa$-dominated.

    We now compute the Hellinger distance between the \emph{clean} trajectory distributions.  Let $\cE$ denote the event that the trajectory reaches $s_H$.  Note that the two laws are identical on $\cE^c$.  On the other hand,
    \begin{align}
        \pp^\pi(\cE) = \pp^{\pi'}(\cE) = (1 - \epsilon)^{H-1} \quad \text{and} \quad \dhel{\pi_H(\cdot | s_H)}{\pi_H'(\cdot | s_H)} = \epsilon
    \end{align}
    Thus,
    \begin{align}
        \dhel{\pp^\pi}{\pp^{\pi'}} = (1 - \epsilon)^{H-1} \cdot \epsilon.
    \end{align}

    We now compute the Hellinger distance between the corrupted trajectory distributions.  For $h < H$ the policies coincide and the probability of reaching $s_H$ is $(1 - \eta)^{H-1} \cdot (1 - \epsilon)^{H - 1}$.  At the final time step,
    \begin{align}
        \dhel{\pi_{\eta,H}(\cdot | s_H)}{\pi_{\eta,H}'(\cdot | s_H)} = \left( \sqrt{\nicefrac \eta 2 + (1 - \eta)\epsilon}  - \sqrt{\nicefrac \eta 2}\right)^2 \leq \frac{(1 - \eta)^2 \epsilon^2}{2 \eta}.
    \end{align}  
    Thus,
    \begin{align}
        \dhel{\pp^{\pi_\eta}}{\pp^{\pi'_\eta}} &\leq \left( (1 - \eta)(1 - \epsilon) \right)^{H - 1} \cdot \frac{(1 - \eta)^2 \cdot \epsilon^2}{2 \eta} \\
        &\leq \frac{(1 - \eta)^{H + 1} \epsilon}{2 \eta} \cdot \dhel{\pp^{\pi}}{\pp^{\pi'}} \\
        &= \frac{(1 - \eta)^{H + 1}}{4 \eta \cdot \kappa} \cdot \dhel{\pp^{\pi}}{\pp^{\pi'}}.
    \end{align}
    The result follows.
\end{proof}

Above, we see that while $\kappa$-domination prevents the corruption from hiding policy differences on actions that the policies never take, it still allows hiding proportional to how little probability the policies put on the differing actions. 

We now provide a lower bound that holds absent $\kappa$-domination.
\begin{proposition}\label{prop:offline_lb_no_kappa}
    For any $H \geq 1$ and $0 < \eta < 1$, there exists a horizon $H$ MDP with $3$ actions, policies $\pistar, \pihat$, and a corruption $\nu$ such that
    \begin{align}
        \dhel{\pp^{\pistar}}{\pp^{\pihat}} \gtrsim \sqrt{\frac{(1 - \eta)^{-H}  - 1}{1 - \eta} \cdot \dhel{\pp^{\pistar_\eta}}{\pp^{\pihat_\eta}}}.
    \end{align}
\end{proposition}
\begin{proof}
    Let $M$ have $H + 1$ states $s_1, \dots, s_H$ and $\perp$, where
    \begin{align}
        P_h(\cdot | s_h, a_2) = \delta_{s_{h+1}} \quad \text{and} \quad P_h(\cdot | s_h, a_1) = P_h(\cdot | s_h, a_3) = P_h(\cdot | \perp, a) = \delta_{\perp}.
    \end{align}
    Let $t < \nicefrac 14$, let
    \begin{align}
        S = \sum_{h = 0}^{H - 1} (1 - \eta)^{- h} = \frac{(1 - \eta) \left( (1 - \eta)^{-H} - 1 \right)}{\eta} \quad \text{and} \quad \lambda = \nicefrac tS.
    \end{align}
    Let
    \begin{align}
        u_h = \lambda ( 1 - \eta)^{- h + 1}.
    \end{align}
    Note that $u_h \leq \nicefrac 14$ for every $h$.  Now, let
    \begin{align}
        \pistar_h(\cdot | s_h) = u_h \cdot \delta_{a_1} + (1 - u_h) \cdot \delta_{a_2} \quad \text{and} \quad \pihat_h(\cdot | s_h) = u_h \cdot \delta_{a_3} + (1 - u_h) \cdot \delta_{a_2};
    \end{align}
    let $\pistar, \pihat$ agree on $\perp$ and define
    \begin{align}
        \nu_h(\cdot | s_h) = \frac 12 \cdot \delta_{a_1} + \frac 12 \cdot \delta_{a_3} \quad \text{and} \quad \nu_h(\cdot | \perp) = \delta_{a_1}.
    \end{align}
    We will let
    \begin{align}
        D_h = \dhel{\pp_h^{\pistar}(s_h)}{\pp_h^{\pihat}(s_h)} \quad \text{and} \quad D_h^\eta = \dhel{\pp_h^{\pistar_\eta}(s_h)}{\pp_h^{\pihat_\eta}(s_h)}.
    \end{align}
    Applying \Cref{lem:dh_recursion}, we see that
    \begin{align}
        D_h = u_h + (1 - u_h) \cdot D_{h + 1}.
    \end{align}
    Thus, by induction, we have that
    \begin{align}
        1 - e^{-t} \leq D_1 = 1 - \prod_{h = 1}^H (1 - u_h) \leq t.
    \end{align}
    Because $t \leq \nicefrac 14$, it thus holds that $\nicefrac t2 \leq D_1 \leq t$.  On the other hand, a direct computation shows that
    \begin{align}
        \dhel{\pistar_{\eta,h}(\cdot | s_h)}{\pihat_{\eta,h}(\cdot | s_h)} = 1 - (1 - \eta)(1 - u_h) - \sqrt{\eta^2 + 2 \eta (1 - \eta) u_h} \leq \frac{(1- \eta)^2}{2 \eta} \cdot u_h^2.
    \end{align}
    By \Cref{lem:dh_recursion} again we see that
    \begin{align}
        D_h^\eta &= 1 - (1 - \eta)(1 - u_h) - \sqrt{\eta^2 + 2 \eta (1 - \eta) u_h} + (1-  \eta)(1 - u_h) \cdot D_{h + 1}^\eta \\
        &\leq \frac{(1 - \eta)^2}{2 \eta} \cdot u_h^2 + (1 - \eta)(1 - u_h) \cdot D_{h + 1}^\eta.
    \end{align}
    Thus, by induction, we have that
    \begin{align}
        D_1^\eta \leq \frac{(1 - \eta)^2}{2 \eta} \cdot \sum_{h  =1}^H (1 - \eta)^{h - 1} \cdot u_h^2 = \frac{(1 - \eta)^2}{2 \eta} \cdot \lambda^2 S.
    \end{align}
    Substituting in the definition of $\lambda$ and $S$, we have that
    \begin{align}
        D_1^\eta \leq \frac{1 - \eta}{2 \left( (1 - \eta)^{- H} - 1 \right)} \cdot t^2 \leq \frac{2 ( 1 - \eta)}{(1 - \eta)^{-H} - 1} \cdot D_1^2.
     \end{align}
     The result follows immediately.
\end{proof}

We now prove the main lower bound.
\begin{proof}[Proof of \Cref{prop:offline_lower_bound}]
    This follows from combining \Cref{prop:offline_lb_kappa_1,prop:offline_lb_kappa,prop:offline_lb_no_kappa}.
\end{proof}

\subsection{Implications for Offline Imitation Learning}\label{app:offline_il_eps_lb}

We conclude this appendix by demonstrating that the exponential dependence on horizon in \Cref{thm:offline_bc} must appear in any offline IL algorithm through a similar construction to the one used in \Cref{app:offline_lb}.  

\begin{proposition}\label{prop:offline_il_eps_lb}
    Let $\kappa \geq 1$ and $0 \leq \eta < 1$.  For any $H \geq 2$ and
    \begin{align}
        \epsilon < \frac{\left( 1 - \nicefrac 1{2 \kappa} \right)^{H-1}}{128\kappa},
    \end{align}
    there exists an MDP, a policy class $\Pi$ of size 2, and a corruption distribution $\nu$ such that both policies in $\Pi$ $\kappa$-dominate $\nu$, but any offline IL algorithm achieving $\dhel{\pp^{\pistar}}{\pp^{\pihat}} \leq \epsilon$ requires
    \begin{align}
        n \gtrsim \frac{\eta \cdot \kappa}{(1 - \eta)^{H + 1} \cdot \epsilon}
    \end{align}
    samples from the corrupted expert.
\end{proposition}
\begin{proof}
    We use the identical construction as the proof in \Cref{prop:offline_lb_kappa}, except we suppose that
    \begin{align}
        \pi_H(\cdot | s_H) = \left( 1 - u \right) \cdot \delta_{a_2} + u \cdot \delta_{a_1}, \quad \pi_H'(\cdot | s_H) = (1 - u) \cdot \delta_{a_2} + u \cdot \delta_{a_3},
    \end{align}
    and
    \begin{align}
         \nu_H(\cdot | s_H) = \kappa \cdot u \cdot \delta_{a_1} + (1 - 2 \kappa \cdot u) \cdot \delta_{a_2} + \kappa \cdot u \cdot \delta_{a_3},
    \end{align}
    for some $u \leq \nicefrac{1}{2 \kappa}$.  This construction is clearly $\kappa$-dominated.  Moreover,
    \begin{align}
        \dhel{\pp^{\pi}}{\pp^{\pi'}} = \left( 1 - \nicefrac 1{2 \kappa} \right)^{H - 1} \cdot u \quad \text{and} \quad \dhel{\pp^{\pi_\eta}}{\pp^{\pi'_\eta}} \leq (1 - \eta)^{H - 1} \cdot \left( 1 - \nicefrac 1{2 \kappa} \right)^{H - 1} \cdot u \cdot \frac{(1 - \eta)^2}{4 \eta \kappa}.
    \end{align}
    Choosing
    \begin{align}
        u = \frac{64 \cdot \epsilon}{\left(1 - \nicefrac 1{2 \kappa}\right)^{H - 1}},
    \end{align}
    which allows $u \leq \nicefrac{1}{2 \kappa}$ by the assumption on $\epsilon$, we see that
    \begin{align}
        \dhel{\pp^\pi}{\pp^{\pi'}} = 64 \cdot \epsilon \quad \text{and} \quad \dhel{\pp^{\pi_\eta}}{\pp^{\pi'_\eta}} \leq (1 - \eta)^{H + 1} \cdot \frac{16 \cdot \epsilon}{\eta \cdot \kappa}.
    \end{align}
    A standard two point argument concludes the proof.
\end{proof}

%% file: body_clean/app_online_proofs.tex
\section{Proofs from Section \ref{sec:online}}\label{app:online_proofs}

In this appendix, we prove the results related to online Imitation Learning stated in the main body. We begin with a warm-up result that controls the trajectory Hellinger distance by the KL divergence of the noisy augmented trajectory distributions, which is sufficient to recover the horizon-free guarantee in \Cref{cor:kl_augmented}. We then show how to get a stronger guarantee in \Cref{thm:augmented_hellinger}, which also recovers \Cref{cor:kl_augmented} via Pinsker's inequality, and then show that the guarantee in \Cref{thm:augmented_hellinger} is tight up to factors in $H$ and $\eta$ in the absence of $\kappa$-domination. We continue on to prove the upper bound on \nail\ in \Cref{app:nail_proof}, concluding with a proof of the lower bound in \Cref{app:online_lb_pf}.

\subsection{Augmented Trajectory Comparisons}
\label{app:augmented_comparisons}

In this section, we prove the augmented trajectory comparisons stated in \Cref{thm:augmented_hellinger} and \Cref{cor:kl_augmented}.  We first show a warm-up result that demonstrates how to control the trajectory Hellinger distance by the KL divergence of the noisy augmented trajectory distributions, which leads to \Cref{cor:kl_augmented}. Then, we show how to get a stronger guarantee in \Cref{thm:augmented_hellinger}, which also recovers \Cref{cor:kl_augmented} via Pinsker's inequality, and finally we show that the guarantee in \Cref{thm:augmented_hellinger} is tight up to factors in $H$ and $\eta$ in the absence of $\kappa$-domination.

\subsubsection{Warm-up: the augmented KL comparison}
\label{app:warmup_augmented_kl}

We first prove a warm-up result that controls the trajectory Hellinger distance by the KL divergence of the noisy augmented trajectory distributions. This result is not tight, but it is sufficient to recover \Cref{cor:kl_augmented} and is much simpler to prove than the stronger guarantee in \Cref{thm:augmented_hellinger}. 

The key idea is to use the \emph{subadditivity} of the Hellinger distance from \citet{foster2021statistical,foster2024online} to control the trajectory Hellinger distance by the sum of the per-step Hellinger distances, and then demonstrate that these per-step Hellinger distances can be controlled by the KL divergence of the noisy augmented trajectory distributions.  Critically, we then have to use KL divergence instead of Hellinger distance in the upper bound in order to stitch the per-step distances back into a divergence over the \emph{trajectory distributions} because KL divergence satisfies the \emph{chain rule} (cf. \Cref{prop:kl_properties}) while Hellinger distance does not.

We first state a slightly tighter version of \Cref{cor:kl_augmented}, albeit with a more complicated expression in the upper bound.  We break this into two separate results, one using forward KL and one using reverse KL, which together imply \Cref{cor:kl_augmented}.  First, under $\kappa$-domination, we have the following result.
\begin{proposition}\label{prop:kl_augmented_tight}
    Let $\pi, \pi'$ be any two policies that $\kappa$-dominate the corruption $\nu$.  Then, for any $0 \leq \eta < 1$ it holds that
    \begin{align}
        \dhel{\pp^{\pi}}{\pp^{\pi'}} \leq \frac{2 (1 + \eta (2 \kappa - 1))}{(1 - \eta)^2} \cdot \left(\kld{\pp^{\pi', \pi_\eta}}{\pp^{\pi', \pi_\eta'}} \wedge  \kld{\pp^{\pi', \pi_\eta'}}{\pp^{\pi', \pi_\eta}}\right).
    \end{align}
\end{proposition}
Note that the first conclusion of \Cref{cor:kl_augmented} can be recovered immediately.
\begin{proof}[Proof of \Cref{prop:kl_augmented_tight}]
    We make use of the subadditivity of Hellinger squared divergence  (\Cref{prop:hellinger_subadditive}).  Indeed, it holds that
    \begin{align}
        \dhel{\pp^{\pi}}{\pp^{\pi'}} &\leq \ee^{\pi'}\left[ \sum_{h  = 1}^H \dhel{\pi(\cdot | s_h)}{\pi'(\cdot | s_h)} \right].
    \end{align}
    Now, applying \Cref{lem:dhel_contraction_kappa}, we have that for any $s_h$,
    \begin{align}
        \dhel{\pi(\cdot | s_h)}{\pi'(\cdot | s_h)} \leq \frac{2(1 + \eta (2 \kappa - 1))}{(1 - \eta)^2} \cdot \dhel{\pi_\eta(\cdot | s_h)}{\pi_\eta'(\cdot | s_h)}.
    \end{align}
    Combining this fact with the preceding display and applying \Cref{prop:pinsker}, we have
    \begin{align}
        \dhel{\pp^{\pi}}{\pp^{\pi'}} &\leq \frac{2(1 + \eta (2 \kappa - 1))}{(1 - \eta)^2} \cdot \ee^{\pi'}\left[ \sum_{h  =1}^H \dhel{\pi_{\eta,h}(\cdot | s_h)}{\pi_{\eta,h}'(\cdot | s_h)} \right] \label{eq:dhel_subadditive} \\
        &\leq \frac{2(1 + \eta (2 \kappa - 1))}{(1 - \eta)^2} \cdot \ee^{\pi'}\left[ \sum_{h  =1}^H \kld{\pi_{\eta,h}(\cdot | s_h)}{\pi_{\eta,h}'(\cdot | s_h)} \right].
    \end{align} 
    We may now apply the Chain rule for KL divergence (\Cref{prop:kl_properties}) to observe that
    \begin{align}
        \ee^{\pi'}\left[ \sum_{h  =1}^H \kld{\pi_{\eta,h}(\cdot | s_h)}{\pi_{\eta,h}'(\cdot | s_h)} \right] = \kld{\pp^{\pi', \pi_\eta}}{\pp^{\pi',\pi_\eta'}}.
    \end{align}
    The first result follows.  The second follows by the same argument by observing that the Hellinger distance is symmetric in \eqref{eq:dhel_subadditive}.
\end{proof}

We now provide a similar comparison between the clean and noisy trajectory distributions absent $\kappa$-domination.

\begin{proposition}\label{prop:kl_hellinger_slow}
    Let $\pi, \pi'$ be two policies and let $\nu$ be an arbitrary corruption distribution.  Then for any $0 \leq \eta < 1$ it holds that
    \begin{align}
        \dhel{\pp^{\pi}}{\pp^{\pi'}} \leq \frac 1{1 - \eta} \cdot \sqrt{2 H \cdot \kld{\pp^{\pi', \pi_\eta}}{\pp^{\pi', \pi'_\eta}} \wedge \kld{\pp^{\pi', \pi'_\eta}}{\pp^{\pi', \pi_\eta}}}.
    \end{align}
\end{proposition}
\begin{proof}
    As in the proof of \Cref{cor:kl_augmented}, we use the subadditivity of the Hellinger distance (\Cref{prop:hellinger_subadditive}) to get that
    \begin{align}
        \dhel{\pp^{\pi}}{\pp^{\pi'}} 
        &\leq \ee^{\pi'}\left[ \sum_{h  =1}^H \dhel{\pi_h(\cdot | s_h)}{\pi'_h(\cdot | s_h)} \right] \\
        &\leq \frac {\sqrt{2}}{1 - \eta} \cdot \ee^{\pi'}\left[ \sum_{h  =1}^H \sqrt{\dhel{\pi_{\eta,h}(\cdot | s_h)}{\pi'_{\eta,h}(\cdot | s_h)}} \right] \\
        &\leq \frac{\sqrt{2H}}{1 - \eta} \cdot \sqrt{\ee^{\pi'}\left[ \sum_{h  =1}^H \dhel{\pi_{\eta,h}(\cdot | s_h)}{\pi'_{\eta,h}(\cdot | s_h)} \right]} \\
        &\leq \frac{\sqrt{2H}}{1 - \eta} \cdot \sqrt{\kld{\pp^{\pi', \pi_\eta}}{\pp^{\pi', \pi'_\eta}}},
    \end{align}
    where the second inequality comes from \Cref{lem:hellinger_contraction_arbitrary}, the third inequality comes from Cauchy-Schwarz, and the final inequality comes from applying Pinsker's inequality (\Cref{prop:pinsker}) and the chain rule for KL divergence (\Cref{prop:kl_properties}).  The first claimed bound follows. The second comes from the identical argument but using the symmetry of the Hellinger distance to apply Pinsker's inequality in the reverse direction. The result follows.
\end{proof}
Finally, we note that \Cref{cor:kl_augmented} follows immediately from \Cref{prop:kl_augmented_tight} and \Cref{prop:kl_hellinger_slow}, concluding the proof of this result.

\subsubsection{A stronger augmented Hellinger comparison}

We now demonstrate how to achieve tighter control on the trajectory Hellinger distance by the augmented Hellinger distance, which will allow us to recover \Cref{thm:augmented_hellinger} and \Cref{cor:kl_augmented} via Pinsker's inequality. 

Looking closely at the proofs of these previous results, the main source of looseness is the use of the subadditivity of Hellinger distance to control the trajectory Hellinger distance by the sum of the per-step Hellinger distances, which then forces us to use KL divergence to stitch these per-step distances back into a divergence over the trajectory distributions. In this section, we show how to avoid this looseness by directly comparing the clean and noisy augmented Hellinger distances. We present an algorithmic procedure that leverages this result to achieve a horizon-free guarantee (under $\kappa$-domination) in \Cref{app:horizon_free_testing}. 

The proof splits into two main parts. We start with \Cref{lem:sharp_clean_augmented_hellinger_comparison}, which shows that even without corruption, the clean Hellinger distance between two policies is bounded by the Hellinger distance between their augmented counterparts, with a universal constant. We then compare the clean and noisy augmented Hellinger distances; chaining these two together gives the desired result.

At first glance, the bound in \Cref{lem:sharp_clean_augmented_hellinger_comparison} below may seem a bit surprising: the left-hand side compares the clean rollout laws of $\pi$ and $\pi'$, while the right-hand side only compares the labels produced by $\pi$ and $\pi'$ on states visited by rolling out $\pi'$. The reason such a comparison is possible is that any change in the clean trajectory law must be caused by action-distribution mismatch at some earlier state. The augmented process measures exactly this mismatch, but in a decoupled way: the rollout action determines the next state, while an independent auxiliary action is used to compare the two policies at the current state.

\begin{lemma}
\label{lem:sharp_clean_augmented_hellinger_comparison}
For any two policies $\pi,\pi'$, it holds that
\begin{align}
\dhel{\pp^\pi}{\pp^{\pi'}}
\le
\nicefrac92 \cdot 
\dhel{\pp^{\pi',\pi}}{\pp^{\pi',\pi'}}.
\end{align}
\end{lemma}

\begin{proof}[Proof of \Cref{lem:sharp_clean_augmented_hellinger_comparison}]

For a state $s$ at time $h$, define the standard and augmented suffix affinities (cf. \Cref{prop:hellinger_decomposition})
\begin{align} 
    A_h(s) = 1 - \dhel{\pp_h^{\pi}(s)}{\pp_h^{\pi'}(s)}, \qquad B_h(s) = 1 - \dhel{\pp_h^{\pi',\pi}(s)}{\pp_h^{\pi',\pi'}(s)}. 
\end{align} 
The proof is a backward induction comparing these two affinities through a nonlinear envelope. More precisely, we claim that for every $h$ and $s$, 
\begin{align} 
    \label{eq:clean_augmented_inductive_claim} 
    A_h(s)\ge F(B_h(s)), \end{align} 
for a carefully chosen function $F$. Using properties of $F$, we will then convert this affinity comparison back into the desired Hellinger bound. At the terminal time $H+1$, both affinities are
equal to $1$, so the claim is immediate. 

Assume the claim holds at time $h+1$. Condition on $s_h=s$, and let the
rollout action and next state be sampled according to $a_h\sim \pi'_h(\cdot\mid s)$ and $s_{h+1}\sim\pp_h(\cdot\mid s,a_h)$. Define the likelihood ratio $R_h
=
\sqrt{
\nicefrac{\pi_h(a_h\mid s)}
{\pi'_h(a_h\mid s)}}$,
with an arbitrary value on actions outside the support of $\pi'_h(\cdot\mid s)$.

We first observe that through a simple change of measure, the Hellinger affinity between the suffix distributions generated by $\pi_h$
and $\pi'_{h}$ from state $s$ can be written as
\begin{align}
\label{eq:correlated_product}
A_h(s)
&=
\sum_{a}\pi'_{h}(a\mid s)\,
\sqrt{\frac{\pi_{h}(a\mid s)}{\pi'_{h}(a\mid s)}}\;
\ee_{s'\sim\pp_h(\cdot\mid s,a)}\left[1-\dhel{\pp_{h}^{\pi}(s')}{\pp_{h}^{\pi'}(s')}\right]
\\
&=
\ee_{\pi'_{h}}
\left[
R_h A_{h+1}(s_{h+1})
\mid s_h=s
\right].
\end{align}
Second, we observe that in the augmented process, the learner rolls out $\pi'$ in the
environment, and the auxiliary labels are drawn independently from $\pi$
(resp.\ $\pi'$) at the visited states. Conditional on $s_{h}=s$, their pointwise affinity is
\begin{align}
\ee_{\pi'_{h}}[R_{h}\mid s]=\sum_{a}\sqrt{\pi_{h}(a\mid s)\pi'_{h}(a\mid s)}.
\end{align}
The future state $s_{h+1}$ and the subsequent labels are generated
independently of the current label, because the rollout action that determines
$s_{h+1}$ is drawn from $\pi'_{h}(\cdot\mid s)$ and the label is an
independent draw from the teacher policy.  Hence the affinity of the suffix is $\ee_{\pi'_{h}}[B_{h+1}(s_{h+1})\mid s]$. Because these two parts are independent, the overall suffix affinity multiplies:
\begin{align}
\label{eq:uncorrelated_product}
B_h(s)
=
\ee_{\pi'_{h}}[R_h\mid s_h=s] \cdot 
\ee_{\pi'_{h}}[B_{h+1}(s_{h+1})\mid s_h=s].
\end{align}
From the above analysis, we see that the clean and augmented affinities obey different one-step recursions. The clean recursion contains a correlated product: the same action both contributes the square-root likelihood ratio for changing from $\pi'_{h}$ to $\pi_{h}$, and determines the next state. The augmented recursion decouples these roles: the rollout action determines the next state, while an independent auxiliary label contributes the current action affinity.

So, in order to prove the inductive step, we must lower bound the correlated product from the standard affinity by a function of an uncorrelated product from the augmented affinity. Invoking the inductive hypothesis would then yield the chain
\begin{align}
A_h &= \ee_{\pi'_{h}} \left[ R_h A_{h+1}(s_{h+1}) \mid s_h=s \right]\\
    &\geq 
    \ee_{\pi'_{h}}
\left[
R_h F(B_{h+1}(s_{h+1}))
\mid s_h=s
\right] \\
\label{eq:rhs_of_nonlinear_result}
&\geq
F\left(
\ee_{\pi'_{h}}[R_h\mid s_h=s]\,
\ee_{\pi'_{h}}[B_{h+1}(s_{h+1})\mid s_h=s]
\right) \\
&= F(B_h(s)),
\end{align}
closing the induction. It remains to choose an $F$ such that the second inequality holds.

Write \(m=\ee_{\pi'_{h}}[R]\) and let $T$ be a thresholded function of $X$. First, define the quantity
\begin{align}
1-m
=
1-\sum_a \sqrt{\pi_h(a\mid s)\pi'_h(a\mid s)}
=
\dhel{\pi_h(\cdot\mid s)}{\pi'_h(\cdot\mid s)},
\end{align}
which is exactly the current one-step Hellinger defect. It also controls the size of
the current change-of-measure fluctuation:
\begin{align}
\ee_{\pi'_{h}}[(R-1)^2]
=
\ee_{\pi'_{h}}[R^2]-2\ee_{\pi'_{h}}[R]+1
\le
2(1-m).
\end{align}
Thus, it is easy to see that any possible anticorrelation between the likelihood-ratio factor \(R\) and
the future similarity must be paid for by the current one-step defect \(1-m\).\footnote{A concrete example demonstrating the necessity of a defect is given in the remark following the proof.}

However, a linear transformed moment of $T$ is still too weak. Even if we allow an
additive defect, an inequality of the form
\begin{align}
\ee_{\pi'_{h}}[RT]\ge m\ee_{\pi'_{h}}[T]-c(1-m)
\end{align}
cannot hold uniformly for a constant \(c>0\), which motivates using a nonlinear moment.\footnote{To see this, take \(m=1-\delta\),
let \(R=m-s\) and \(R=m+s\) with equal probability, where
\(s=\sqrt{1-m^2}\), and let \(T=1\) on the lower value of \(R\) and \(T=0\) on
the higher value. Then \(\ee_{\pi'_{h}}[R]=m\), \(\ee_{\pi'_{h}}[R^2]=1\), and
\(
m\ee_{\pi'_{h}}[T]-\ee_{\pi'_{h}}[RT]=\nicefrac{s}{2}\asymp \sqrt{1-m},
\)
which is much larger than \(1-m\). Thus a linear moment would lose a
\(\sqrt{1-m}\) term, not the one-step Hellinger defect \(1-m\).} \Cref{lem:sharp_scalar_defective_holder} shows that
after replacing \(T\) by the more conservative moment \(T^{3/2}\), the loss from
arbitrary anticorrelation is only linear in the one-step defect:
\begin{align}
\left(\ee_{\pi'_{h}}[R T^{\nicefrac{3}{2}}]\right)^{2/3}
\ge
m\ee_{\pi'_{h}}[T]-2(1-m).
\end{align}
The \(\nicefrac{3}{2}\) power is the moment that makes the worst-case
anti-correlation problem tractable; while linear $T$ is too sensitive to anti-correlation, $T^{\nicefrac{3}{2}}$ discounts moderate similarity enough that the worst anti-correlated signal can be controlled by the one-step Hellinger defect $1-m$. 

We must now choose the threshold of $F$ such that the additive defect
\(2(1-m)\) is absorbed by the product form in the augmented recursion. Let
\begin{align}
T(x)=(3x-2)_+.
\end{align}
Then \(T(x)\geq 3x-2\), and hence
\begin{align}
m\ee_{\pi'_{h}}[T(X)]-2(1-m) \geq 3m\ee_{\pi'_{h}}[X]-2.
\end{align}
The right-hand side is exactly the affine part of \(T(m\ee_{\pi'_{h}}[X])\). Therefore \Cref{lem:sharp_scalar_defective_holder} implies
\begin{align}
\ee_{\pi'_{h}}[R T(X)^{3/2}]
\ge
\left(3m\ee_{\pi'_{h}}[X]-2\right)_+^{3/2}
=
T(m\ee_{\pi'_{h}}[X])^{3/2}.
\end{align}
Thus, with
\begin{align}
F(x)=T(x)^{3/2}=(3x-2)_+^{3/2},
\end{align}
we obtain the desired nonlinear closure property
\begin{align}
\ee_{\pi'_{h}}[R F(X)]\ge F(\ee_{\pi'_{h}}[R]\ee_{\pi'_{h}}[X]).
\end{align}
Thus, the threshold has two roles: it is qualitatively necessary because no
positive low-affinity certificate can hold uniformly, and its precise location
is chosen so that the one-step Hellinger defect from the scalar lemma is
absorbed algebraically rather than accumulating over the horizon. The
\(\nicefrac32\) power is the nonlinear moment that makes this absorption possible
with a loss of order \(1-m\), rather than \(\sqrt{1-m}\).

Under \(a_h\sim\pi'_h(\cdot\mid s)\), we have
\(\ee_{\pi'_{h}}[R_h^2\mid s_h=s]\le 1\) and
\(B_{h+1}(s_{h+1})\in[0,1]\). Applying
\Cref{cor:scalar_nonlinear_envelope} conditionally with
\(R=R_h\), \(X=B_{h+1}(s_{h+1})\), and \(F(x)=(3x-2)_+^{\nicefrac32}\), we obtain
\begin{align}
\ee_{\pi'_{h}}
\left[
R_h F(B_{h+1}(s_{h+1}))
\mid s_h=s
\right]
\ge
F\left(
\ee_{\pi'_{h}}[R_h\mid s_h=s]\,
\ee_{\pi'_{h}}[B_{h+1}(s_{h+1})\mid s_h=s]
\right).
\end{align}
Combining this with \eqref{eq:uncorrelated_product} gives
\(A_h(s)\ge F(B_h(s))\), completing the induction.

Finally, we can convert the pointwise affinity comparison into the desired
trajectory-level Hellinger comparison. Let \(s_1\sim\rho\). Since the clean
trajectory laws share the same initial state distribution,
\begin{align}
1-\dhel{\pp^\pi}{\pp^{\pi'}}
=
\ee[A_1(s_1)]
\qquad \text{and}\qquad
1-\dhel{\pp^{\pi',\pi}}{\pp^{\pi',\pi'}}
=
\ee_{\pi'_{1}}[B_1(s_1)]
.
\end{align}
By the induction and convexity of \(F\), we can use Jensen's inequality to write
\begin{align}
1-\dhel{\pp^\pi}{\pp^{\pi'}}
=
\ee[A_1(s_1)]
\ge
\ee_{\pi'_{1}}[F(B_1(s_1))]
\ge
F(\ee_{\pi'_{1}}[B_1(s_1)])
=
F\left(1-\dhel{\pp^{\pi',\pi}}{\pp^{\pi',\pi'}}\right).
\end{align}
Writing $d=\dhelinline{\pp^{\pi',\pi}}{\pp^{\pi',\pi'}}$,
this gives
\begin{align}
\dhel{\pp^\pi}{\pp^{\pi'}}
\le
1-F(1-d)
=
1-(1-3d)_+^{3/2}.
\end{align}
If \(d\ge \nicefrac13\), this bound is already trivial since the right-hand side equals
\(1\). If \(d<\nicefrac13\), then
\begin{align}
1-(1-3d)^{\nicefrac32}
\le
\frac92 d,
\end{align}
because the derivative of \(d\mapsto 1-(1-3d)^{\nicefrac32}\) is
\(\nicefrac92\sqrt{1-3d}\le \nicefrac92\). Therefore, for all \(d\in[0,1]\),
\begin{align}
\dhel{\pp^\pi}{\pp^{\pi'}}
\le
\frac92
\dhel{\pp^{\pi',\pi}}{\pp^{\pi',\pi'}}.
\end{align}

To summarize, we design \(F\) so that two critical properties hold. First, \(F\)
satisfies the one-step nonlinear closure property
\begin{align}
\ee_{\pi'_{1}}[R F(X)]\ge F(\ee_{\pi'_{1}}[R]\ee_{\pi'_{1}}[X])
\end{align}
for every \(R\ge0\) with \(\ee_{\pi'_{1}}[R^2]\le1\) and every \(X\in[0,1]\). This is the
property that closes the backward induction, replacing the false linear
positive-correlation inequality. The threshold in \(F\) is necessary because
low augmented product affinity can be generated entirely by anticorrelation, and
the higher moment is what allows the scalar comparison to lose only the
one-step Hellinger defect. Second, \(F\) converts the resulting affinity
comparison into the desired Hellinger comparison at the initial distribution. 
Thus the same envelope both propagates the induction locally and turns the final
augmented affinity bound into a linear bound on Hellinger distance.

It is important to note that the choice of \(F\) is not intended to be optimal.
Any function satisfying similar properties would
yield an analogous comparison. Improving the constants amounts to
sharpening 
\Cref{lem:sharp_scalar_defective_holder} and then choosing the corresponding
envelope.
\end{proof}

\begin{remark}
If we naively tried to take a linear envelope, then the induction would require
\begin{align}
\ee_{\pi'_{h}}[RX]\gtrsim \ee_{\pi'_{h}}[R]\ee_{\pi'_{h}}[X]
\end{align}
for every \(R\ge0\) with \(\ee_{\pi'_{h}}[R^2]\le1\) and every \(X\in[0,1]\),
which is false under worst-case anticorrelation. In the rollout interpretation,
this failure corresponds to the case where the actions favored by the change of
measure are precisely those that lead to states with low future similarity.

Moreover, this failure cannot be repaired by simply choosing an envelope $F$ that is positive at small values of the augmented product affinity.\footnote{Indeed, let
\(\cE\) be an event of probability \(p\), set
\(R=p^{-1/2}\ind{\cE}\), and set \(X=\ind{\cE^c}\). Then
\(\ee_{\pi'_{h}}[R^2]=1\) and
\(
\ee_{\pi'_{h}}[R]\ee_{\pi'_{h}}[X]=\sqrt p(1-p)>0,
\)
but \(R\) is supported entirely on the region where \(X=0\). Since any admissible
envelope must have \(F(0)=0\), we get
\(
\ee_{\pi'_{h}}[R F(X)]=0\). Thus, the desired scalar closure property
$
\ee_{\pi'_{h}}[R F(X)]\ge F(\ee_{\pi'_{h}}[R]\ee_{\pi'_{h}}[X])
$
forces \(F(\sqrt p(1-p))=0\).}
In particular, any uniformly valid envelope must
vanish on a nontrivial low-affinity interval. This motivates using a thresholded function: below the threshold, augmented product affinity can
be generated entirely by anticorrelation and therefore cannot certify any clean
affinity.
\end{remark}

As we will see next, under $\kappa$-domination, we can compare the standard and augmented Hellinger distances at each step via \Cref{lem:dhel_contraction_kappa}, which we then lift to the whole trajectory using a product comparison argument and careful algebraic manipulation. The key is that the per-step Hellinger distances are small, which allows us to leverage the product structure of the trajectory distributions to obtain a tighter comparison than the naive union bound suggested by \Cref{prop:hellinger_decomposition}, which would yield a factor of $H$ rather than a constant. 

\begin{lemma}
\label{lem:product_comparison}
Let $L\ge1$. If $d_1,\ldots,d_H,e_1,\ldots,e_H\in[0,1]$ satisfy
$d_h\le L e_h$ for every $h$, then
\begin{align}
1-\prod_{h=1}^H(1-d_h)
\le
L\left(1-\prod_{h=1}^H(1-e_h)\right).
\end{align}
\end{lemma}
\begin{proof}
It suffices to consider the largest possible $d_h$, namely
$d_h=\min\{Le_h,1\}$. If $e_h\ge 1/L$ for some $h$, then the left-hand side
is at most $1$, while the right-hand side is at least
\begin{align}
L(1-(1-e_h))=Le_h\ge1.
\end{align}
Thus assume $e_h<1/L$ for all $h$. Let $x_h=1-e_h$. Then
\begin{align}
1-d_h=1-Le_h=Lx_h-(L-1).
\end{align}
Define $T(x)=Lx-(L-1)$. Observe that for $x,y\in[1-1/L,1]$, we have the property that
\begin{align}
T(x)T(y)-T(xy)
=
L(L-1)(1-x)(1-y)\ge0.
\end{align}
By induction on $H$, this implies that for every $x_1,\ldots,x_H\in[1-1/L,1]$,
\begin{align}
\prod_{h=1}^H T(x_h)
\ge
T\left(\prod_{h=1}^H x_h\right).
\end{align}
Applying this with $x_h=1-e_h$ gives
\begin{align}
\prod_{h=1}^H(1-d_h)
\ge
L\prod_{h=1}^H(1-e_h)-(L-1).
\end{align}
Rearranging gives the desired inequality.
\end{proof}

Combining the above results, we are now ready to state the main comparison between the standard and augmented Hellinger distances under $\kappa$-domination.

\begin{proposition}
\label{prop:augmented_hellinger_testing_comparison}
Let $\pi, \pi'$ be two policies that $\kappa$-dominate the corruption $\nu$. Then for any $0 \leq \eta < 1$ it holds that
\begin{align}
\dhel{\pp^{\pi}}{\pp^{\pi'}}
\lesssim
\frac{1+\eta(2\kappa-1)}{(1-\eta)^2}
\dhel{\pp^{\pi',\pi_\eta}}{\pp^{\pi',\pi'_\eta}}.
\end{align}
\end{proposition}

\begin{proof}
\Cref{lem:sharp_clean_augmented_hellinger_comparison} gives the initial comparison between the standard and augmented Hellinger distances:
\begin{align}
\dhel{\pp^{\pi}}{\pp^{\pi'}}
\le
\frac{31}{2}
\dhel{\pp^{\pi',\pi}}{\pp^{\pi',\pi'}}.
\end{align}
It remains to compare the standard and noisy augmented Hellinger distances. Fix a realized state sequence from a $\pi'$-rollout. For each $h$, let $d_h = \dhel{\pi_h(\cdot\mid s_h)}{\pi'_h(\cdot\mid s_h)}$ and let $e_h = 
\dhel{\pi_{\eta,h}(\cdot\mid s_h)}{\pi'_{\eta,h}(\cdot\mid s_h)}$.
\Cref{lem:dhel_contraction_kappa} provides the pointwise bound
\begin{align}
d_h
\le
\frac{2(1+\eta(2\kappa-1))}{(1-\eta)^2} e_h.
\end{align}
Since $\frac{2(1+\eta(2\kappa-1))}{(1-\eta)^2}\ge1$, \Cref{lem:product_comparison} gives the product comparison
\begin{align}
1-\prod_{h=1}^H(1-d_h)
\le
\frac{2(1+\eta(2\kappa-1))}{(1-\eta)^2}
\left(1-\prod_{h=1}^H(1-e_h)\right).
\end{align}
By \Cref{prop:hellinger_decomposition}, the left-hand side corresponds to the standard trajectory-level augmented Hellinger and the right-hand side corresponds to the noisy version. Averaging over the $\pi'$-rollout state sequence yields
\begin{align}
\dhel{\pp^{\pi',\pi}}{\pp^{\pi',\pi'}}
\le
\frac{2(1+\eta(2\kappa-1))}{(1-\eta)^2}
\dhel{\pp^{\pi',\pi_\eta}}{\pp^{\pi',\pi'_\eta}}.
\end{align}
Combining the two displays concludes the proof.
\end{proof}

We now provide a similar comparison between the standard and augmented Hellinger distances absent $\kappa$-domination. First, we give the following product comparison lemma, which allows us to compare the standard and noisy augmented Hellinger distances without $\kappa$-domination. We then use this lemma to give an analogous version of \Cref{prop:augmented_hellinger_testing_comparison}, albeit with a worse dependence on $H$ and $\eta$.

\begin{lemma}
\label{lem:sqrt_product_comparison}
Let $L\ge1$. If $d_1,\ldots,d_H,e_1,\ldots,e_H\in[0,1]$ satisfy $d_h\le L\sqrt{e_h}$ for every $h$, then
\begin{align}
1-\prod_{h=1}^H(1-d_h)
\le
L\sqrt{
2H\left(1-\prod_{h=1}^H(1-e_h)\right)
}.
\end{align}
\end{lemma}

\begin{proof}
By the Weierstrass product inequality and the assumed pointwise bound, we have
\begin{align}
1-\prod_{h=1}^H(1-d_h)
\le
\sum_{h=1}^H d_h
\le
L\sum_{h=1}^H\sqrt{e_h}
\le
L\sqrt{H \cdot \left(\sum_{h=1}^H e_h\right)},
\end{align}
where we used Cauchy-Schwarz in the last step. We now have an upper bound of $L\sqrt{H \cdot \left(\sum_{h=1}^H e_h\right)}$ on the left-hand side, and need to compare it to $L\sqrt{2H \cdot \left(1-\prod_{h=1}^H(1-e_h)\right)}$.

Observe that from an elementary bound, we have $1-\prod_{h=1}^H(1-e_h)\ge 1-e^{-\sum_{h=1}^H e_h}$. We split into two cases.
If $\sum_{h=1}^H e_h\le1$, then $1-e^{-\sum_{h=1}^H e_h}\ge \nicefrac{\sum_{h=1}^H e_h}{2}$, and therefore
\begin{align}
L\sqrt{H \cdot \left(\sum_{h=1}^H e_h\right)}
\le
L\sqrt{2H \cdot \left(1-\prod_{h=1}^H(1-e_h)\right)}.
\end{align}
If $\sum_{h=1}^H e_h>1$, then $1-\prod_{h=1}^H(1-e_h)\ge 1-e^{-1}\ge \nicefrac12$. Hence
\begin{align}
L\sqrt{2H \cdot \left(1-\prod_{h=1}^H(1-e_h)\right)}\ge L\sqrt{H}\ge1,
\end{align}
while the left-hand side is always at most $1$. Combining the two cases proves
the claim.
\end{proof}

Next, we use the above product comparison to give a version of \Cref{prop:augmented_hellinger_testing_comparison} without $\kappa$-domination. The key is that, even without $\kappa$-domination, we can still control the per-step clean Hellinger distances by the square root of the noisy augmented Hellinger distances via \Cref{lem:hellinger_contraction_arbitrary}, which allows us to apply \Cref{lem:sqrt_product_comparison} to compare the clean and noisy augmented trajectory Hellinger distances.

\begin{proposition}
\label{prop:augmented_hellinger_testing_comparison_no_domination}
Let $\pi, \pi'$ be two policies and let $\nu$ be an arbitrary corruption distribution.  Then for any $0 \leq \eta < 1$ it holds that
\begin{align}
\dhel{\pp^{\pi}}{\pp^{\pi'}}
\lesssim
\frac{1}{1-\eta}
\sqrt{H \cdot
\dhel{\pp^{\pi',\pi_\eta}}{\pp^{\pi',\pi'_\eta}}
}.
\end{align}
\end{proposition}

\begin{proof}
We proceed in a similar manner to the proof of \Cref{prop:augmented_hellinger_testing_comparison}, but without using \Cref{lem:dhel_contraction_kappa} in the absence of $\kappa$-domination. By \Cref{lem:sharp_clean_augmented_hellinger_comparison},
\begin{align}
\dhel{\pp^{\pi}}{\pp^{\pi'}}
\le
\frac{31}{2}
\dhel{\pp^{\pi',\pi}}{\pp^{\pi',\pi'}}.
\end{align}
It remains to compare the clean and noisy augmented Hellinger distances. Fix a realized state sequence from a $\pi'$-rollout. For each $h$, let $d_h = \dhel{\pi_h(\cdot\mid s_h)}{\pi'_h(\cdot\mid s_h)}$ and let $e_h = 
\dhel{\pi_{\eta,h}(\cdot\mid s_h)}{\pi'_{\eta,h}(\cdot\mid s_h)}$. \Cref{lem:hellinger_contraction_arbitrary} provides the pointwise bound 
\begin{align}
d_h
\le
\frac{\sqrt2}{1-\eta}\sqrt{e_h}.
\end{align} 
Applying \Cref{lem:sqrt_product_comparison} gives, for this fixed state sequence,
\begin{align}
1-\prod_{h=1}^H(1-d_h)
\le
\frac{2}{1-\eta}
\sqrt{ H \cdot \left(
1-\prod_{h=1}^H(1-e_h)
\right)
}.
\end{align}
Averaging over the $\pi'$-rollout state sequence yields, by Jensen's inequality,
\begin{align}
\dhel{\pp^{\pi',\pi}}{\pp^{\pi',\pi'}}
\le
\frac{2}{1-\eta}
\sqrt{ H \cdot 
\dhel{\pp^{\pi',\pi_\eta}}{\pp^{\pi',\pi'_\eta}}
}.
\end{align}
Combining the two displays concludes the proof.
\end{proof}

The following lemma is the scalar inequality that drives the horizon-free comparison in \Cref{lem:sharp_clean_augmented_hellinger_comparison}.

\begin{lemma}
\label{lem:sharp_scalar_defective_holder}
Let $R\ge 0$ be a random variable with $\ee[R^2]\le 1$. Then for every random variable
$W\in[0,1]$, 
\begin{align}
\left(\ee[R W^3]\right)^{2/3}
\ge
\ee[R]\ee[W^2]-2(1-\ee[R]).
\end{align}
\end{lemma}

\begin{proof}
Let $m=\ee[R]$. The case $m=0$ is immediate, so assume $m>0$. We first reduce to threshold
functions. For fixed law of $R$, consider minimizing
\begin{align}
\left(\ee[R W^3]\right)^{2/3}-m\ee[W^2]
\end{align}
over measurable $W\in[0,1]$. We will see that the defect $2(1-\ee[R])$ is naturally controlled by the second moment bound on $R$. Solving the KKT condition suggests that it suffices
to prove the claim for
\begin{align}
W_c(r)=\min\left\{1,\frac{c}{r}\right\},
\end{align}
with the convention $W_c(0)=1$. If $R$ has atoms, the minimizer may randomize
on the threshold set $\{R=c\}$; this is immaterial by continuity.

For $c\ge0$, define
\begin{align}
A(c)=\ee[W_c(R)^2], \qquad \text{and} \qquad 
B(c)=\ee[R W_c(R)^3].
\end{align}
It is enough to prove
\begin{align}
\Phi(c):=B(c)^{2/3}-mA(c)+2(1-m)\ge0
\end{align}
for all $c\ge0$. At the boundaries this is immediate. Indeed, as
$c\to\infty$, $W_c\to1$, so
\begin{align}
\Phi(c)\to m^{2/3}-m+2(1-m)\ge0.
\end{align}
At $c=0$, we have
\begin{align}
\Phi(0)\ge -m\pp(R=0)+2(1-m)\ge 2(1-m)^2\ge0,
\end{align}
because on the event $\{R=0\}$, we have $(1-R)^2=1$, and thus recalling $m \le 1$,
\begin{align}
\pp(R=0)\le \ee[(1-R)^2]\le 2(1-m).
\end{align}

It remains to consider interior minima. Away from atoms of $R$, let $I(c)=\ee[R^{-2}\mathbf 1\{R>c\}]$. Then
\begin{align}
A'(c)=2cI(c),
\qquad \text{and} \qquad 
B'(c)=3c^2I(c),
\end{align}
and hence, by the chain rule,
\begin{align}
\Phi'(c)
=
2cI(c)\left(cB(c)^{-1/3}-m\right).
\end{align}
Moreover, $cB(c)^{-1/3}$ is nondecreasing, since
\begin{align}
\frac{d}{dc}\log\left(cB(c)^{-1/3}\right)
=
\frac{B(c)-c^3I(c)}{cB(c)}
=
\frac{\ee[R\mathbf 1\{R\le c\}]}{cB(c)}
\ge0.
\end{align}
Therefore $\Phi$ has at most one interior minimum, and the criticality condition for an interior minimum is
\begin{align}
B(c)=\frac{c^3}{m^3}.
\label{eq:sharp_scalar_criticality}
\end{align}
At such a point,
\begin{align}
B(c)=cA(c)-\ee[(c-R)_+],
\end{align}
so, substituting into the definition of $\Phi$ and using \eqref{eq:sharp_scalar_criticality}, we have
\begin{align}
\Phi(c)
=
2(1-m)-\frac{m}{c}\ee[(c-R)_+].
\end{align}
Thus, at the critical point, $\Phi(c)\ge0$ is equivalent to the tail bound
\begin{align}
\frac{m}{c}\ee[(c-R)_+]\le 2(1-m).
\label{eq:sharp_critical_tail}
\end{align}
It remains to verify \eqref{eq:sharp_critical_tail}. We split into two cases. First suppose $c\le 1/2$. For every $r\le c$,
\begin{align}
\frac{m}{c}(c-r)\le \frac{c-r}{c}\le (1-r)^2.
\end{align}
Therefore
\begin{align}
\frac{m}{c}\ee[(c-R)_+]
\le
\ee[(1-R)^2]
=
\ee[R^2] - 2\ee[R] + 1
\le
2(1-m),
\end{align}
where the final inequality follows from the second moment bound $\ee[R^2]\le1$,
proving \eqref{eq:sharp_critical_tail} in this case.

Now suppose $c>1/2$. At a critical point, since $B(c)\le \ee[R]=m$, we have
\begin{align}
\frac{c^3}{m^3}\le m,
\qquad\text{hence}\qquad
m\ge c^{3/4}.
\end{align}
Also $m\le1$, so $c<1$ unless $m=c=1$, in which case the claim is trivial.

We will now define a function that will help us bound the tail of the distribution. Let
\begin{align}
h_c(r)=
\begin{cases}
0, & r\le c,\\
\frac{2r-3c+c^3/r^2}{3}, & r>c.
\end{cases}
\end{align}
Using the criticality condition \eqref{eq:sharp_scalar_criticality}, we have
\begin{align}
\ee[(c-R)_+]
=
\ee[h_c(R)]
+
c-\frac{2m}{3}-\frac{c^3}{3m^3}.
\label{eq:sharp_tail_identity}
\end{align}
Indeed, expanding the right-hand side gives
\begin{align}
&\frac{2}{3}\ee[R\mathbf 1\{R>c\}]
-c\pp(R>c)
+\frac{c^3}{3}\ee[R^{-2}\mathbf 1\{R>c\}]
+c-\frac{2m}{3}
-\frac{1}{3}
\left(
\ee[R\mathbf 1\{R\le c\}]
+c^3\ee[R^{-2}\mathbf 1\{R>c\}]
\right)
\\
&\qquad
=
c\pp(R\le c)-\ee[R\mathbf 1\{R\le c\}]
=
\ee[(c-R)_+].
\end{align}

Next, we will upper bound $h_c$ by a quadratic. Let
\begin{align}
a_c=\frac{2c+2c^2-1}{1+c+c^2},
\qquad \text{and} \qquad
\gamma_c=\frac{(1+c+c^2)^2}{3(2+c)},
\qquad \text{and} \qquad
q_c(r)=\gamma_c(r-a_c)^2.
\end{align}
We claim that
\begin{align}
h_c(r)\le q_c(r)\qquad\forall r\ge0.
\label{eq:sharp_quadratic_majorizer}
\end{align}
For $r\le c$, this is immediate because $h_c(r)=0$. For $r>c$, direct
factorization gives
\begin{align}
q_c(r)-h_c(r)
=
\frac{(r-1)^2}{3r^2(2+c)}P_c(r),
\end{align}
where
\begin{align}
P_c(r)
=
r^2(1+2c+3c^2+2c^3+c^4)
-r(4c^3+2c^4)
-(2c^3+c^4).
\end{align}
Moreover,
\begin{align}
P_c(c)=c^2(1-c)^2(1+c)^2\ge0,
\qquad
P_c'(c)=2c(1+c)(1+c+c^3)>0,
\end{align}
and
\begin{align}
P_c''(r)=2(1+c+c^2)^2>0.
\end{align}
Thus $P_c(r)\ge0$ for all $r\ge c$, proving
\eqref{eq:sharp_quadratic_majorizer}. Taking expectations in \eqref{eq:sharp_quadratic_majorizer},
\begin{align}
\ee[h_c(R)]
\le
\gamma_c\ee[(R-a_c)^2]
\le
\gamma_c(1-2a_cm+a_c^2),
\end{align}
where we used $\ee[R^2]\le1$. Combining this with
\eqref{eq:sharp_tail_identity}, it suffices to show
\begin{align}
\gamma_c(1-2a_cm+a_c^2)
+
c-\frac{2m}{3}
-\frac{c^3}{3m^3}
\le
\frac{2c(1-m)}{m}.
\label{eq:sharp_residual_goal}
\end{align}
A direct simplification gives
\begin{align}
&\gamma_c(1-2a_cm+a_c^2)
+
c-\frac{2m}{3}
-\frac{c^3}{3m^3}
-
\frac{2c(1-m)}{m}
\\
&\qquad
=
\frac{(1-m)G_c(m)}{3m^3(c+2)},
\label{eq:sharp_residual_factorization}
\end{align}
where
\begin{align}
G_c(m)
={}&
4c^4m^3-c^4m^2-c^4m-c^4
+8c^3m^3-2c^3m^2-2c^3m-2c^3
\\
&\quad
+6c^2m^3-6c^2m^2
+4cm^3-12cm^2
+2m^3.
\end{align}
Since $m\in[c^{3/4},1]$, it remains to prove $G_c(m)\le0$ on this interval.

We first show that $G_c$ is convex in $m$. We have
\begin{align}
\frac12G_c''(m)
=
m(12c^4+24c^3+18c^2+12c+6)
-(c^4+2c^3+6c^2+12c),
\end{align}
which is increasing in $m$. Hence it suffices to check positivity at
$m=c^{3/4}$. Let $x=c^{1/4}$, so $x\in[2^{-1/4},1]$. Then
\begin{align}
\frac12G_c''(c^{3/4})
=
x^3\Big[
12x^{16}-x^{13}
+24x^{12}-2x^9
+18x^8-6x^5
+12x^4-12x+6
\Big].
\end{align}
Each bracketed group is positive:
\begin{align}
12x^{16}-x^{13}>0,\qquad
24x^{12}-2x^9>0,\qquad
18x^8-6x^5>0,
\end{align}
and $12x^4-12x+6>0$ on $x\in[2^{-1/4},1]$. Therefore
$G_c''(m)>0$ on $m\in[c^{3/4},1]$, so $G_c$ is convex and its maximum is
attained at an endpoint.

At $m=1$, we have that
\begin{align}
G_c(1)=c^4+2c^3-8c+2.
\end{align}
The function $p(c)=c^4+2c^3-8c+2$ is convex on $[1/2,1]$, and
\begin{align}
p(1/2)=-\frac{27}{16}<0,\qquad p(1)=-3<0.
\end{align}
Thus $G_c(1)<0$.

At $m=c^{3/4}$, again write $x=c^{1/4}$. Then
\begin{align}
G_c(c^{3/4})=x^9Q(x),
\end{align}
where
\begin{align}
Q(x)
={}&
4x^{16}-x^{13}+8x^{12}-x^{10}-2x^9+6x^8-x^7-2x^6
-6x^5+4x^4-2x^3-12x+2.
\end{align}
This polynomial factors as
\begin{align}
Q(x)=-3+(x-1)S(x),
\end{align}
where
\begin{align}
S(x)
={}&
4x^{15}+4x^{14}+4x^{13}+3x^{12}
+11x^{11}+11x^{10}+10x^9+8x^8
\\
&\quad
+14x^7+13x^6+11x^5+5x^4
+9x^3+7x^2+7x-5.
\end{align}
Since $x\in[2^{-1/4},1]$, we have $7x-5>0$, and all other terms in
$S(x)$ are nonnegative. Hence $S(x)>0$. Since $x-1\le0$, it follows that
$Q(x)\le -3$, and therefore $G_c(c^{3/4})<0$.

Thus $G_c(m)\le0$ for every $m\in[c^{3/4},1]$. By
\eqref{eq:sharp_residual_factorization}, the residual in
\eqref{eq:sharp_residual_goal} is nonpositive. This proves
\eqref{eq:sharp_critical_tail} for $c>1/2$. Combining the two cases proves the
critical-tail bound, hence $\Phi(c)\ge0$ for every threshold $c$, and
therefore the lemma.
\end{proof}

The preceding lemma gives a defective product rule. If $T=W^2$, it says that for any $K \ge 2$,
\begin{align}
\left(\ee[R T^{3/2}]\right)^{2/3}
\ge
\ee[R]\ee[T]-2(1-\ee[R]).
\end{align}
We identify the defect on the right hand side as the one-step Hellinger loss in \Cref{lem:sharp_clean_augmented_hellinger_comparison}, and the nonlinear envelope chosen below is designed to be stable under this defect, in that if the augmented affinity is large, the defect is small enough relative to the signal, and the nonlinear potential propagates a positive lower bound. If augmented affinity is not large, the potential gives zero; but then the final distance comparison is already trivial because the augmented Hellinger distance is bounded away from zero. This allows the induction in \Cref{lem:sharp_clean_augmented_hellinger_comparison} to close with a constant loss.

We next prove a more general version of the result invoked in the proof of \Cref{lem:sharp_clean_augmented_hellinger_comparison} that allows the nonlinear envelope to be stable under the defect in \Cref{lem:sharp_scalar_defective_holder}.

\begin{corollary}
\label{cor:sharp_scalar_nonlinear_envelope}
\label{cor:scalar_nonlinear_envelope}
Define the nonlinear function
\begin{align}
F_K(x)=((K+1)x-K)_+^{\nicefrac32}.
\end{align}
If $R\ge0$, $X\in[0,1]$, and $\ee[R^2]\le1$, then
\begin{align}
\ee[R F_K(X)]
\ge
F_K(\ee[R]\ee[X]).
\end{align}
\end{corollary}

\begin{proof}
Let $T=((K+1)X-K)_+$ and $W=T^{\nicefrac12}$. 
Note that $W\in[0,1]$ and $F_K(X)=T^{\nicefrac32}=W^3$.

Writing $m=\ee[R]$, \Cref{lem:sharp_scalar_defective_holder} gives
\begin{align}
\left(\ee[R T^{\nicefrac32}]\right)^{\nicefrac23}
=
\left(\ee[R W^3]\right)^{\nicefrac23}
\ge
m\ee[W^2]-K(1-m)
=
m\ee[T]-K(1-m).
\end{align}
Moreover, since $X\le \nicefrac{K+T}{K+1}$, we have that
\begin{align}
(K+1)m\ee[X]-K
\le
m\ee[T]-K(1-m).
\end{align}
The result follows from chaining together the two displays.
\end{proof}

\subsection{Necessity of Horizon Dependence in the Absence of \texorpdfstring{$\kappa$}{kappa}-Domination}\label{app:online_nokappa}

We now show that the polynomial dependence in horizon that appears in \Cref{prop:augmented_hellinger_testing_comparison_no_domination} is necessary with the following lower bound.  In particular, \Cref{prop:augmented_hellinger_testing_comparison_no_domination} is tight up to a polynomial dependence on $(1 - \eta)$.

\begin{proposition}
\label{prop:no_kappa_augmented_hellinger_lower_bound}
For any $H\ge 1$ and $0<\eta<1$, there exists a horizon $H$ MDP with two
actions, policies $\pistar,\pihat$, and a corruption distribution
$\nu$ such that
\begin{align}
\dhel{\pp^{\pistar}}{\pp^{\pihat}}
\ge 4^{-1} \cdot
\sqrt{
\frac{H\cdot\eta}{1-\eta}
\cdot
\dhel{\pp^{\pihat,\pistar_\eta}}{\pp^{\pihat,\pihat_\eta}}
}.
\end{align}
\end{proposition}

\begin{proof}
Let the MDP $M$ have two actions, $\cA = \left\{ 0,1 \right\}$ and $H$ steps such that $s_h$ transitions to $s_{h+1}$ deterministically, independent of the action. Fix a parameter $\epsilon>0$ satisfying $H\epsilon\le \nicefrac12$ and $\epsilon\le \nicefrac{\eta}{1-\eta}$. Let
\begin{align}
\pi^\star_h(0\mid s_h)=1,
\qquad
\pihat_h(0\mid s_h)=1-\epsilon,
\qquad
\nu_h(0\mid s_h)=0
\end{align}
for every $h$. Thus actions are independent across time.  We thus compute
\begin{align}
    \dhel{\pp^{\pistar}}{\pp^{\pihat}} &= 1 - (1 - \epsilon)^{\nicefrac H2}
\end{align}
and thus, since $\epsilon\le \nicefrac{1}{2H}$,
\begin{align}
\dhel{\pp^{\pi^\star}}{\pp^{\pihat}}
\ge
\frac{H\epsilon}{4}.
\label{eq:clean_hellinger_lower_no_kappa}
\end{align}

We now compute the augmented noisy Hellinger distance. Because the transitions
are independent of the action, the rollout distribution of $\pihat$ fixes
the same deterministic state sequence $s_1,\ldots,s_H$. Therefore the augmented rollout distributions
$\pp^{\pihat,\pi^\star_\eta}$ and $\pp^{\pihat,\pihat_\eta}$ differ only in their auxiliary labels, and
these labels are independent across time. We calculate at each state, 
\begin{align}
\pi^\star_{\eta,h}
=
(1-\eta)\delta_0+\eta\delta_1,
\qquad \text{and}\qquad
\pihat_{\eta,h}
=
(1-\eta)(1-\epsilon)\delta_0
+
\bigl(\eta+(1-\eta)\epsilon\bigr)\delta_1.
\end{align}
Then the per-step Hellinger distance is
\begin{align}
d_\eta
\defeq
\dhel{\pi^\star_{\eta,h}}{\pihat_{\eta,h}} = 
1
-
(1-\eta)\sqrt{1-\epsilon}
-
\sqrt{\eta\bigl(\eta+(1-\eta)\epsilon\bigr)}.
\end{align}
Let $a=\nicefrac{(1-\eta)\epsilon}{\eta}$. 
By our choice of $\epsilon$, $a\le1$. Using $\sqrt{1-x}\ge 1-\nicefrac{x}{2}-\nicefrac{x^2}{2}$ for $x\in[0,1]$ and $\sqrt{1+x}\ge 1+\nicefrac{x}{2}-\nicefrac{x^2}{2}$ for $x\in[0,1]$, we obtain
\begin{align}
(1-\eta)\sqrt{1-\epsilon}
+
\sqrt{\eta\bigl(\eta+(1-\eta)\epsilon\bigr)}
&=
(1-\eta)\sqrt{1-\epsilon}
+
\eta\sqrt{1+a}
\\
&\ge
(1-\eta)\left(1-\frac{\epsilon}{2}-\frac{\epsilon^2}{2}\right)
+
\eta\left(1+\frac{a}{2}-\frac{a^2}{2}\right)
\\
&=
1
-
\frac{1-\eta}{2\eta}\epsilon^2.
\end{align}
Since the augmented labels are independent across time, we have
\begin{align}
\dhel{\pp^{\pihat,\pi^\star_\eta}}
{\pp^{\pihat,\pihat_\eta}}
=
1-(1-d_\eta)^H
\le
H d_\eta
\le
\frac{H(1-\eta)}{2\eta}\epsilon^2.
\label{eq:augmented_noisy_hellinger_upper_no_kappa}
\end{align}
Combining \eqref{eq:clean_hellinger_lower_no_kappa} and
\eqref{eq:augmented_noisy_hellinger_upper_no_kappa}, we get
\begin{align}
\dhel{\pp^{\pistar}}{\pp^{\pihat}}
\ge
\frac{H\epsilon}{4}
\ge
\frac{1}{2\sqrt 2}
\sqrt{
\frac{H\eta}{1-\eta}
\cdot
\dhel{\pp^{\pihat,\pi^\star_\eta}}
{\pp^{\pihat,\pihat_\eta}}
}.
\end{align}
This concludes the proof. 
\end{proof}

\subsection{Proof of Theorem \ref{thm:nail}}\label{app:nail_proof}

The main content of the proof is to show the following result, which is that the on-policy KL divergence $\ee\left[ \kld{\pp^{\pihat, \pistar_\eta}}{\pp^{\pihat,\pihat_\eta}} \right]$ for $\pihat$ returned by \Cref{alg:nail} is small.  The result will then follow from \Cref{cor:kl_augmented}.  Indeed, we have the following result.
\begin{lemma}\label{lem:nail_kl}
    Let $\Pi$ be a finite policy class, $\nu$ be an arbitrary corruption distribution, and $0 \leq \eta < 1$ be a corruption level.  Let $\pihat$ denote the policy returned by \Cref{alg:nail}.  Then it holds that
    \begin{align}
        \ee\left[ \kld{\pp^{\pihat, \pistar_\eta}}{\pp^{\pihat, \pihat_\eta}} \right] \leq \frac{H \cdot \log(\abs{\Pi})}{n}.
    \end{align}
\end{lemma}
\begin{proof}
    Let $\mu_t = \sum_{\pi \in \Pi} w_t(\pi) \cdot \pi$ denote the mixture policy defined in \Cref{alg:nail}.
    We first claim that for \emph{any} sequence of trajectories $\tau^{(t)}, \tau^{(t)'}$, it holds that
    \begin{align}\label{eq:regret_guarantee}
        \sum_{t = 1}^n \sum_{h = 1}^H \log\left( \frac{\pistar_{\eta,h}(a_h^{(t)'} | s_h^{(t)})}{\mu_{\eta,h}^{(t)}(a_h^{(t)'} | s_h^{(t)}) } \right) \leq H \cdot \log(\abs{\Pi}).
    \end{align}
    To see this, let
    \begin{align}
        \ell(\pi, \tau') = - \sum_{h  =1}^H \log\left( \pi_{\eta,h}(a_h' | s_h) \right)
    \end{align}
    denote a loss function.  
    We claim that $\ell$ is $\nicefrac 1H$-mixable (\Cref{def:mixability}) and that \Cref{alg:nail} amounts to running the exponential weights algorithm (\Cref{alg:exponential_weights}) with respect to $\ell$ over the policy class
    \begin{align}
        \Pi_\eta = \left\{ \pi_\eta | \pi \in \Pi \right\}.
    \end{align}
    Indeed, fix $\tau'$ and let $w \in \Delta(\Pi)$ and
    \begin{align}
        \mu_w = \sum_{\pi \in \Pi} w(\pi) \cdot \pi.
    \end{align}
    Then, we note that $(\mu_w)_\eta = \sum_{\pi \in \Pi} w(\pi) \cdot \pi_\eta$ by linearity and thus
    \begin{align}
        e^{- H^{-1} \cdot \ell(\mu_w, \tau')} &= \left(\prod_{h = 1}^H (\mu_w)_{\eta,h}(a_h' | s_h)\right)^{\nicefrac 1H} \\
        &= \prod_{h = 1}^H \left( \sum_{\pi \in \Pi} w(\pi) \cdot \pi_{\eta,h}(a_h' | s_h) \right)^{\nicefrac 1H} \\
        &\geq \sum_{\pi \in \Pi} w(\pi) \cdot \left( \prod_{h = 1}^H \pi_{\eta,h}(a_h' | s_h) \right)^{\nicefrac 1H} \\
        &= \sum_{\pi \in \Pi} w(\pi) \cdot e^{- H^{-1} \cdot \ell(\pi, \tau')} \\
        &= \ee_{\pi \sim w}\left[ e^{- H^{-1} \cdot \ell(\pi, \tau')} \right],
    \end{align}
    where the inequality follows from Hölder's inequality.  Thus $\ell$ is $\nicefrac 1H$-mixable.  Moreover, it is immediate that the update in \Cref{alg:nail} corresponds to the exponential weights update with respect to $\ell$ over the policy class $\Pi_\eta$ with learning rate $\lambda = H^{-1}$.  Thus it holds by \Cref{prop:exp_weights_mixability} that \eqref{eq:regret_guarantee} holds for any sequence $\tau^{(t)}$.

    Now, for fixed $1 \leq t \leq n$ and temporarily suppressing the notational dependence on $t$, note that
    \begin{align}
        \ee\left[\sum_{h = 1}^H \log\left( \frac{\pistar_{\eta,h}(a_h' | s_h)}{\mu_{\eta,h}(a_h' | s_h)} \right) \right] &= \ee^{\mu}\ee_{a_h' \sim \pistar_{\eta, h}(\cdot | s_h)}\left[\sum_{h = 1}^H \log\left( \frac{\pistar_{\eta,h}(a_h' | s_h) \cdot \pp_h(s_{h+1} | s_h, a_h)}{\mu_{\eta,h}(a_h' | s_h) \cdot \pp_h(s_{h+1} | s_h, a_h)} \right)  \right] \\
        &= \kld{\pp^{\mu, \pistar_\eta}}{\pp^{\mu, \mu_\eta}}, \label{eq:kl_is_regret}
    \end{align}
    where the last equality follows from the chain rule for KL divergence (\Cref{prop:kl_properties}).  Combining \eqref{eq:regret_guarantee} and \eqref{eq:kl_is_regret} and renormalizing, we have that
    \begin{align}
        \ee\left[ \frac 1n \sum_{t = 1}^n \kld{\pp^{\mu^{(t)}, \pistar_\eta}}{\pp^{\mu^{(t)}, \mu_\eta^{(t)}}}\right] \leq \frac{H \cdot \log\left( \abs{\Pi} \right)}{n}.
    \end{align}
    Since $\pihat=\mu_T$ for $T\sim\mathrm{Unif}([n])$, independent of the training randomness, the left-hand side is exactly the average over $t$. The desired bound follows.
\end{proof}

We can now prove the main result.
\begin{proof}[Proof of \Cref{thm:nail}]
    The result follows immediately by combining \Cref{lem:nail_kl,cor:kl_augmented}. 
\end{proof}

\subsection{Proof of Proposition \ref{prop:onlin_lb}}\label{app:online_lb_pf}

In this section we prove the lower bound \Cref{prop:onlin_lb} demonstrating that even with online access, noisy, stochastic experts necessitate the linear in horizon dependence in sample complexity that is present in \Cref{thm:nail}.  
We first state a more formal version of the lower bound, which is stated in the main body as \Cref{prop:onlin_lb}.

\begin{proposition}\label{prop:onlin_lb_formal}
    For any $H \geq 2$, corruption level $0 < \eta < 1$, and $\epsilon < \nicefrac 1{32}$, there exists a horizon $H$ MDP with three actions, deterministic transitions, a known corruption distribution $\nu$, and a policy class $\Pi$ of size $\abs{\Pi} = 2$ such that $\nu$ is $\kappa$-dominated by $\Pi$ with $\kappa \asymp \nicefrac H \epsilon$ and any online IL algorithm must observe at least
    \begin{align}
        n \geq \frac{\eta \cdot H}{512 (1 - \eta)^2 \cdot \epsilon^2}
    \end{align}
    trajectories of interaction with the noisy expert $\pistar_\eta$ in order to achieve regret $J(\pistar) - J(\pihat) \leq \epsilon$.
\end{proposition}
\begin{proof}
    By \citet{foster2024behavior}, it suffices to show a lower bound against learning in trajectory-wise Hellinger distance to order $\epsilon$.
    Let an MDP have states $s_1, \dots, s_H$ with transitions $s_h \to s_{h+1}$ deterministically independent of the action.  Let the action space be $\cA = \left\{ a_1, a_2, a_3 \right\}$ and the policy space be $\Pi = \left\{ \pi_+, \pi_- \right\}$.  For some $u$ to be determined such that $H u \leq \nicefrac 12$ and $u \leq \nicefrac 14$, let
    \begin{align}
        \pi_{+,h}(\cdot | s_h) = u \cdot \delta_{a_1} + (1 - u) \cdot \delta_{a_2}, \quad \pi_{-,h}(\cdot | s_h) = u \cdot \delta_{a_3} + (1 - u) \cdot \delta_{a_2}, \quad \text{and} \quad \nu_h = \frac 14 \cdot \left( \delta_{a_1} + \delta_{a_3} \right) + \frac 12 \cdot \delta_{a_2}.
    \end{align}
    Note that $\nu_h$ is $\kappa$-dominated by $\Pi$ with $\kappa = \nicefrac 1{4u}$.

    Now observe that
    \begin{align}
        \dhel{\pi_{+,h}(\cdot | s_h)}{\pi_{-,h}(\cdot | s_h)} = 1 - \sqrt{(1 - u)^2} = u.
    \end{align}
    Moreover, because the action distributions are independent,
    \begin{align}
        \dhel{\pp^{\pi_+}}{\pp^{\pi_-}} = 1 - (1 - u)^H \geq \frac{H u}{2},
    \end{align}
    where the inequality used the assumption that $2 H u \leq 1$.  On the other hand, an elementary computation reveals that
    \begin{align}
        \pi_{+,\eta,h} &= \left( \frac{\eta}{4} + (1 - \eta)u \right) \cdot \delta_{a_1} + \left( \frac{\eta}{2} + (1 - \eta)(1 - u) \right)\cdot \delta_{a_2} + \frac{\eta}{4} \cdot \delta_{a_3} \\
        \pi_{-,\eta,h} &= \frac{\eta}{4} \cdot \delta_{a_1} + \left( \frac{\eta}{2} + (1 - \eta)(1 - u) \right)\cdot \delta_{a_2} + \left( \frac{\eta}{4} + (1 - \eta)u \right) \cdot \delta_{a_3}.
    \end{align}
    Thus,
    \begin{align}
        \kld{\pi_{+,\eta,h}}{\pi_{-,\eta,h}} &= \left( \frac{\eta}{4} + (1 - \eta)u \right) \cdot \log\left( 1 + \frac{4(1 - \eta)u}{\eta} \right) + (1 - \eta) u \cdot \log\left( \frac{\eta}{4(1 - \eta)u + \eta} \right) \\
        &= (1 - \eta) u \cdot \log\left( 1 + \frac{4(1 - \eta)u}{\eta} \right) \\
        &\leq \frac{4 ( 1 - \eta)^2 u^2}{\eta}.
    \end{align}
    Thus, after $n$ rounds of interaction, the chain rule for KL divergence (\Cref{prop:kl_properties}) ensures that the KL divergence between the distributions over trajectories induced by $\pi_+$ and $\pi_-$ is at most
    \begin{align}
        n H \cdot \frac{4 (1 - \eta)^2 u^2}{\eta}
    \end{align}
    and thus if
    \begin{align}
        n \leq \frac{\eta}{8 H (1 - \eta)^2 u^2},
    \end{align}
    then by Le Cam's inequality no algorithm can identify $\pistar$ with probability greater than $\nicefrac 34$, which is required in order to achieve regret at most $\epsilon$ if $\nicefrac{H u}{2} \geq 4 \epsilon$.  Thus, setting $u = \nicefrac{8 \epsilon}{H}$, we see that as long as $\epsilon \leq \nicefrac 1{32}$, then $H u \leq \nicefrac 12$ and $u \leq \nicefrac 14$.  Plugging in concludes the proof.
\end{proof}

We now show the necessity of $\kappa$-domination in order to shave off the quadratic dependence on $\epsilon$ in the sample complexity.  We have the following result.
\begin{proposition}\label{prop:kappa_domination_necessary_online}
    For any $\kappa \geq 1$, any corruption level $0 < \eta < 1$, and any $\epsilon < 2^{-4} \cdot \kappa^{-1}$, there exists a horizon $H = 1$ MDP with three actions, a policy class $\Pi$ of size $\abs{\Pi} = 2$, and a known corruption distribution $\nu$ such that:
    \begin{enumerate}
        \item[(a)] $\nu$ is $\kappa$-dominated by $\Pi$ and $\pistar \in \Pi$, and
        \item[(b)] in order for an algorithm to achieve regret $J(\pistar) - J(\pihat) \leq \epsilon$, the learner must observe at least $n \gtrsim \nicefrac{\eta \cdot \kappa}{\epsilon(1-\eta)^2}$ trajectories of interaction with the noisy expert $\pistar_\eta$.
    \end{enumerate}
\end{proposition}
\begin{proof}
    Suppose there is a single state $s$ and three actions $a_1, a_2, a_3$.  Fix some $0 < u \leq \nicefrac{1}{2 \kappa}$ and suppose that $\Pi = \left\{ \pi_+, \pi_- \right\}$ where
    \begin{align}
        \pi_+ = u \cdot \delta_{a_1} + (1 - u) \cdot \delta_{a_2}, \quad \pi_- = u \cdot \delta_{a_3} + (1 - u) \cdot \delta_{a_2}, \quad \text{and} \quad \nu = \kappa u \cdot \delta_{a_1} + \kappa u \cdot \delta_{a_3} + (1 - 2 \kappa u) \cdot \delta_{a_2}.
    \end{align}
    It is immediate that $\nu$ is $\kappa$-dominated by $\Pi$ from the construction.  Moreover, observe that
    \begin{align}
        \dhel{\pp^{\pi_+}}{\pp^{\pi_-}} = 1 - \sqrt{(1 - u)^2} = u.
    \end{align}
    Thus, if $u > 4\epsilon$, then any algorithm returning a policy $\pihat$ such that $\dhel{\pp^{\pihat}}{\pp^{\pistar}} \leq \epsilon$ must identify $\pistar$.  We now observe that an elementary computation reveals that
    \begin{align}
        \kld{\pp^{\pi_{+,\eta}}}{\pp^{\pi_{-,\eta}}} = u (1 - \eta) \cdot \log\left( 1 + \frac{1 - \eta}{\eta \kappa} \right) \leq u \cdot \frac{(1 - \eta)^2}{\eta \kappa}.
    \end{align}
    Thus, if we were to set $u = 8 \epsilon$, then we would have by Le Cam's inequality that if
    \begin{align}
        n \cdot \frac{8 \epsilon \cdot (1 - \eta)^2}{\eta \kappa} \leq \frac{1}{8}
    \end{align}
    then with constant probability, $\pihat \neq \pistar$.  If $\epsilon \leq \nicefrac{1}{16 \kappa}$, then we can find such a $u$ and thus the stated sample complexity is required to achieve regret at most $\epsilon$.  The result follows.
\end{proof}

%% file: body_clean/app_greedy_online_proofs.tex
\section{Proofs from Section \ref{sec:generalizations}}\label{app:greedy_online}

In this appendix, we provide the proofs of the results stated in \Cref{sec:generalizations}.  We begin by describing the algorithm \gnail\ and providing intuition for its design in \Cref{app:gnail_description}.  We then provide the proof of \Cref{thm:gnail} in \Cref{app:gnail_proof}.  Finally, we provide the proof of the lower bound \Cref{prop:gnail_lb} in \Cref{app:gnail_lb_proof}.

\begin{algorithm}[t]
\caption{\gnail: Noise-robust Aggregation for Imitation Learning with Greedy UNcertainty}
\label{alg:gnail}
\begin{algorithmic}[1]

\Require Number of rounds $n$, deterministic policy class $\Pi$, noisy expert $\pistar_\eta$, noise ceiling $\alpha < 1$, contamination ceiling $\rho > 0$.
\State Set $r = \sqrt{\nicefrac{(1 - \alpha)}{\alpha \rho}}$ and $w_1 = \Unif(\Pi)$.
\For{$t = 1$ to $n$}
    \State Define $\mu_t = \sum_{\pi \in \Pi} w_t(\pi) \cdot \pi$ and $\mubar_t(\cdot | s) = \argmax_{a} \mu_t(a | s)$.
    \State Deploy $\mubar_t$ to get trajectory $\tau^{(t)} \sim \pp^{\mubar_t}$.
    \State Query noisy expert $\pistar_\eta$ on $\tau^{(t)}$ to obtain augmented trajectory $\tau^{'(t)}$.
    \State Update $w_{t+1}(\pi) \propto w_t(\pi) \cdot r^{\sum_{h = 1}^H \ind{\pi(s_h^{(t)}) = a_h^{(t)'}}}$.
\EndFor

\State \Return $\pihat = \mubar_T$ for $T \sim \Unif([n])$.

\end{algorithmic}
\end{algorithm}

\subsection{The Problem of Identifiability}\label{app:identifiability}

In the unknown corruption setting, identifiability can be a major concern, i.e., it may be the case that there are multiple policies $\pi$ that are consistent with the observed noisy expert $\pistar_\eta$ and thus it is impossible to identify the clean expert $\pistar$.  We now provide a simple example of this phenomenon. 

Let $M$ be a horizon $H = 1$ MDP with a single state and action space $\cA$.  Let $\pistar = \delta_{\astar}$ be a deterministic expert and let $\nu$ be an arbitrary policy.  For $\eta \leq \alpha$, let
\begin{align}
    \pi = \frac{1 - \alpha}{1 - \eta} \cdot \delta_{\astar} + \frac{\alpha - \eta}{1 - \eta} \cdot \nu.
\end{align}
Then it holds that $\pistar_\alpha = \pi_{\eta}$, so the noisy expert $\pistar_\alpha$ is consistent with both $\pistar$ and $\pi$ as the underlying expert policy, and thus it is impossible to distinguish between $\pistar$ and $\pi$.  In particular, if $J(\pistar) \gg J(\pi)$, then it is impossible to learn a policy with good performance without additional assumptions on the corruption or the feedback.

\subsection{Description of \Cref{alg:gnail}}\label{app:gnail_description}

In \Cref{thm:nail}, we saw that we could get regret bounds that scale polynomially in the horizon $H$ by appealing to exponential weights and in particular the mixability of the trajectory-level log-loss.  Unfortunately, this strategy fundamentally relies on the learner having access to the loss at each round of interaction, which depends on knowing both $\eta$ and $\nu$.  In the unknown corruption setting, we do not have access to this information and thus we cannot directly apply this strategy.  Instead we use a surrogate loss function specifically adapted to the assumption that the clean expert $\pistar$ is \emph{deterministic}.  In particular, we use the following loss function:
\begin{align}\label{eq:margin_loss}
    \ell(\pi, \tau') = \log(r) \cdot \sum_{h = 1}^H \ind{\pi(s_h) \neq a_h'}.
\end{align}
While this loss function itself can grow linearly in horizon, we demonstrate that its induced behavior on the probability of error is substantially nicer.
We eliminate this concern by rolling out our estimated policies \emph{greedily}.  This leads naturally to \Cref{alg:gnail}, which we call \gnail.  At each step, we roll out our current policy greedily, collect noisy expert feedback, and then run \Cref{alg:exponential_weights} with the loss function defined in \eqref{eq:margin_loss}.  The key step in the analysis, provided below, is the observation that the loss in \eqref{eq:margin_loss} has the property that when we make a mistake at any point along the trajectory, the weight we place on a policy that is not the expert goes down by a constant factor in expectation, where critically this constant factor is independent of horizon.  This analysis is thus substantially different from the standard analysis of exponential weights, which relies on the mixability of the loss function to control the sum of losses across rounds, and is more similar to the analysis of the \emph{halving algorithm} for online learning with expert advice \citep{littlestone1988learning}.

\subsection{Proof of Theorem \ref{thm:gnail}}\label{app:gnail_proof}

We begin by restating the theorem with a slightly tighter dependence on the parameters in question.
\begin{theorem}\label{thm:gnail_formal}
    Let $\Pi$ be a class of \emph{deterministic} policies and suppose that $\pistar \in \Pi$.  Let $\nu$ denote a corruption distribution such that for all $h \in [H]$, $a \in \cA$, and $s \in \cS$ it holds that $\nu_h(a | s) \leq \rho$ for some $0 < \rho < 1$.  Let $\alpha > 0$ such that $(1 + \rho) \alpha < 1$ and suppose that the corruption noise in the noisy expert $\eta$ satisfies $\eta \leq \alpha$.  If $\pihat$ is the policy returned by \Cref{alg:gnail}, then it holds that
    \begin{align}
        \ee\left[ \dhel{\pp^{\pihat}}{\pp^{\pistar}} \right] \leq \frac{4 \cdot \log(\abs{\Pi})}{n \left( \sqrt{1 - \alpha} - \sqrt{\alpha \rho} \right)^2}.
    \end{align}
\end{theorem}
We now derive \Cref{thm:gnail} directly from this result.
\begin{proof}[Proof of \Cref{thm:gnail}]
    We compute
    \begin{align}
        \left( \sqrt{1 - \alpha} - \sqrt{\alpha \rho} \right)^2 = \frac{\left( 1 - \alpha (1 + \rho) \right)^2}{\left( \sqrt{1 - \alpha} + \sqrt{\alpha \rho} \right)^2} \geq \frac{\left( 1 - \alpha(1 + \rho) \right)^2}{2}.
    \end{align}
    Plugging this into \Cref{thm:gnail_formal} concludes the proof.
\end{proof}

In order to prove \Cref{thm:gnail_formal}, we must introduce some notation.  First, for any $1 \leq t \leq n$, let $\cF_t$ denote the $\sigma$-algebra generated by all randomness before time $t$. Thus \(w_t\), \(W_t\),
\(\mu_t\), and \(\mubar_t\) are \(\cF_{t-1}\)-measurable.  We recall that
\begin{align}
    r = \sqrt{\frac{1 - \alpha}{\alpha \rho}} \geq 1, \quad \mu_t = \sum_{\pi \in \Pi} w_t(\pi) \cdot \pi, \quad \text{and} \quad w_{t+1}(\pi) \propto w_t(\pi) r^{\sum_{h = 1}^H \ind{\pi\left( s_h^{(t)}  \right) = a_h^{(t)'}}}.
\end{align}
Moreover, $\mubar_t$ denotes the greedy policy associated with $\mu_t$.  We introduce the following notation:
\begin{align}\label{eq:rt_def}
    R_t(\pi) = \frac{w_t(\pi)}{w_{t}(\pistar)} \quad \text{and} \quad W_t = \frac 1{w_t(\pistar)} = 1 + \sum_{\pi \neq \pistar} R_t(\pi).
\end{align}
Note that if $W_t$ is small, then $\mu_t$ places a lot of weight on the correct policy $\pistar$ and, moreover, if $W_t$ is very small, then we expect $\mubar_t$ to agree with $\pistar$ on a trajectory.  We let
\begin{align}\label{eq:mt_def}
    M_t = \left\{ \text{there exists } h \leq H \text{ such that } \mubar_{t,h}(s_h^{(t)}) \neq \pistar_h(s_h^{(t)}) \right\}
\end{align}
denote the event that the rolled out policy $\mubar_t$ disagrees with the \emph{clean} expert $\pistar$ on at least one action.  The key lemma that we will show is that, in expectation, any time a mistake is made, $W_t$ goes down by a constant factor.
\begin{lemma}\label{lem:gnail}
    Let $W_t$ be as in \eqref{eq:rt_def} and let $M_t$ be as in \eqref{eq:mt_def}.  Then it holds for any $1 \leq t \leq n$ that
    \begin{align}
        \ee\left[ W_{t+1} | \cF_{t-1} \right] \leq W_t \left( 1 - \frac{\left( \sqrt{1 - \alpha} - \sqrt{\alpha \rho} \right)^2}{2} \cdot \ee\left[ M_t | \cF_{t-1} \right] \right).
    \end{align}
\end{lemma}

We will now prove \Cref{thm:gnail_formal} assuming \Cref{lem:gnail} and return to the proof of this key result below.
\begin{proof}[Proof of \Cref{thm:gnail_formal}]
    Let $M_t$ be as in \eqref{eq:mt_def}. Applying the fact that \(\hat\pi=\mu_T\) for \(T\sim\mathrm{Unif}([n])\) as well as \citet[Lemma D.3]{foster2024behavior}, we see that
    \begin{align}
        \ee\left[ \dhel{\pp^{\pihat}}{\pp^{\pistar}} \right] &\leq \frac 1n \sum_{t = 1}^n \ee\left[ \dhel{\pp^{\mubar_t}}{\pp^{\pistar}} \right] \leq \frac 2n \sum_{t = 1}^n \pp\left( M_t \right).
    \end{align}
    We now apply Jensen's inequality and \Cref{lem:gnail} to observe that
    \begin{align}
        \ee\left[ \log(W_{t+1}) | \cF_{t-1}\right] &\leq \log \left( \ee\left[ W_{t+1} |\cF_{t-1}\right] \right) \\
        &\leq \log\left(W_t \left( 1 - \frac{\left( \sqrt{1 - \alpha} - \sqrt{\alpha \rho} \right)^2}{2} \cdot \ee\left[ M_{t} | \cF_{t-1} \right] \right)  \right) \\
        &= \log (W_{t}) + \log\left( 1 - \frac{\left( \sqrt{1 - \alpha} - \sqrt{\alpha \rho} \right)^2}{2} \cdot \pp\left( M_{t} | \cF_{t-1} \right) \right) \\
        &\leq \log( W_{t}) - \frac{\left( \sqrt{1 - \alpha} - \sqrt{\alpha \rho} \right)^2}{2} \cdot \pp\left( M_{t} | \cF_{t-1} \right).
    \end{align}
    The last inequality follows due to the fact that
    \begin{align}
        0 \leq \frac{\left( \sqrt{1 - \alpha} - \sqrt{\alpha \rho} \right)^2}{2} \pp\left( M_{t} | \cF_{t-1} \right) \leq 1
    \end{align}
    almost surely and the numerical inequality $\log(1 - x) \leq -x$ for $0 < x < 1$.  Rearranging, taking expectations, and applying the tower property of conditional expectation, we see that
    \begin{align}
        \frac{\left( \sqrt{1 - \alpha} - \sqrt{\alpha \rho} \right)^2}{2} \cdot \sum_{t = 1}^n \pp\left( M_t \right) &\leq \sum_{t = 1}^n \ee\left[ \log(W_t)  - \log(W_{t+1})\right] \\
        &\leq \log(W_1) \\
        &= \log\left( \abs{\Pi} \right),
    \end{align}
    where the second inequality follows because $W_t \geq 1$ by construction in \eqref{eq:rt_def} and the equality follows by assuming a uniform prior $w_1$.  Rearranging concludes the proof.
\end{proof}

Thus it remains to prove the key intermediate result, \Cref{lem:gnail}.
\begin{proof}[Proof of \Cref{lem:gnail}]
    We fix a time $t$ and omit the $t$ from the notation of the trajectories $\tau$ and $\tau'$ in order to simplify the presentation. 
    Let \(\cG_t\) denote the sigma-algebra generated by \(\cF_{t-1}\)
and the states \(s_1^{(t)},\ldots,s_H^{(t)}\), but not the noisy labels
\(a_1^{(t)\prime},\ldots,a_H^{(t)\prime}\). That is, $\cG_t=\cF_{t-1} \vee \sigma(s^{(t)}_{1:H})$.
    Observe that for $\pi \neq \pistar$, letting $\astar_h = \pistar_h(s_h)$ and $\abar_h = \pi_h(s_h)$, we have 
    \begin{align}
        R_{t+1}(\pi) = R_t(\pi) \cdot \exp\left( \log(r) \cdot \sum_{h = 1}^H \left(\ind{\abar_h = a_h'} - \ind{\astar_h = a_h'} \right)\right).
    \end{align}
    In the event that $\abar_h = \astar_h$, clearly $\ind{\abar_h = a_h'} - \ind{\astar_h = a_h'} = 0$.  On the other hand, if $\abar_h \neq \astar_h$, then by the margin assumption,
    \begin{align}
        \pp\left( a_h' = \astar_h | s_h \right) \geq 1 - \alpha \quad \text{and} \quad \pp\left( a_h' = \abar_h | s_h \right) \leq \alpha \rho.
    \end{align}
    Thus,
    \begin{align}
        \ee\left[ r^{\ind{\pi_h(s_h) = a_h'} - \ind{\pistar_h(s_h) = a_h'}} | s_h\right] &= 1 - \pp\left( a_h' = \astar_h | s_h \right) - \pp\left( a_h' = \abar_h | s_h \right) \\
        &\quad + \frac{\pp\left( a_h' = \astar_h | s_h \right)}{r} + r \cdot \pp\left( a_h' = \abar_h | s_h \right) \\
        &= 1 - \left( 1 - \frac 1r \right) \cdot \pp\left( a_h' = \astar_h | s_h \right) + (r - 1) \cdot \pp\left( a_h' = \abar_h | s_h \right).
    \end{align}
    By the assumptions on $\alpha, \rho$, it holds that $r > 1$ and thus combining the preceding two displays, we have
    \begin{align}
        \ee\left[ r^{\ind{\pi_h(s_h) = a_h'} - \ind{\pistar_h(s_h) = a_h'}} | s_h\right] &\leq 1 - \left( 1 - \frac 1r \right) \cdot (1 - \alpha) + (r - 1) \cdot \alpha \rho \\
        &= 1 - \left( \sqrt{1 - \alpha} - \sqrt{\alpha \rho} \right)^2,
    \end{align}
    where we used the fact that $r = \sqrt{\nicefrac{(1 - \alpha)}{\alpha \rho}}$ in the last equality.

    Letting
    \begin{align}
        N_t(\pi) = \sum_{h =1}^H \ind{\pi_h(s^{(t)}_h) \neq \pistar_h(s^{(t)}_h)},
    \end{align}
    we thus conclude that
    \begin{align}
        \ee\left[ R_{t+1}(\pi) | \cG_t \right] &\leq R_t(\pi) \cdot \left( 1 -  \left( \sqrt{1 - \alpha} - \sqrt{\alpha \rho} \right)^2\right)^{N_t(\pi)} \\
        &\leq R_t(\pi) \left( 1 - \left( \sqrt{1 - \alpha} - \sqrt{\alpha \rho} \right)^2 \cdot \ind{N_t(\pi) \geq 1} \right).
    \end{align}
    Summing over $\pi \neq \pistar$, we have that
    \begin{align}
        \ee\left[ W_{t+1} - 1 | \cG_t \right] &= \ee\left[ \sum_{\pi \neq \pistar} R_{t+1}(\pi) | \cG_t \right] \\
        &\leq \sum_{\pi \neq \pistar} R_{t}(\pi) \left( 1 - \left( \sqrt{1 - \alpha} - \sqrt{\alpha \rho} \right)^2 \cdot \ind{N_t(\pi) \geq 1} \right) \\
        &= W_t - 1 - \left( \sqrt{1 - \alpha} - \sqrt{\alpha \rho} \right)^2 \cdot \sum_{N_t(\pi) \geq 1} R_t(\pi). \label{eq:lem_gnail_key}
    \end{align}
    Now let $a_h = \mubar_{t,h}(s_h)$ be the rolled out action and suppose that $M_t$ occurs; then there is some minimal $h \leq H$ such that $a_h \neq \astar_h$.  Because $\mubar_t$ is the greedy policy, it must then hold that
    \begin{align}
        \sum_{\substack{\pi \in \Pi \\ \pi(s_h) = a_h}} w_t(\pi) \geq \sum_{\substack{\pi \in \Pi \\ \pi(s_h) = \astar_h}} w_t(\pi)
    \end{align}
    and thus $\mu_{t,h}(\astar_h | s_h) \leq \nicefrac 12$ and, in particular,
    \begin{align}
        \sum_{N_t(\pi) \geq 1} w_t(\pi) \geq \frac 12.
    \end{align}
    Dividing by $w_t(\pistar)$, we see that
    \begin{align}
        \sum_{N_t(\pi) \geq 1} R_t(\pi) \geq \frac{1}{2 \cdot w_t(\pistar)} = \frac{W_t}{2}.
    \end{align}
    By \eqref{eq:lem_gnail_key}, it clearly holds that $W_{t+1}$ is a supermartingale.  Moreover, combining that equation with the preceding display, we see that the result holds.
\end{proof}

\subsection{Proof of Proposition \ref{prop:gnail_lb}}\label{app:gnail_lb_proof}

We now prove that \gnail\ is essentially optimal up to constants, the content of \Cref{prop:gnail_lb}.  We first state a more formal version of the result before giving the construction.
\begin{proposition}\label{prop:gnail_lb_formal}
    For any $\epsilon < \nicefrac 18$, any $0 < \alpha < 1$ and any $0 < \rho < 1$ such that $\alpha (1 + \rho) < 1$, there exists a horizon $H = 2$ MDP and a \emph{deterministic} policy class $\Pi$ of size $\abs{\Pi} = M$ with $M \leq 1 + \nicefrac 1\rho$ such that there exists a family of $\rho$-smooth measures $\left\{ \nu_i \right\}_{i=1}^{M}$ such that for some $\eta \leq \alpha$, any online IL algorithm with access to a noisy expert $\pistar_\eta$ must observe
    \begin{align}
        n \gtrsim \frac{1}{\epsilon \left( 1 - \alpha(1 + \rho) \right)\cdot \log\left(1 + \frac{1 - \alpha (1 + \rho)}{\alpha \rho}\right)}
    \end{align}
    trajectories in order to guarantee expected regret $J(\pistar) - J(\pihat) \leq \epsilon$.
\end{proposition}
\begin{proof}
    Let $m = \left\lceil \nicefrac 1\rho \right\rceil$ and let $\Pi = \left\{ \pi^1, \ldots, \pi^m \right\}$ be a class of $m$ deterministic policies. 
    Let $\cA = \left[ m \right] \cup \left\{ \perp \right\}$ and take the corruption level to be $\eta = \alpha$.  We suppose there is a deterministic initial state $s_1$ and all policies map to the same action, which transitions to state $s_2$ with probability $q$ and $s_2'$ with probability $1 - q$ for some $q$ to be determined.  For $1 \leq i \leq m$, let
    \begin{align}
        \pi^i(s) = \begin{cases}
            \perp & s = s_1 \\
            i & s = s_2 \\
            \perp & s = s_2'
        \end{cases},
    \end{align}
    and let
    \begin{align}
        \nu_i(i) = 0, \quad \nu_i(j) = \rho \text{ for } j \in [m] \setminus \left\{ i \right\}, \quad \text{and} \quad \nu_i(\perp) = 1 - (m - 1) \rho.
    \end{align}
    
    In other words, $\pi^i$ is informative only on state $s_2$ and reveals the `correct' action, whereas $\nu_i$ places no mass on $i$ and distributes mass evenly otherwise.  Note that by construction $\nu_i$ is always $\rho$-smooth.

    We first note that $\pi_i$ and $\pi_j$ differ only upon transitioning to $s_2$ and thus
    \begin{align}
        \dhel{\pp^{\pi_i}}{\pp^{\pi_j}} = q
    \end{align}
    and so for any possibly randomized $\pihat$,
    \begin{align}
        \ee\left[ \dhel{\pp^{\pi_i}}{\pp^{\pihat}} \right] = q \cdot \pp\left( \pihat(s_2) \neq i \right).
    \end{align}
    Thus if $q = 4 \epsilon$, it suffices to show that for insufficiently many rounds of interaction, $\pp\left( \pihat(s_2) \neq i \right) \geq \nicefrac 14$.

    Now, we will let $Q_i$ denote the trajectory distribution obtained by rolling out $(1 - \alpha)\cdot \pi_i + \alpha \cdot \nu_i$ and we will denote by $p_i = (1 - \alpha) \delta_i + \alpha \cdot \nu_i$ and note that
    \begin{align}
        p_i(j) = \begin{cases}
            1 - \alpha & j = i \\
            \alpha \rho & j \in [m] \setminus \left\{ i \right\} \\
            \alpha(1-(m-1)\rho) & j = \bot \\
            0 & \text{otherwise}.
        \end{cases}
    \end{align}
    Thus,
    \begin{align}
        \kld{p_i}{p_j} = (1 - \alpha) \cdot \log\left( \frac{1 - \alpha}{\alpha \rho} \right) + \alpha \rho \cdot \log\left( \frac{\alpha \rho}{1 - \alpha} \right) = (1 - \alpha(1 + \rho)) \cdot \log\left( 1 + \frac{1 - \alpha(1 + \rho)}{\alpha \rho} \right).
    \end{align}
    Moreover, because $p_i$ and $p_j$ differ from each other only upon transitioning to $s_2$, it holds by the chain rule (\Cref{prop:kl_properties}) that
    \begin{align}
        \kld{Q_i}{Q_j} = q \cdot (1 - \alpha(1 + \rho)) \cdot \log\left( 1 + \frac{1 - \alpha(1 + \rho)}{\alpha \rho} \right) = 4 \epsilon \cdot (1 - \alpha(1 + \rho)) \cdot \log\left( 1 + \frac{1 - \alpha(1 + \rho)}{\alpha \rho} \right).
    \end{align}
    Applying Le Cam's two-point method to any pair $i\neq j$ and the chain rule again, it holds that in order for $\pihat$ to have a better than $\nicefrac 34$ probability of selecting the correct index, it must hold that
    \begin{align}
        n \gtrsim \frac{1}{4 \epsilon \cdot (1 - \alpha(1 + \rho)) \cdot \log\left( 1 + \frac{1 - \alpha(1 + \rho)}{\alpha \rho} \right)}.
    \end{align}
    The result follows.

\end{proof}

Finally, we can prove the main result.
\begin{proof}[Proof of \Cref{prop:gnail_lb}]
    We observe that
    \begin{align}
        \log\left( 1 + \frac{1 - \alpha(1 + \rho)}{\alpha \rho} \right) \leq \frac{1 - \alpha(1 + \rho)}{\alpha \rho}
    \end{align}
    and apply \Cref{prop:gnail_lb_formal}.
\end{proof}

%% file: body_clean/app_testing_online_proofs.tex
\section{Achieving Horizon-Free Guarantees under \texorpdfstring{$\kappa$}{kappa}-Domination}
\label{app:horizon_free_testing}

In this section, we show that under $\kappa$-domination, it is possible to remove the horizon dependence present in the analysis of \nail\ (cf. \Cref{thm:nail}) via a different procedure that pays linearly in the size of the policy class. We also demonstrate how to shave off a factor of $H$ from \Cref{thm:nail} even without $\kappa$-domination using a similar approach. 

The analysis of \nail\ controls an augmented KL objective through an exponential-weights argument over trajectory log-losses. Since the loss is accumulated over an entire trajectory, the update is effectively based on a $1/H$-mixable loss, which leads to an $H\cdot \nicefrac{\log|\Pi|}{n}$ term in the guarantee, raising the question of whether this horizon dependence is necessary.

One plausible strategy is to perform exponential-weights updates within an episode, after each queried expert label. Such an update would operate on per-step log-losses, which are $1$-mixable, and therefore would avoid the $1/H$-mixability loss. However, this procedure does not directly define a deployable rollout policy. Indeed, the policy used at step $h$ would depend on the noisy expert labels observed at earlier states in the same episode. Thus, rolling out this policy requires access to the expert during the rollout, not merely during training. 

Indeed, one could try to recover deployability by running standard (unnoised) behavior cloning on this feedback-adaptive rollout policy. Unfortunately, this procedure would then become misspecified and we would be forced to learn a time-dependent mixture of policy classes which would naturally incur both a horizon factor and a linear dependence on $|\Pi|$ in the sample complexity.

In this subsection, we explore the alternative strategy of controlling the augmented Hellinger distance directly through a testing procedure, by directly appealing to \Cref{thm:augmented_hellinger}, as opposed to controlling the augmented KL divergence like \nail\ does using \Cref{cor:kl_augmented}. More specifically, we give a different procedure that tests candidate rollout policies one at a time on their own induced state distributions, using noisy expert labels to decide whether the candidate's noisy predictions are likelihood-plausible. The resulting algorithm outputs a deployable policy and, under $\kappa$-domination, removes the explicit horizon dependence in the sample complexity. The price we pay for this improvement is that the guarantee scales linearly in $|\Pi|$, rather than logarithmically in $|\Pi|$. Absent $\kappa$-domination, the same algorithm again shaves off a factor of $H$ from \Cref{thm:nail} but is still inefficient. We leave open whether a horizon-free algorithm with worst-case logarithmic dependence on $|\Pi|$ is possible under $\kappa$-domination alone. 

We want to emphasize that this theoretical result should be viewed as complementary to the OPD/\nail\ analysis, rather than as a practical replacement for the augmented KL objectives studied in OPD. Indeed, the key practical challenge in bounding the augmented Hellinger distance is that it does not decompose additively over time steps, which makes it difficult to optimize directly. 

\subsection{Algorithm and Guarantees.}

\begin{algorithm}[t]
\caption{\neil: Noise-robust Elimination for Imitation Learning}
\label{alg:candidate_rollout_testing}
\begin{algorithmic}[1]
\Require Policy class $\Pi$, noisy expert $\pistar_\eta$, known corruption
level $\eta$, known corruption distribution $\nu$, number of rollouts per
candidate $m$, threshold $\beta>0$.
\State Initialize the surviving set $V\gets \Pi$.
\For{$\pi\in\Pi$}
    \State Roll out $\pi$ for $m$ episodes. Denote the visited states by
    \begin{align}
    \{s^{\pi,i}_h: i\in[m],\, h\in[H]\}.
    \end{align}
    \State Query the noisy expert at each visited state:
    \begin{align}
    a^{\prime,\pi,i}_h \sim \pistar_{\eta,h}(\cdot\mid s^{\pi,i}_h).
    \end{align}
    \State For each $\pi\in\Pi$, compute the noisy-label log-likelihood
    \begin{align}
    L_{\pi}(\pi)
    =
    \sum_{i=1}^m\sum_{h=1}^H
    \log \pi_{\eta,h}(a^{\prime,\pi,i}_h\mid s^{\pi,i}_h).
    \end{align}
    \State Let $L_{\pi}^{\max}=\max_{\pi\in\Pi}L_\pi(\pi)$.
    \If{$L_{\pi}^{\max}-L_{\pi}(\pi)>\beta$}
        \State Remove $\pi$ from $V$.
    \EndIf
\EndFor
\State \Return any $\widehat\pi\in V$.
\end{algorithmic}
\end{algorithm}

The algorithm functions through successive elimination of candidate policies based on their performance in a noisy-label likelihood test. For each candidate rollout policy $\pi \in \Pi$, the algorithm performs $m$ rollouts of $\pi$, collecting the states visited during these rollouts. It then queries the noisy expert for labels at these visited states, and computes the log-likelihood of these labels under each candidate policy's noisy version $\pi_\eta$. If the log-likelihood of $\pi$ is significantly worse than the best-performing candidate (by more than a threshold $\beta$), then $\pi$ is eliminated from consideration. After testing all candidates, the algorithm returns any surviving policy. The following theorem shows that the resulting algorithm, whose pseudocode is given in \Cref{alg:candidate_rollout_testing},
achieves a horizon-free guarantee under $\kappa$-domination. 

\begin{theorem}
\label{thm:kappa_candidate_testing}
Let $\pistar\in\Pi$ and suppose that $\nu$ is $\kappa$-dominated by $\Pi$. Assume that $\eta$ and $\nu$ are known. If $\pihat$ is the policy returned by \Cref{alg:candidate_rollout_testing} after a total of $n$ episodes of interaction with the noisy expert $\pistar_\eta$, then there exist universal constants $c,C>0$ and a threshold $\beta$ such that with probability at least $1-\delta$, 
\begin{align}
    \dhel{\pp^{\pistar}}{\pp^{\widehat\pi}}\lesssim
\frac{1+\eta\kappa}{(1-\eta)^2}
\cdot
\frac{|\Pi|\log(|\Pi|/\delta)}{n}.
\end{align}
\end{theorem}
We also give a version of \Cref{thm:kappa_candidate_testing} without the assumption of $\kappa$-domination.
\begin{theorem}
\label{thm:candidate_testing_no_domination}
Let \(\pistar\in\Pi\). Assume that \(\eta\) and \(\nu\) are known, but no
domination assumption on \(\nu\) is made. If \(\widehat\pi\) is the policy returned by
\Cref{alg:candidate_rollout_testing} after a total of \(n\) episodes of
interaction with the noisy expert \(\pistar_\eta\), then there exist universal
constants \(c,C>0\) and a threshold \(\beta\) such that with probability at
least \(1-\delta\),
\begin{align}
    \dhel{\pp^{\pistar}}{\pp^{\widehat\pi}}
    \lesssim
    \frac{\sqrt{H}}{1-\eta}
    \sqrt{
    \frac{|\Pi|\log(|\Pi|/\delta)}{n}
    }.
\end{align}
\end{theorem}

\subsection{Proofs of \Cref{thm:kappa_candidate_testing} and \Cref{thm:candidate_testing_no_domination}.}

The heart of the analysis of \Cref{thm:kappa_candidate_testing} is the horizon‑free comparison between the clean trajectory Hellinger distance and an observable “augmented” Hellinger distance from \Cref{thm:augmented_hellinger}, which roughly states that, under $\kappa$-domination, for any candidate rollout policy $\pi$,  
\begin{align}
\dhel{\pp^{\pistar}}{\pp^{\pi}}
\lesssim
\frac{1+\eta\kappa}{(1-\eta)^2} \cdot 
\dhel{\pp^{\pi,\pistar_\eta}}{\pp^{\pi,{\pi}_\eta}}.
\end{align}
In light of this result, the testing procedure in \Cref{alg:candidate_rollout_testing} can be viewed as an attempt to find a policy that minimizes the augmented Hellinger distance on the right-hand side, which is observable through the noisy expert labels.

To prove the final result, we combine this comparison with the following standard finite-class MLE/Hellinger testing bound applied to the class
$\{\pp^{\pi,\pistar_\eta}:\pi\in\Pi\}$ for each fixed rollout policy $\pi$.

\begin{lemma}
\label{lem:finite_class_likelihood_testing}
Let $\mathcal Q$ be a finite class of distributions and consider a distribution
$P\in\mathcal Q$. Let $X_1,\ldots,X_m\sim P$ independently. For
$Q\in\mathcal Q$, let
\begin{align}
L_m(Q)=\sum_{i=1}^m \log Q(X_i).
\end{align}
First, for any $\beta>0$,
\begin{align}
\pp\left(
\exists Q\in\mathcal Q:
L_m(Q)-L_m(P)>\beta
\right)
\le
|\mathcal Q|e^{-\beta/2}.
\end{align}
Second, for any fixed $Q\in\mathcal Q$,
\begin{align}
\pp\left(
L_m(P)-L_m(Q)\le \beta
\right)
\le
\exp\left(\frac{\beta}{2}-m\dhel{P}{Q}\right).
\end{align}
\end{lemma}

\begin{proof}
For the first claim, fix $Q\in\mathcal Q$. The event
$L_m(Q)-L_m(P)>\beta$ is equivalent to
$\prod_{i=1}^m\nicefrac{Q(X_i)}{P(X_i)}>e^\beta.
$
Therefore
\begin{align}
\ind{L_m(Q)-L_m(P)>\beta}
\le
e^{-\beta/2}
\prod_{i=1}^m
\sqrt{\frac{Q(X_i)}{P(X_i)}}.
\end{align}
Taking an expectation under $P$, we get
\begin{align}
\pp\left(L_m(Q)-L_m(P)>\beta\right)
\le
e^{-\beta/2}
\left(
\sum_x \sqrt{P(x)Q(x)}
\right)^m
\le
e^{-\beta/2}.
\end{align}
A union bound proves the first claim. The second claim follows from a similar argument and recalling the definition of Hellinger distance.
\end{proof}
We are now ready to prove \Cref{thm:kappa_candidate_testing}.
\begin{proof}[Proof of \Cref{thm:kappa_candidate_testing}]
We refer to a policy $\pi\in\Pi$ as $\epsilon$-bad if $\dhel{\pp^{\pistar}}{\pp^{\pi}}>\epsilon$. We will prove that, with high probability, every $\epsilon$-bad candidate policy $\pi$ is eliminated
and $\pistar$ is not eliminated.

First consider the test for $\pi=\pistar$, in which case the data are drawn
from $\pp^{\pistar,\pistar_\eta}$. By the first part of
\Cref{lem:finite_class_likelihood_testing}, applied to the finite class $\{\pp^{\pistar,\pistar_\eta}:\pi\in\Pi\}$, we have
\begin{align}
\pp\left(
\exists \pi\in\Pi:
L_{\pistar}(\pi)-L_{\pistar}(\pistar)>\beta
\right)
\le
|\Pi|e^{-\beta/2}.
\end{align}
The complement of this event is the event that $\pistar$ is not eliminated, because if $\pistar$ were eliminated, then there would exist some $\pi\in\Pi$ such that $L_{\pistar}(\pi)-L_{\pistar}(\pistar)>\beta$. Now fix an $\epsilon$-bad $\pi\in\Pi$. If $\pi$ is not eliminated, then 
\begin{align}
L_{\pi}^{\max} - L_{\pi}(\pi)\le\beta, 
\qquad 
\text{so}
\qquad 
L_{\pi}(\pistar) - L_{\pi}(\pi)\le\beta.
\end{align}
Since the data collected in this test are distributed according to
$\pp^{\pi,\pistar_\eta}$, applying the second part of
\Cref{lem:finite_class_likelihood_testing} yields
\begin{align}
\pp(\pi\text{ is not eliminated})
\le
\exp\left(
\frac{\beta}{2}
-
m\dhel{\pp^{\pi,\pistar_\eta}}{\pp^{\pi,\pi_\eta}}
\right).
\end{align}
Letting $\alpha = \nicefrac{\epsilon \cdot (1-\eta)^2}{31(1+\eta(2\kappa-1))}$, we can use \Cref{prop:augmented_hellinger_testing_comparison} and the fact that $\pi$ is $\epsilon$-bad to lower bound the Hellinger distance in the exponent:
\begin{align}
\exp\left(\frac{\beta}{2}-m\dhel{\pp^{\pi,\pistar_\eta}}{\pp^{\pi,\pi_\eta}}\right)
\le
\exp\left(\frac{\beta}{2}-m \alpha\right),
\end{align}
Taking a union bound over all $\pi\in\Pi$, the probability that either
$\pistar$ is eliminated or some $\epsilon$-bad policy survives is at most
\begin{align}
|\Pi|e^{-\beta/2}
+
|\Pi|\exp\left(\frac{\beta}{2}-m\alpha\right).
\end{align}
In order to make the above failure probability at most $\delta$, it suffices to have
\begin{align}
\beta=2\log\frac{4|\Pi|}{\delta} \qquad
\text{and } \qquad
m
\ge
\frac{2\beta+2\log(4|\Pi|/\delta)}{\alpha}.
\end{align} 
Thus, it suffices to take
\begin{align}
m
\gtrsim
\frac{1+\eta\kappa}{(1-\eta)^2}
\cdot
\frac{\log(|\Pi|/\delta)}{\epsilon}.
\end{align}
Then, any returned policy
$\widehat\pi\in V$ satisfies $\dhel{\pp^{\pistar}}{\pp^{\widehat\pi}}\le \epsilon$. The total number of interaction episodes is $n=m|\Pi|$, which yields the stated result. 
\end{proof}

\Cref{thm:kappa_candidate_testing} gives a horizon-free guarantee under
$\kappa$-domination, but it does so by testing policies one at a time. Thus it
removes the explicit horizon factor present in the exponential-weights analysis
of \Cref{alg:nail}, at the price of replacing the logarithmic
$\log|\Pi|$ dependence by a linear $|\Pi|$ dependence. 

\begin{proof}[Proof of \Cref{thm:candidate_testing_no_domination}]
We proceed using a similar argument to the proof of \Cref{thm:kappa_candidate_testing}, but using \Cref{prop:augmented_hellinger_testing_comparison_no_domination} instead of \Cref{prop:augmented_hellinger_testing_comparison}. We refer to a policy \(\pi\in\Pi\) as \(\epsilon\)-bad if $\dhel{\pp^{\pistar}}{\pp^\pi}>\epsilon$. We will prove that, with high probability, every \(\epsilon\)-bad candidate policy \(\pi\) is eliminated and \(\pistar\) is not eliminated.

First consider the test for \(\pi=\pistar\), in which case the data are drawn
from \(\pp^{\pistar,\pistar_\eta}\). By the first part of
\Cref{lem:finite_class_likelihood_testing}, applied to the finite class
\(\{\pp^{\pistar,\pi_\eta}:\pi\in\Pi\}\), we have
\begin{align}
\pp\left(
\exists \pi\in\Pi:
L_{\pistar}(\pi)-L_{\pistar}(\pistar)>\beta
\right)
\le
|\Pi|e^{-\beta/2}.
\end{align}
On the complement of this event, \(\pistar\) is not eliminated. Now fix an \(\epsilon\)-bad \(\pi\in\Pi\). If \(\pi\) is not eliminated, then
\begin{align}
L_\pi^{\max}-L_\pi(\pi)\le\beta,
\qquad
\text{so}
\qquad
L_\pi(\pistar)-L_\pi(\pi)\le\beta.
\end{align}
Since the data collected in this test are distributed according to
\(\pp^{\pi,\pistar_\eta}\), applying the second part of
\Cref{lem:finite_class_likelihood_testing} yields
\begin{align}
\pp(\pi\text{ is not eliminated})
\le
\exp\left(
\frac{\beta}{2}
-
m\dhel{\pp^{\pi,\pistar_\eta}}{\pp^{\pi,\pi_\eta}}
\right),
\end{align}
where \(m\) denotes the number of rollouts per candidate. Letting
$\alpha
=
\nicefrac{(1-\eta)^2\epsilon^2}{31^2H}$, we can use  \Cref{prop:augmented_hellinger_testing_comparison_no_domination} and the fact that \(\pi\) is \(\epsilon\)-bad to lower bound the Hellinger distance in the exponent as
\begin{align}
\pp(\pi\text{ is not eliminated})
\le
\exp\left(\frac{\beta}{2}-m\alpha\right).
\end{align}
Taking a union bound over all \(\pi\in\Pi\), the probability that either
\(\pistar\) is eliminated or some \(\epsilon\)-bad policy survives is at most
\begin{align}
|\Pi|e^{-\beta/2}
+
|\Pi|\exp\left(\frac{\beta}{2}-m\alpha\right).
\end{align}
In order to make the above failure probability at most \(\delta\), it suffices
to take
\begin{align}
\beta=2\log\frac{4|\Pi|}{\delta}
\qquad
\text{and}
\qquad
m
\ge
\frac{2\beta+2\log(4|\Pi|/\delta)}{\alpha}.
\end{align}
Thus, it suffices to take
\begin{align}
m
\gtrsim
\frac{H}{(1-\eta)^2}
\cdot
\frac{\log(|\Pi|/\delta)}{\epsilon^2}.
\end{align}
Then, any returned policy \(\widehat\pi\in V\) satisfies $\dhel{\pp^{\pistar}}{\pp^{\widehat\pi}}\le\epsilon$. The total number of interaction episodes is \(n=m|\Pi|\), which yields the stated result.
\end{proof}

%% file: refs.bib
@article{foster2024behavior,
  title={Is behavior cloning all you need? understanding horizon in imitation learning},
  author={Foster, Dylan J and Block, Adam and Misra, Dipendra},
  journal={Advances in Neural Information Processing Systems},
  volume={37},
  pages={120602--120666},
  year={2024}
}

@article{foster2024online,
  title={Online estimation via offline estimation: An information-theoretic framework},
  author={Foster, Dylan J and Han, Yanjun and Qian, Jian and Rakhlin, Alexander},
  journal={Advances in Neural Information Processing Systems},
  volume={37},
  pages={42840--42898},
  year={2024}
}

@inproceedings{block2024butterfly,
  title={Butterfly Effects of SGD Noise: Error Amplification in Behavior Cloning and Autoregression},
  author={Block, Adam and Foster, Dylan J and Krishnamurthy, Akshay and Simchowitz, Max and Zhang, Cyril},
  booktitle={The Twelfth International Conference on Learning Representations},
  year={2024},
}

@inproceedings{chang2023learning,
  title={Learning to Generate Better Than Your LLM},
  author={Chang, Jonathan and Brantley, Kiant{\'e} and Ramamurthy, Rajkumar and Misra, Dipendra and Sun, Wen},
  booktitle={NeurIPS 2023 Workshop on Instruction Tuning and Instruction Following},
  year={2023}
}

@inproceedings{rohatgi2025computational,
  title={Computational-Statistical Tradeoffs at the Next-Token Prediction Barrier: Autoregressive and Imitation Learning under Misspecification},
  author={Rohatgi, Dhruv and Block, Adam and Huang, Audrey and Krishnamurthy, Akshay and Foster, Dylan J},
  booktitle={The Thirty Eighth Annual Conference on Learning Theory},
  pages={4831--4837},
  year={2025},
  organization={PMLR}
}

@inproceedings{ross2011reduction,
  title={A reduction of imitation learning and structured prediction to no-regret online learning},
  author={Ross, St{\'e}phane and Gordon, Geoffrey and Bagnell, Drew},
  booktitle={Proceedings of the fourteenth international conference on artificial intelligence and statistics},
  pages={627--635},
  year={2011},
  organization={JMLR Workshop and Conference Proceedings}
}

@inproceedings{ross2010efficient,
  title={Efficient reductions for imitation learning},
  author={Ross, St{\'e}phane and Bagnell, Drew},
  booktitle={Proceedings of the thirteenth international conference on artificial intelligence and statistics},
  pages={661--668},
  year={2010},
  organization={JMLR Workshop and Conference Proceedings}
}

@article{pomerleau1988alvinn,
  title={{Alvinn}: An autonomous land vehicle in a neural network},
  author={Pomerleau, Dean A},
  journal={Advances in neural information processing systems},
  volume={1},
  year={1988}
}

@article{ho2016generative,
  title={Generative adversarial imitation learning},
  author={Ho, Jonathan and Ermon, Stefano},
  journal={Advances in neural information processing systems},
  volume={29},
  year={2016}
}

@article{block2023provable,
  title={Provable guarantees for generative behavior cloning: Bridging low-level stability and high-level behavior},
  author={Block, Adam and Jadbabaie, Ali and Pfrommer, Daniel and Simchowitz, Max and Tedrake, Russ},
  journal={Advances in Neural Information Processing Systems},
  volume={36},
  pages={48534--48547},
  year={2023}
}

@article{chi2025diffusion,
  title={Diffusion policy: Visuomotor policy learning via action diffusion},
  author={Chi, Cheng and Xu, Zhenjia and Feng, Siyuan and Cousineau, Eric and Du, Yilun and Burchfiel, Benjamin and Tedrake, Russ and Song, Shuran},
  journal={The International Journal of Robotics Research},
  volume={44},
  number={10-11},
  pages={1684--1704},
  year={2025},
  publisher={Sage Publications Sage UK: London, England}
}

@article{zhang2026embarrassinglysimpleselfdistillationimproves,
  title={Embarrassingly Simple Self-Distillation Improves Code Generation},
  author={Zhang, Ruixiang and Bai, Richard He and Zheng, Huangjie and Jaitly, Navdeep and Collobert, Ronan and Zhang, Yizhe},
  journal={arXiv preprint arXiv:2604.01193},
  year={2026}
}

@article{gou2021knowledge,
  title={Knowledge distillation: A survey},
  author={Gou, Jianping and Yu, Baosheng and Maybank, Stephen J and Tao, Dacheng},
  journal={International journal of computer vision},
  volume={129},
  number={6},
  pages={1789--1819},
  year={2021},
  publisher={Springer}
}

@article{wang2020minilm,
  title={{Minilm}: Deep self-attention distillation for task-agnostic compression of pre-trained transformers},
  author={Wang, Wenhui and Wei, Furu and Dong, Li and Bao, Hangbo and Yang, Nan and Zhou, Ming},
  journal={Advances in neural information processing systems},
  volume={33},
  pages={5776--5788},
  year={2020}
}

@inproceedings{jiao2020tinybert,
  title={{Tinybert}: Distilling bert for natural language understanding},
  author={Jiao, Xiaoqi and Yin, Yichun and Shang, Lifeng and Jiang, Xin and Chen, Xiao and Li, Linlin and Wang, Fang and Liu, Qun},
  booktitle={Findings of the association for computational linguistics: EMNLP 2020},
  pages={4163--4174},
  year={2020}
}

@article{sanh2019distilbert,
  title={{DistilBERT}, a distilled version of BERT: smaller, faster, cheaper and lighter},
  author={Sanh, Victor and Debut, Lysandre and Chaumond, Julien and Wolf, Thomas},
  journal={arXiv preprint arXiv:1910.01108},
  year={2019}
}

@inproceedings{kim2016sequence,
  title={Sequence-level knowledge distillation},
  author={Kim, Yoon and Rush, Alexander M},
  booktitle={Proceedings of the 2016 conference on empirical methods in natural language processing},
  pages={1317--1327},
  year={2016}
}

@article{rusu2015policy,
  title={Policy distillation},
  author={Rusu, Andrei A and Colmenarejo, Sergio Gomez and Gulcehre, Caglar and Desjardins, Guillaume and Kirkpatrick, James and Pascanu, Razvan and Mnih, Volodymyr and Kavukcuoglu, Koray and Hadsell, Raia},
  journal={arXiv preprint arXiv:1511.06295},
  year={2015}
}

@article{hinton2015distilling,
  title={Distilling the knowledge in a neural network},
  author={Hinton, Geoffrey and Vinyals, Oriol and Dean, Jeff},
  journal={arXiv preprint arXiv:1503.02531},
  year={2015}
}

@inproceedings{ko2024distillm,
  title={{DISTILLM}: towards streamlined distillation for large language models},
  author={Ko, Jongwoo and Kim, Sungnyun and Chen, Tianyi and Yun, Se-Young},
  booktitle={Proceedings of the 41st International Conference on Machine Learning},
  pages={24872--24895},
  year={2024}
}

@inproceedings{guminillm,
  title={{MiniLLM}: Knowledge Distillation of Large Language Models},
  author={Gu, Yuxian and Dong, Li and Wei, Furu and Huang, Minlie},
  booktitle={The Twelfth International Conference on Learning Representations},
  year={2024}
}

@inproceedings{chen2024self,
  title={Self-Play Fine-Tuning Converts Weak Language Models to Strong Language Models},
  author={Chen, Zixiang and Deng, Yihe and Yuan, Huizhuo and Ji, Kaixuan and Gu, Quanquan},
  booktitle={International Conference on Machine Learning},
  pages={6621--6642},
  year={2024},
  organization={PMLR}
}

@article{zhao2026self,
  title={Self-Distilled Reasoner: On-Policy Self-Distillation for Large Language Models},
  author={Zhao, Siyan and Xie, Zhihui and Liu, Mengchen and Huang, Jing and Pang, Guan and Chen, Feiyu and Grover, Aditya},
  journal={arXiv preprint arXiv:2601.18734},
  year={2026}
}

@article{yang2026self,
  title={Self-Distilled RLVR},
  author={Yang, Chenxu and Qin, Chuanyu and Si, Qingyi and Chen, Minghui and Gu, Naibin and Yao, Dingyu and Lin, Zheng and Wang, Weiping and Wang, Jiaqi and Duan, Nan},
  journal={arXiv preprint arXiv:2604.03128},
  year={2026}
}

@article{hubotter2026reinforcement,
  title={Reinforcement Learning via Self-Distillation},
  author={H{\"u}botter, Jonas and L{\"u}beck, Frederike and Behric, Lejs and Baumann, Anton and Bagatella, Marco and Marta, Daniel and Hakimi, Ido and Shenfeld, Idan and Buening, Thomas Kleine and Guestrin, Carlos and others},
  journal={arXiv preprint arXiv:2601.20802},
  year={2026}
}

@article{song2026survey,
  title={A Survey of On-Policy Distillation for Large Language Models},
  author={Song, Mingyang and Zheng, Mao},
  journal={arXiv preprint arXiv:2604.00626},
  year={2026}
}

@inproceedings{laskey2017dart,
  title={{Dart}: Noise injection for robust imitation learning},
  author={Laskey, Michael and Lee, Jonathan and Fox, Roy and Dragan, Anca and Goldberg, Ken},
  booktitle={Conference on robot learning},
  pages={143--156},
  year={2017},
  organization={PMLR}
}

@inproceedings{agarwal2024policy,
  title={On-policy distillation of language models: Learning from self-generated mistakes},
  author={Agarwal, Rishabh and Vieillard, Nino and Zhou, Yongchao and Stanczyk, Piotr and Garea, Sabela Ramos and Geist, Matthieu and Bachem, Olivier},
  booktitle={The twelfth international conference on learning representations},
  year={2024}
}

@article{lu2025onpolicydistillation,
  author = {Kevin Lu and Thinking Machines Lab},
  title = {On-Policy Distillation},
  journal = {Thinking Machines Lab: Connectionism},
  year = {2025},
  note = {https://thinkingmachines.ai/blog/on-policy-distillation},
  doi = {10.64434/tml.20251026},
}

@inproceedings{joshi2025theory,
  title={A Theory of Learning with Autoregressive Chain of Thought},
  author={Joshi, Nirmit and Vardi, Gal and Block, Adam and Goel, Surbhi and Li, Zhiyuan and Misiakiewicz, Theodor and Srebro, Nathan},
  booktitle={The Thirty Eighth Annual Conference on Learning Theory},
  pages={3161--3212},
  year={2025},
  organization={PMLR}
}

@inproceedings{
altabaa2025cot,
title={CoT Information: Improved Sample Complexity under Chain-of-Thought Supervision},
author={Awni Altabaa and Omar Montasser and John Lafferty},
booktitle={The Thirty-ninth Annual Conference on Neural Information Processing Systems},
year={2026}
}

@article{li2024interactive,
  title={Interactive and Hybrid Imitation Learning: Provably Beating Behavior Cloning},
  author={Li, Yichen and Zhang, Chicheng},
  journal={arXiv preprint arXiv:2412.07057},
  year={2024}
}

@article{li2023agnostic,
  title={Agnostic Interactive Imitation Learning: New Theory and Practical Algorithms},
  author={Li, Yichen and Zhang, Chicheng},
  journal={arXiv preprint arXiv:2312.16860},
  year={2023}
}

@book{geer2000empirical,
  title={Empirical Processes in M-estimation},
  author={Geer, Sara A},
  volume={6},
  year={2000},
  publisher={Cambridge university press}
}

@article{zhang2006epsilon,
  title={From $\varepsilon$-Entropy to KL-Entropy: Analysis of Minimum Information Complexity Density Estimation},
  author={Zhang, Tong},
  journal={The Annals of Statistics},
  pages={2180--2210},
  year={2006},
  publisher={JSTOR}
}

@article{sekhari2023selective,
  title={Selective sampling and imitation learning via online regression},
  author={Sekhari, Ayush and Sridharan, Karthik and Sun, Wen and Wu, Runzhe},
  journal={Advances in Neural Information Processing Systems},
  volume={36},
  pages={67213--67268},
  year={2023}
}

@article{bai2023qwen,
  title={{Qwen} technical report},
  author={Bai, Jinze and Bai, Shuai and Chu, Yunfei and Cui, Zeyu and Dang, Kai and Deng, Xiaodong and Fan, Yang and Ge, Wenbin and Han, Yu and Huang, Fei and others},
  journal={arXiv preprint arXiv:2309.16609},
  year={2023}
}

@article{foster2021statistical,
  title={The statistical complexity of interactive decision making},
  author={Foster, Dylan J and Kakade, Sham M and Qian, Jian and Rakhlin, Alexander},
  journal={arXiv preprint arXiv:2112.13487},
  year={2021}
}

@book{polyanskiy2025information,
  title={Information theory: From coding to learning},
  author={Polyanskiy, Yury and Wu, Yihong},
  year={2025},
  publisher={Cambridge university press}
}

@article{foster2021efficient,
  title={Efficient first-order contextual bandits: Prediction, allocation, and triangular discrimination},
  author={Foster, Dylan J and Krishnamurthy, Akshay},
  journal={Advances in Neural Information Processing Systems},
  volume={34},
  pages={18907--18919},
  year={2021}
}

@book{sutton1998reinforcement,
  title={Reinforcement learning: An introduction},
  author={Sutton, Richard S and Barto, Andrew G and others},
  volume={1},
  year={1998},
  publisher={MIT press Cambridge}
}

@book{cesa2006prediction,
  title={Prediction, learning, and games},
  author={Cesa-Bianchi, Nicolo and Lugosi, G{\'a}bor},
  year={2006},
  publisher={Cambridge university press}
}

@article{sekhari2023contextual,
  title={Contextual bandits and imitation learning with preference-based active queries},
  author={Sekhari, Ayush and Sridharan, Karthik and Sun, Wen and Wu, Runzhe},
  journal={Advances in Neural Information Processing Systems},
  volume={36},
  pages={11261--11295},
  year={2023}
}

@inproceedings{li2024chain,
  title={Chain of thought empowers transformers to solve inherently serial problems},
  author={Li, Zhiyuan and Liu, Hong and Zhou, Denny and Ma, Tengyu},
  booktitle={The Twelfth International Conference on Learning Representations},
  year={2024}
}

@article{littlestone1988learning,
  title={Learning quickly when irrelevant attributes abound: A new linear-threshold algorithm},
  author={Littlestone, Nick},
  journal={Machine learning},
  volume={2},
  number={4},
  pages={285--318},
  year={1988},
  publisher={Springer}
}

@inproceedings{
liu2023tinygsm,
title={Tiny{GSM}: achieving 80\% on {GSM}8k with one billion parameters},
author={Bingbin Liu and Sebastien Bubeck and Ronen Eldan and Janardhan Kulkarni and Yuanzhi Li and Anh Nguyen and Rachel Ward and Yi Zhang},
booktitle={The 3rd Workshop on Mathematical Reasoning and AI at NeurIPS'23},
year={2023}
}

@article{cobbe2021training,
  title={Training verifiers to solve math word problems},
  author={Cobbe, Karl and Kosaraju, Vineet and Bavarian, Mohammad and Chen, Mark and Jun, Heewoo and Kaiser, Lukasz and Plappert, Matthias and Tworek, Jerry and Hilton, Jacob and Nakano, Reiichiro and others},
  journal={arXiv preprint arXiv:2110.14168},
  year={2021}
}

@misc{Karpathy2022,
  author = {Andrej Karpathy},
  title = {\text{NanoGPT}},
  year = {2022},
  publisher = {GitHub},
  journal = {GitHub repository},
  howpublished = {\url{https://github.com/karpathy/nanoGPT}},
  commit = {325be85d9be8c81b436728a420e85796c57dba7e}
}

@inproceedings{kingma-2015,
  author       = {Diederik P. Kingma and
                  Jimmy Ba},
  title        = {{Adam}: {A} Method for Stochastic Optimization},
  booktitle    = {3rd International Conference on Learning Representations, {ICLR} 2015},
  year={2015}
}

@article{fu2026revisiting,
  title={Revisiting On-Policy Distillation: Empirical Failure Modes and Simple Fixes},
  author={Fu, Yuqian and Huang, Haohuan and Jiang, Kaiwen and Liu, Jiacai and Jiang, Zhuo and Zhu, Yuanheng and Zhao, Dongbin},
  journal={arXiv preprint arXiv:2603.25562},
  year={2026}
}

@article{barreiros2026careful,
  title={A careful examination of large behavior models for multitask dexterous manipulation},
  author={Barreiros, Jose and Beaulieu, Andrew and Bhat, Aditya and Cory, Rick and Cousineau, Eric and Dai, Hongkai and Fang, Ching-Hsin and Hashimoto, Kunimatsu and Irshad, Muhammad Zubair and Itkina, Masha and others},
  journal={Science Robotics},
  volume={11},
  number={113},
  pages={eaea6201},
  year={2026},
  publisher={American Association for the Advancement of Science}
}

@inproceedings{chen2019deep,
  title={Deep imitation learning for autonomous driving in generic urban scenarios with enhanced safety},
  author={Chen, Jianyu and Yuan, Bodi and Tomizuka, Masayoshi},
  booktitle={2019 IEEE/RSJ international conference on intelligent robots and systems (IROS)},
  pages={2884--2890},
  year={2019},
  organization={IEEE}
}

@article{yang2025qwen3,
  title={{Qwen3} technical report},
  author={Yang, An and Li, Anfeng and Yang, Baosong and Zhang, Beichen and Hui, Binyuan and Zheng, Bo and Yu, Bowen and Gao, Chang and Huang, Chengen and Lv, Chenxu and others},
  journal={arXiv preprint arXiv:2505.09388},
  year={2025}
}

@article{olmo20242,
  title={2 {OLMo} 2 Furious},
  author={OLMo, Team and Walsh, Pete and Soldaini, Luca and Groeneveld, Dirk and Lo, Kyle and Arora, Shane and Bhagia, Akshita and Gu, Yuling and Huang, Shengyi and Jordan, Matt and others},
  journal={arXiv preprint arXiv:2501.00656},
  year={2024}
}

@article{li2024datacomp,
  title={{Datacomp-lm}: In search of the next generation of training sets for language models},
  author={Li, Jeffrey and Fang, Alex and Smyrnis, Georgios and Ivgi, Maor and Jordan, Matt and Gadre, Samir and Bansal, Hritik and Guha, Etash and Keh, Sedrick and Arora, Kushal and others},
  journal={Advances in Neural Information Processing Systems},
  volume={37},
  pages={14200--14282},
  year={2024}
}

@article{weber2024redpajama,
	title   = {{RedPajama}: an Open Dataset for Training Large Language Models},
	author  = {Maurice Weber and Daniel Y. Fu and Quentin Anthony and Yonatan Oren and Shane Adams and Anton Alexandrov and Xiaozhong Lyu and Huu Nguyen and Xiaozhe Yao and Virginia Adams and Ben Athiwaratkun and Rahul Chalamala and Kezhen Chen and Max Ryabinin and Tri Dao and Percy Liang and Christopher Ré and Irina Rish and Ce Zhang},
	journal = {NeurIPS Datasets and Benchmarks Track},
	year    = 2024,
}

@article{shao2024deepseekmath,
  title={{Deepseekmath}: Pushing the limits of mathematical reasoning in open language models},
  author={Shao, Zhihong and Wang, Peiyi and Zhu, Qihao and Xu, Runxin and Song, Junxiao and Bi, Xiao and Zhang, Haowei and Zhang, Mingchuan and Li, YK and Wu, Yang and others},
  journal={arXiv preprint arXiv:2402.03300},
  year={2024}
}

@article{guo2025deepseek,
  title={{DeepSeek-R1}: Incentivizing reasoning capability in LLMs via reinforcement learning},
  author={Guo, Daya and Yang, Dejian and Zhang, Haowei and Song, Junxiao and Wang, Peiyi and Zhu, Qihao and Xu, Runxin and Zhang, Ruoyu and Ma, Shirong and Bi, Xiao and others},
  journal={arXiv preprint arXiv:2501.12948},
  year={2025}
}

@inproceedings{swamy2022causal,
  title={Causal imitation learning under temporally correlated noise},
  author={Swamy, Gokul and Choudhury, Sanjiban and Bagnell, Drew and Wu, Steven},
  booktitle={International Conference on Machine Learning},
  pages={20877--20890},
  year={2022},
  organization={PMLR}
}

@article{abdin2024phi,
  title={{Phi-4} technical report},
  author={Abdin, Marah and Aneja, Jyoti and Behl, Harkirat and Bubeck, S{\'e}bastien and Eldan, Ronen and Gunasekar, Suriya and Harrison, Michael and Hewett, Russell J and Javaheripi, Mojan and Kauffmann, Piero and others},
  journal={arXiv preprint arXiv:2412.08905},
  year={2024}
}

@inproceedings{
chen2025coverage,
title={The Coverage Principle: How Pre-Training Enables Post-Training},
author={Fan Chen and Audrey Huang and Noah Golowich and Sadhika Malladi and Adam Block and Jordan T. Ash and Akshay Krishnamurthy and Dylan J Foster},
booktitle={The Fourteenth International Conference on Learning Representations},
year={2026}
}

@inproceedings{
wang2025beyond,
title={Beyond the 80/20 Rule: High-Entropy Minority Tokens Drive Effective Reinforcement Learning for {LLM} Reasoning},
author={Shenzhi Wang and Le Yu and Chang Gao and Chujie Zheng and Shixuan Liu and Rui Lu and Kai Dang and Xiong-Hui Chen and Jianxin Yang and Zhenru Zhang and Yuqiong Liu and An Yang and Andrew Zhao and Yang Yue and Shiji Song and Bowen Yu and Gao Huang and Junyang Lin},
booktitle={The Thirty-ninth Annual Conference on Neural Information Processing Systems},
year={2026}
}

@article{zheng2026scope,
  title={{SCOPE}: Signal-Calibrated On-Policy Distillation Enhancement with Dual-Path Adaptive Weighting},
  author={Zheng, Binbin and Ma, Xing and Liang, Yiheng and Ruan, Jingqing and Fu, Xiaoliang and Lin, Kepeng and Zhu, Benchang and Zeng, Ke and Cai, Xunliang},
  journal={arXiv preprint arXiv:2604.10688},
  year={2026}
}

@article{ross2014reinforcement,
  title={Reinforcement and imitation learning via interactive no-regret learning},
  author={Ross, Stephane and Bagnell, J Andrew},
  journal={arXiv preprint arXiv:1406.5979},
  year={2014}
}

@book{ronchetti2009robust,
  title={Robust statistics},
  author={Ronchetti, Elvezio M and Huber, Peter J},
  year={2009},
  publisher={John Wiley \& Sons Hoboken, NJ, USA}
}

@article{team2024gemma,
  title={Gemma 3 technical report},
  author={Kamath, Aishwarya and Ferret, Johan and Pathak, Shreya and Vieillard, Nino and Merhej, Ramona and Perrin, Sarah and Matejovicova, Tatiana and Ram{\'e}, Alexandre and Rivi{\`e}re, Morgane and Rouillard, Louis and others},
  journal={arXiv preprint arXiv:2503.19786},
  year={2025},
  publisher={ArXiv}
}
